\documentclass[twoside,11pt]{article}

\usepackage{jmlr2e}

\usepackage{times}
\usepackage{epsfig}
\usepackage{graphicx}
\usepackage{amsmath}
\usepackage{amssymb}
\usepackage{booktabs,wrapfig,picinpar}
\usepackage{subcaption}
\usepackage{microtype,enumitem,multirow,amssymb,pifont,algorithmic,algorithm,xcolor,colortbl,bbding,pifont,changepage,siunitx,epstopdf}
\usepackage[flushleft]{threeparttable}
\DeclareMathOperator*{\argmin}{argmin}

\newtheorem{innercustomgeneric}{\customgenericname}
\providecommand{\customgenericname}{}
\newcommand{\newcustomtheorem}[2]{%
  \newenvironment{#1}[1]
  {%
  \renewcommand\customgenericname{#2}%
  \renewcommand\theinnercustomgeneric{##1}%
  \innercustomgeneric
  }
  {\endinnercustomgeneric}
}

 \newcustomtheorem{customthm}{Theorem}
 \newcustomtheorem{customlemma}{Lemma}
 \newcustomtheorem{customcorollary}{Corollary}

\jmlrheading{1}{2000}{1-48}{4/00}{10/00}{meila00a}{Marina Meil\u{a} and Michael I. Jordan}

\ShortHeadings{Optimizing Large-Scale Hyperparameters via Automated Learning Algorithm}{Gu et al.}
\firstpageno{1}

\begin{document}

\title{Optimizing Large-Scale Hyperparameters via Automated Learning Algorithm}

\author{\name Bin Gu \email bin.gu@mbzuai.ac.ae \\
\addr Department of machine learning, Mohamed bin ZayedUniversity of Artificial Intelligence, UAE \\
 JD Finance America Corporation, Mountain View, CA, 94043, USA
\\
\name Guodong Liu \email guodong.liu.e@pitt.edu \\
\name Yanfu Zhang \email YAZ91@pitt.edu \\
\addr Department of Electrical and Computer Engineering, University of Pittsburgh, Pittsburgh, PA, 15261, USA\\
\name Xiang Geng  \email  gengxiang@nuist.edu.cn \\
\addr School of Computer \& Software, Nanjing University of Information Science \& Technology, Nanjing,  P.R.China \\
   \name Heng Huang \email {heng.huang@pitt.edu} \\
\addr Department of Electrical and Computer Engineering, University of Pittsburgh, Pittsburgh, PA, 15261, USA\\
       JD Finance America Corporation, Mountain View, CA, 94043, USA
}

\editor{}

\maketitle

\begin{abstract}
Modern machine learning algorithms usually involve tuning multiple (from one to  thousands) hyperparameters which play a pivotal role  in terms of model  generalizability.
Black-box optimization and gradient-based algorithms are  two dominant approaches to hyperparameter optimization while  they have totally  distinct advantages.
How to design a new hyperparameter optimization technique inheriting all benefits  from both approaches is still an open problem. To address this challenging problem, in this paper, we  propose a new  hyperparameter optimization  method with zeroth-order hyper-gradients (HOZOG). Specifically, we first exactly formulate  hyperparameter optimization  as  an $\mathcal{A}$-based
constrained optimization problem, where $\mathcal{A}$ is  a black-box optimization algorithm (such as deep neural network).
Then, we use the average zeroth-order hyper-gradients  to update  hyperparameters. We provide the feasibility analysis of using  HOZOG to achieve hyperparameter optimization. Finally, the experimental  results on three representative  hyperparameter (the size is from 1 to 1250) optimization tasks  demonstrate the benefits of HOZOG  in terms of  \textit{simplicity, scalability, flexibility, effectiveness and efficiency} compared with the  state-of-the-art hyperparameter optimization methods.
\end{abstract}


\section{Introduction}
Modern machine learning algorithms usually involve tuning multiple  hyperparameters whose size could be from one to  thousands. For example,   support vector machines \citep{vapnik2013nature} have the regularization parameter and kernel hyperparameter,  deep neural networks \citep{krizhevsky2012imagenet} have the  optimization hyperparameters (e.g., learning rate schedules and
momentum) and regularization hyperparameters (e.g., weight decay and dropout rates). The performance
of the most prominent algorithms strongly depends  on the appropriate  setting of these hyperparameters. 

Traditional hyperparameter tuning is treated as a bi-level optimization  problem as follows.
\vspace{-0.05in}
\begin{eqnarray}\label{formulation_regularized}
 \min_{ \lambda \in \mathbb{R}^p}  f(\lambda) = E(w(\lambda),\lambda),
\quad  s.t. \ \  w(\lambda) \in {\argmin}_{w \in \mathbb{R}^d} L(w,\lambda)
\end{eqnarray}
\vspace{-0.16in}
\\
where $w \in \mathbb{R}^d$ are the model parameters,  $\lambda \in \mathbb{R}^p$ are the hyperparameters,   the outer objective $E$ \footnote{The choice of objective function $E$ depends on the specified
tasks. For example, accuracy, AUC or F1 can be used
for binary classification problem. Square  error loss or absolute error loss  can be used   as the objective of $E$  for regression problems on validation samples.} represents a proxy of the generalization error w.r.t. the hyperparameters,   the inner objective $L$ represents  traditional learning problems (such as regularized empirical risk minimization problems), and $w(\lambda)$ are the optimal model parameters  of the inner objective  $L$ for the fixed hyperparameters $\lambda$. Note that the size of hyperparameters is normally much smaller than the one of  model parameters (\emph{i.e.}, $p\ll d$).
Choosing appropriate values of hyperparameters is extremely  computationally challenging due to the nested structure involved in the optimization  problem.
 However,  at the same time both researchers and practitioners desire the hyperparameter optimization methods  as \textit{effective}, \textit{efficient},  \textit{scalable}, \textit{simple} and \textit{flexible}\footnote{``effective'': good generalization performance. ``efficient'': running fast. ``scalable'': scalable in terms of the sizes of hyperparameters and model parameters. ``simple'': easy to be implemented.  ``flexible'':  flexible to various learning algorithms.}  as possible.

Classic  techniques such as grid search \citep{gu2015new} and random search \citep{bergstra2012random} have a very restricted  application in modern hyperparameter optimization tasks, because they only can manage a very small number  of hyperparameters and cannot guarantee to converge
to  local/global minima. For modern hyperparameter tuning tasks,  black-box optimization \citep{snoek2012practical,falkner2018bohb} and gradient-based algorithms \citep{maclaurin2015gradient,franceschi2018bilevel,franceschi2017forward} are currently the dominant approaches  due to the  advantages in terms of \textit{effectiveness, efficiency, scalability, simplicity and flexibility} which are abbreviated as E2S2F in this paper.  We provide a brief review of representative black-box optimization and gradient-based
 hyperparameter optimization   algorithms in \S \ref{Section21}, and  a detailed  comparison   of them   in terms of  the above properties   in Table \ref{table:methods}.
  \begin{table}[!t]
  \vspace*{-14pt}
\small
 \center
 \caption{Representative black-box optimization and gradient-based
 hyperparameter optimization   algorithms. (``BB'' and ``G'' are  the abbreviations of black-box and gradient respectively, and ``$\clubsuit$'' denotes that  the property holds for a small number of hyperparmaters or medium-sized training set. ``Scalable-H'' and ``Scalable-P'' denotes scalability in terms of hyperparameters and model parameters respectively.)}
\vspace*{-6pt}
\setlength{\tabcolsep}{0.4mm}
\begin{tabular}{c|c||c|c|c||c|c|c}
 \hline
 \multirow{2}{*}{\textbf{Algorithm}}  &   \multirow{2}{*}{\textbf{Type}}     
& \multicolumn{6}{c}{\textbf{Properties}}
\\ \cline{3-8}
  &   &  \textbf{Effective} &  \textbf{Efficient}  &  \textbf{Scalable-H} & \textbf{Simple} & \textbf{Flexible} & \textbf{Scalable-P}   \\ \hline
 GPBO \citep{snoek2012practical} & BB  & $\clubsuit$ & $\clubsuit$  & \ding{55}   &  \cellcolor{green} $\checkmark$  &  \cellcolor{green} $\checkmark$   &  \cellcolor{green} $\checkmark$ \\
BOHB \citep{falkner2018bohb} & BB   & $\clubsuit$ & $\clubsuit$  & \ding{55}   &  \cellcolor{green}  $\checkmark$ &  \cellcolor{green}  $\checkmark$  &  \cellcolor{green} $\checkmark$   \\
HOAG \citep{pedregosa2016hyperparameter} & G &\cellcolor{green}  $\checkmark$ & \cellcolor{green}  $\checkmark$ & \cellcolor{green}  $\checkmark$   & \ding{55}  &  \ding{55}& \ding{55} \\
RMD \citep{maclaurin2015gradient} & G  &   \cellcolor{green}  $\checkmark$ & \cellcolor{green}  $\checkmark$ & \cellcolor{green}  $\checkmark$   & \ding{55}  &\ding{55}  &\ding{55} \\ 
RFHO  \citep{franceschi2017forward,franceschi2018bilevel} & G  &  \cellcolor{green}   $ \checkmark$ &  \cellcolor{green}  $\checkmark$  &  \cellcolor{green}  $\checkmark$  &\ding{55}  &\ding{55} &\ding{55} \\
\hline
HOZOG & BB+G    & \cellcolor{green} $\checkmark$ & \cellcolor{green} $\checkmark$ & \cellcolor{green} $\checkmark$ & \cellcolor{green} $\checkmark$ & \cellcolor{green} $\checkmark$ & \cellcolor{green} $\checkmark$   \\
 \hline 
\end{tabular}
\label{table:methods}
\vspace*{-10pt}
\end{table}
Table \ref{table:methods} clearly shows that   black-box optimization  and gradient-based approaches have totally  distinct advantages, \emph{i.e.}, black-box optimization approach is simple, flexible and salable in term of  model parameters, while gradient-based approach is effective, efficient and scalable in term of  hyperparmeters. Each property of  E2S2F  is an important criterion to a successful  hyperparameter
optimization method.  To the best of our knowledge, there is still no  algorithm satisfying all the five properties    simultaneously. 
Designing a  hyperparameter optimization method having  the  benefits  of both approaches is still an open problem.

 To address this challenging problem, in this paper, we propose a new  hyperparameter optimization  method with zeroth-order hyper-gradients (HOZOG). Specifically, we first exactly formulate  hyperparameter optimization  as an $\mathcal{A}$-based
constrained optimization problem, where $\mathcal{A}$ is  a black-box optimization algorithm (such as the deep neural network).
Then, we use the average zeroth-order hyper-gradients  to update  hyperparameters. We provide the feasibility analysis of using  HOZOG to achieve hyperparameter optimization. Finally, the experimental  results of various hyperparameter (the size is from 1 to \num[group-separator={,}]{1250}) optimization problems  demonstrate the benefits of HOZOG  in terms of  E2S2F compared with the  state-of-the-art hyperparameter optimization methods.

 \noindent \textbf{Contributions.} The  main contributions of this paper are summarized as
follows:
\begin{enumerate}[leftmargin=0.2in]

\item \textit{Effectiveness, efficiency, scalability, simplicity and flexibility} are the most important evaluation criterions for hyperparameter optimization methods. As far as we know, there does not exist a hyperparameter optimization method having all these advantages. We creatively propose a zeroth-order gradient algorithm to solve the problem  which is the first  method having all these benefits to the best of our knowledge.

\item As summarized in Table 1, black-box optimization approach has good simplicity and flexibility, while weak scalability in term of number of hyperparmeters. Meanwhile, gradient-based methods have poor flexibility and simplicity, while good scalability in term of number of hyperparmeters. We creatively proposed a zeroth-order gradient algorithm to solve the problem of hyperparameter optimization which inherits all benefits of black-box optimization approach and gradient-based methods.

\item After replacing the inner problem by an optimization algorithm (\emph{i.e.}, $w(\lambda) = \mathcal{A} (\lambda)$),  we provide an upper bound to the Lipschitz constant of the $\mathcal{A}$-based constrained optimization problem which theoretically guarantees that zeroth-order gradient algorithm can solve the problem of hyperparameter optimization.

\end{enumerate}
\noindent \textbf{Organization.}
We organize the rest of paper as follows.  In Section \ref{section2},  we propose our HOZOG algorithm. In Section \ref{section3}, we show the experimental results of HOZOG on  three hyperparameter  optimization problems.   Finally,  we conclude the paper.
\section{Hyperparameter Optimization  based on Zeroth-Order Hyper-Gradients}\label{section2}
In this section, we first give a brief  review of black-box optimization and gradient-based algorithms, and then provide our HOZOG algorithm. Finally, we provide the feasibility analysis of HOZOG.
\subsection{Brief Review of Black-Box Optimization and Gradient-based Algorithms}\label{Section21}
\noindent \textbf{Black-box optimization algorithms:} \
Black-box optimization algorithms   view the bilevel optimization problem $f$ as a black-box function. Existing black-box optimization methods \citep{snoek2012practical,falkner2018bohb} mainly employ Bayesian optimization  \citep{brochu2010tutorial}   to solve (\ref{formulation_regularized}).  Black-box optimization approach has good  simplicity and flexibility. However, a lot of references have pointed out that it can only handle  hyperparmeters from a few to several dozens \citep{falkner2018bohb} while the number of hyperparmeters in real hyperparameter optimization problems would range from hundreds to thousands.  Thus, black-box optimization approach has weak  scalability in term of the size of of hyperparmeters.

\noindent \textbf{Gradient-based algorithms:} \ The existing gradient-based algorithms can be divided into two parts (\emph{i.e.}, inexact gradients and exact gradients). The approach of inexact gradients first solves the inner problem approximately,    and then estimates the  gradient of (\ref{formulation_regularized}) based on the approximate solution by the approach of implicit differentiation \citep{pedregosa2016hyperparameter}. Because the implicit differentiation involves Hessian matrices of sizes of $d \times d$ and $d \times p$ where $p\ll d$, they have poor scalability. 
 The approach of exact gradients\footnote{Although the inner-problem is usually solved  approximately \textit{e.g.} by
taking a finite number of steps of gradient descent, we still call this kind of methods as exact gradients throughout  this paper to avoid using too complex terminology.} treats the inner level problem  as  a dynamic system, and use chain rule \citep{rudin1964principles} to compute the gradient. Because the chain rule highly depends on specific learning algorithms, this   approach has  poor flexibility and simplicity. Computing the gradients involves Hessian matrices of sizes of $p \times p$ and $d \times p$. Thus, the approach of exact gradients has better scalability  than the  approach of inexact gradients because normally we have $p\ll d$.

\noindent \textbf{\Cross \ Enlightenment:} \ 
As introduced in \citep{nesterov2017random,gu2018faster}, zeroth-order gradient (also known as finite difference approximation \citep{Cui2017Strong}) technique is a black-box optimization method which  estimates the gradient only by
two function evaluations. Thus, zeroth-order gradient technique belongs both to black-box optimization and gradient-based optimization (please see Figure \ref{Fig_HOZOG}).
We hope that the hyperparameter optimization method bases on zeroth-order hyper-gradients\footnote{We call the gradient \emph{w.r.t.} hyperparameter as hyper-gradient in this paper.} can inherit  all benefits as described   in Table \ref{table:methods}.
 \begin{figure}[ht]
\center
\includegraphics[scale=0.33]{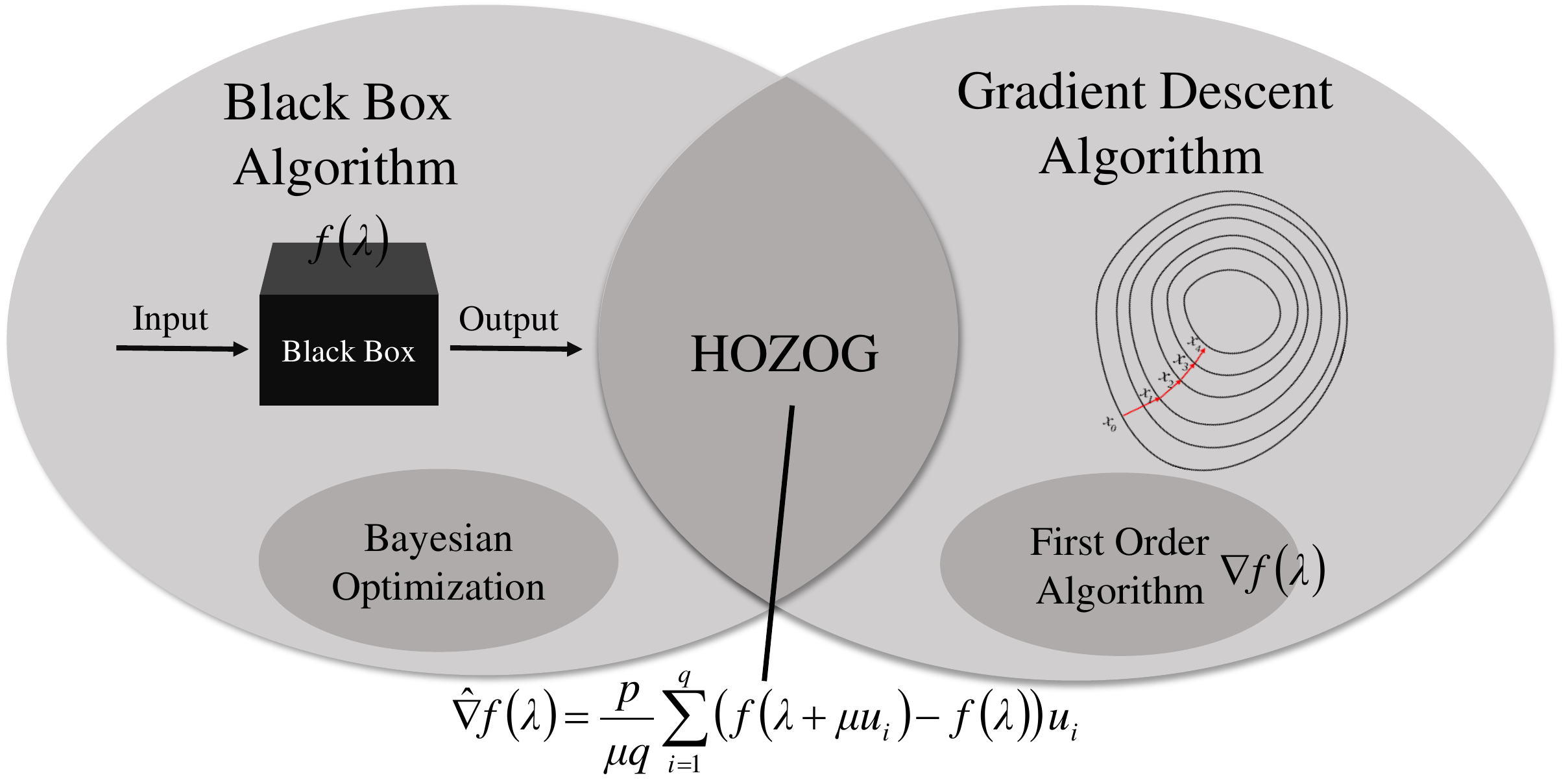}
 \caption{Principle  of our HOZOG.}
 \label{Fig_HOZOG}
\end{figure}
\subsection{HOZOG Algorithm}
\vspace*{-4pt}
\noindent \textbf{$\blacktriangleright$ \ Principle:} \ Instead of directly computing the hyper-gradient as in \citep{pedregosa2016hyperparameter,maclaurin2015gradient,franceschi2017forward,franceschi2018bilevel}, we use   two function evaluations  (\emph{i.e.}, the zeroth-order hyper-gradient technique \citep{nesterov2017random,gu2018faster}) to estimate the hyper-gradient, and   update
hyperparameters with hyper-gradients which derives  our HOZOG algorithm.

Before presenting  HOZOG algorithm  in detail, we first clarify what problem   we are solving exactly.

\noindent \textbf{$\blacktriangleright$ \ What problem   we are solving exactly?}  \
 As mentioned in (\ref{formulation_regularized}),  the inner level problem in the traditional hyperparameter tuning is finding the model parameters that minimize the inner objective $L$, (\emph{i.e.},    $w(\lambda) \in {\argmin}_{w \in \mathbb{R}^d} L(w,\lambda)$). However, in the real-world hyperparameter tuning problems, we are usually trying to find an approximate  minimum  solution of $L$ by an optimization algorithm if the inner level problem $L$ in convex.  If the inner level problem $L$ in non-convex, we usually try to find an approximate local solution or a stationary point. Thus, we replace the  inner level problem  by $w(\lambda) = \mathcal{A}(\lambda)$ where $\mathcal{A}$ is an optimization algorithm which approximately solves the inner objective $L$. Further, we replace the bi-level optimization problem (\ref{formulation_regularized}) by  the following $\mathcal{A}$-based constrained optimization problem (\ref{formulation_real}).
\begin{eqnarray}\label{formulation_real}
\min_{ \lambda \in \mathbb{R}^p}  f(\lambda) = E(w(\lambda),\lambda),
\quad s.t. \ \ w(\lambda) = \mathcal{A}(\lambda)
\end{eqnarray}
\\
where 
$w(\lambda)$ are the  values returned by the optimization algorithm $\mathcal{A}$.

\noindent \textit{$\bullet$ Hyperparameters}:
Hyperparameters  can be divided into two types, \emph{i.e.}, problem-based hyperparameters and algorithm-based hyperparameters.
\begin{enumerate}[leftmargin=0.2in]
\item Problem-based hyperparameters: The problem-based hyperparameters are  the hyperparameters involved in \textit{learning problems} such as the regularization parameter and the architectural hyperparameters in deep neural networks.
 \item Algorithm-based hyperparameters: These are the  hyperparameters  involved in  \textit{optimization algorithms} such as the learning rate,
momentum and dropout rates.
 \end{enumerate}
The traditional  bi-level optimization problem (\ref{formulation_regularized}) can only formulate the problem-based hyperparameters. However, our $\mathcal{A}$-based constrained
optimization problem (\ref{formulation_real}) can formulate  both  types of hyperparameters.


\noindent \textbf{$\blacktriangleright$ \ Algorithm:} \ To solve the $\mathcal{A}$-based constrained
optimization problem (\ref{formulation_real}), we propose  HOZOG algorithm in Algorithm \ref{algorithm2}, where the ``for" loop is  referred to as ``meta-iteration''. We describe the  two key operations of Algorithm \ref{algorithm2} (\emph{i.e.}, estimating the function value and  average  zeroth-order hyper-gradient)  in detail  as follows.

\noindent \textit{$\bullet$ Estimating the function value}: \ We treat the optimization algorithm $\mathcal{A}$ as a black-box oracle.  Given  hyperparameters $\lambda$, the black-box oracle $\mathcal{A}$ returns   model parameters $w(\lambda)$. Based on the pair of $\lambda$ and $w(\lambda)$,   the function value can be estimated as $E(w(\lambda),\lambda)$.

\noindent \textit{$\bullet$ Computing the average  zeroth-order hyper-gradient}: \  Zeroth-order hyper-gradient can be computed as $\bar{\nabla} f( \lambda)= \frac{p}{\mu } \left ( f(\lambda +\mu u) - f(\lambda) \right ) u$ based on the two function evaluations $f(\lambda+\mu u)$ and $f(\lambda)$, where $u  \sim N(0, I_p)$  is a random direction drawn from a uniform distribution over a unit sphere, and $\mu$ is an  approximate parameter.  $\bar{\nabla} f( \lambda)$  has a large variance
due to   single direction $u$. To reduce the  variance, we use the average  zeroth-order hyper-gradient (\ref{ZO_gradient}) by sampling a set of directions $\{u_i\}_{i=1}^q$.
\begin{eqnarray}\label{ZO_gradient}
\hat{\nabla} f( \lambda)= \frac{p}{\mu q} \sum_{i=1}^{q}\left ( f(\lambda +\mu u_i) - f(\lambda) \right ) u_i
\end{eqnarray}
\\
Based on the average zeroth-order hyper-gradient $\hat{\nabla} f( \lambda)$, we  update the hyperparameters as follows.
\begin{eqnarray}
\lambda \leftarrow  \lambda - \gamma  \hat{\nabla} f( \lambda)
\end{eqnarray}
\\
Note that $\hat{\nabla} f( \lambda)$ is a biased approximation to the true gradient ${\nabla} f( \lambda)$. Its bias can be reduced by
decreasing the value of $\mu$. However, in a practical system,  $\mu$ could not be too small, because in that case the function difference
could be dominated by the system noise (or error) and fails to represent the function differential \citep{lian2016comprehensive}.

\noindent \textit{$\bullet$ Parallel acceleration.} \ Because the average zeroth-order hyper-gradient involves $q+1$  function evaluations as shown in (\ref{ZO_gradient}), we can use GPU or multiple cores to compute the $q+1$  function evaluations in parallel to accelerate the  computation of  average zeroth-order hyper-gradients.


\begin{algorithm}[t]
\renewcommand{\algorithmicrequire}{\textbf{Input:}}
\renewcommand{\algorithmicensure}{\textbf{Output:}}
\caption{Hyperparameter optimization method with zeroth-order hyper-gradients (HOZOG)}
\begin{algorithmic}[1]
\REQUIRE  Learning rate   $\gamma$, approximate parameter $\mu$, size of directions $q$ and black-box inner solver $\mathcal{A}$.
 \STATE  Initialize  $\lambda_0 \in \mathbb{R}^p$.
\FOR{$t=0,1,2,\ldots,T-1$}
\STATE Generate $u = [u_1,\ldots, u_q]$, where $u_i  \sim N(0; I_p)$.
 \STATE  Compute the average zeroth-order hyper-gradient $\hat{\nabla} f( \lambda_t) = \frac{p}{\mu q} \sum_{i=1}^{q}\left ( f(\lambda_t+\mu u_i) - f(\lambda_t) \right ) u_i$, where $f(\lambda_t)$ is estimated based on the solution returned by the black-box inner solver $\mathcal{A}$.
 \STATE  Update $\lambda_{t+1} \leftarrow  \lambda_{t} - \gamma  \hat{\nabla} f( \lambda_t) $.
\ENDFOR
\ENSURE $\lambda_{T}$.
\end{algorithmic}
\label{algorithm2}
\end{algorithm}
\subsection{Feasibility Analysis}
\noindent \textbf{$\blacktriangleright$ \ Challenge:} \ In treating the optimization algorithm $\mathcal{A}(\lambda)$ as a black-box oracle that maps  $\lambda$ to $w$, the most important problem is whether the mapping function $\mathcal{A}(\lambda)$ is \textit{continuous} which is the basis of using the  zeroth-order hyper-gradient technique to optimize (\ref{formulation_real}). 

\noindent \textbf{$\bullet$ \ Continuity:} \ Before discussing the continuity of the $\mathcal{A}$-based constrained optimization problem $f(\lambda)$, we first give the definitions  of  iterative algorithm and continuous function  in Definitions \ref{Iterative_alg} and \ref{continuous} respectively.

\begin{definition}[Iterative algorithm] \label{Iterative_alg} Assume the optimization algorithm $\mathcal{A}(\lambda)$ can be formulated as a nested function as $\mathcal{A}(\lambda)=w_T$ and $w_t= \Phi_t (w_{t-1},\lambda)$ for $t=1,\ldots,T$, where $T$ is the number of iterations, $w_0$ is an initial solution, and, for every $t\in \{1,\ldots,T \}$, $\Phi_t : (\mathbb{R}^d \times \mathbb{R}^p) \rightarrow \mathbb{R}^d$ is a mapping function that represents the operation performed by the t-th step of the optimization
algorithm. We call the optimization algorithm $\mathcal{A}(\lambda)$ as an iterative algorithm.
\end{definition}
\begin{definition}[Continuous function] \label{continuous} For  all $\lambda \in \mathbb{R}^p $,  if the limit of $f(\lambda+\delta)$ as $\delta \in \mathbb{R}^p$ approaches $\textbf{0}$  exists and is equal to $f(\lambda)$, we call the function $f(\lambda)$ is continuous everywhere.
\end{definition}
Based on Definitions \ref{Iterative_alg} and \ref{continuous}, we give Theorem \ref{theorem1} to show that the $\mathcal{A}$-based constrained optimization problem $f(\lambda)$ is continuous under mild assumptions. The proof is provided in Appendix.
\begin{theorem}\label{theorem1}
If the hyperparameters $\lambda$ are continuous and the mapping functions $\Phi_t (w_{t-1},\lambda)$ (for every $t\in \{1,\ldots,T \}$) are continuous, the mapping function $\mathcal{A}(\lambda)$  is continuous, and the outer objective $E$ is continuous, we have that the  $\mathcal{A}$-based constrained optimization problem  $f(\lambda)$ is continuous \emph{w.r.t.} $\lambda$.
\end{theorem}

We provide several popular types of optimization algorithms  to show that almost existing iterative algorithms are continuous mapping functions which would make $f(\lambda)$ continious.
 \begin{enumerate}[leftmargin=0.2in]
\item \textbf{Gradient descent algorithms}: If $\mathcal{A}$ is a gradient descent algorithm (such as  SGD \citep{ghadimi2013stochastic}, SVRG \citep{reddi2016stochastic,allen2016variance}, SAGA \citep{defazio2014saga}, SPIDER \citep{fang2018spider}), the updating rules can be formulated as  $w\leftarrow  w - \gamma'  v$, where $v$ is a stochastic or deterministic  gradient estimated by the current $w$, and $\gamma'$ is the learning rate.
 To accelerate the training of deep neural networks, multiple adaptive variants of SGD (e.g., Adagrad, RMSProp and Adam \citep{goodfellow2016deep}) have emerged.
 \item  \textbf{Proximal gradient descent algorithms}: If $\mathcal{A}$ is a proximal gradient descent algorithm \citep{zhao2014accelerated,xiao2014proximal,gu2018asynchronous}, the updating rules should be the form of $w \leftarrow  \textrm{Prox} ( w - \gamma'  v)$, where $\textrm{Prox}$ is a proximal operator (such as the soft-thresholding operator for Lasso \citep{tibshirani1996regression}) which is normally continuous \citep{bredies2007iterative,zou2006adaptive}.
\end{enumerate}
 It is easy to verify that the mapping functions $\mathcal{A}(\lambda)$ corresponding to these iterative algorithms are continuous according to Theorem \ref{theorem1}.


For a continuous function $f(\lambda)$, there  exists a  Lipschitz constant $L$ (see Definition \ref{smooth}) which upper bounds $\frac{| f( \lambda_1) -  f( \lambda_2) |}{\|\lambda_1 - \lambda_2\|}$, $\forall \lambda_1, \lambda_2 \in \mathbb{R}^p$. 
Unfortunately,  exactly calculating the   Lipschitz constant of $f(\lambda)$ is NP-hard problem \citep{virmaux2018lipschitz}. We provide an upper bound\footnote{Although the upper bound is related to $T$, our simulation results  show that it does not grow exponentially with $T$ because $L_{A_t}$ or $L_{B_t}$ is not larger than one at most times.} to the Lipschitz constant of $f(\lambda)$ in Theorem \ref{theorem2}.
\begin{definition}[Lipschitz continuous constant] \label{smooth}
For a continuous function $f(\lambda)$, there exists a constant $L$ such that, $\forall \lambda_1, \lambda_2 \in \mathbb{R}^p$, we have $\| f( \lambda_1) -  f( \lambda_2) \| \leq L  \|\lambda_1 - \lambda_2\|$.
 The smallest $L$ for which the  inequality is true is called the Lipschitz constant of $f(\lambda)$.
\end{definition}

\begin{theorem}\label{theorem2}
Given the continuous mapping functions $\Phi_t (w_{t-1},\lambda)$ where $t\in \{1,\ldots,T \}$), $A_t=\frac{\partial \Phi_t (w_{t-1},\lambda)}{ \partial w_{t-1}}$, $B_t=\frac{\partial \Phi_t (w_{t-1},\lambda)}{ \partial \lambda}$. Given the continuous objective function $E(w_T,\lambda)$,  $A_{T+1}=\frac{\partial E (w_{T},\lambda)}{ \partial w_T}$ and $B_{T+1}=\frac{\partial E (w_{T},\lambda)}{ \partial \lambda }$. Let $L_{A_{t}}=\sup_{\lambda \in \mathbb{R}^p, w\in \mathbb{R}^d} \left  \| A_{t+1} \right  \|_2$, $L_{B_{t}}=\sup_{\lambda \in \mathbb{R}^p, w\in \mathbb{R}^d} \left  \|  B_{t}\right  \|_2$. Let $L(f) $ denote the  Lipschitz constant of the continuous function $f(\lambda)$,  we  can  upper bound  $L(f)$ by $\sum_{t=1}^{T+1}   L_{B_{t}}  L_{A_{t+1}} \ldots L_{A_{T+1}} $.
\end{theorem}

\noindent \textbf{$\blacktriangleright$ \ Conclusion:} \  Because the $\mathcal{A}$-based constrained optimization problem  $f(\lambda)$ is continuous, we can use the zeroth-order hyper-gradient technique to optimize $f(\lambda)$ \citep{nesterov2017random}.
\cite{nesterov2017random}  provided the convergence guarantee of zeroth-order hyper-gradient method when $f(\lambda)$ is  Lipschitz continuous as defined in Definition \ref{smooth}.
\section{Experiments}\label{section3}
We conduct the hyperparameter optimization experiments on three representative learning problems (\emph{i.e.}, $l_2$-regularized logistic regression,  deep neural networks (DNN) and data hyper-cleaning), whose sizes of hyperparameters are from 1 to \num[group-separator={,}]{1250}. We also test the parameter sensitivity analysis of HOZOG under different  settings of  parameters $q$, $\mu$ and $\gamma$, which are included in Appendix due to the page limit. All the experiments are conducted on a Linux system equipped with four NVIDIA Tesla P40 graphic cards.

\noindent \textbf{$\bullet$ \ Compared algorithms:} \ We compare our HOZOG with 	the representative hyperparameter optimization approaches such as random search (RS) \citep{bergstra2012random}, RFHO with  forward (FOR) or reverse (REV) gradients \citep{franceschi2017forward} \footnote{The code of RFHO is is available at \url{https://github.com/lucfra/RFHO}.},
HOAG  \citep{pedregosa2016hyperparameter}\footnote{The code of HOAG is available at \url{https://github.com/fabianp/hoag}.},
GPBO \cite{snoek2012practical} \footnote{The code of GPBO is available at \url{http://github.com/fmfn/BayesianOptimization/}.} and BOHB  \citep{falkner2018bohb} \footnote{The code of BOHB is available at \url{https://github.com/automl/HpBandSter}. Note that BOHB is an improved version of Hyperband \citep{li2017hyperband}. Thus, we do not compare HOZOG with Hyperband.}. Most of them are the  representative  black-box optimization and gradient-based hyperparameter optimization
algorithms as presented in Table \ref{table:methods}.  We implement our  HOZOG in Python\footnote{Our code is available at \url{https://github.com/jsgubin/HOZOG}.}. 

\noindent \textbf{$\bullet$ \ Evaluation criteria:} \ We compare different algorithms with  three criteria, \emph{i.e.}, $\| \nabla f(\lambda)\|_2$, suboptimality and   test error, where ``suboptimality'' denotes $f(\lambda)- f(\lambda^{\diamond})$ and $f(\lambda^{\diamond})$ is the minimum value of $f(\lambda)$  for all $\lambda$ which have been explored, and test error is the average loss on the testing set. Note the hyper-gradients $ \nabla f(\lambda) $ for all method except for FOR and REV  are computed by Eq. (\ref{ZO_gradient}).

\noindent \textbf{$\bullet$ Datasets:} The datasets used in experiments are News20, Covtype, Real-sim, CIFAR-10 and Mnist datasets from LIBSVM repository, which   is available at \url{https://www.csie.ntu.edu.tw/~cjlin/libsvmtools/datasets/}. Especially, for News20 and Mnist two multi-class datasets, we transform them to  binary classification problems by randomly partitioning the data into two groups.

\noindent \textbf{$\bullet$ Parameters of HOZOG:}  The values  of parameters $q$, $\mu$ and $\gamma$ in HOZOG are given in Table \ref{table:parameter}. Especially, $q$ plays an important role to HOZOG because it determines the accuracy and the running time of estimating the gradients. We empirically observe that  $q \leq 5$  has a good balance between the two objectives.

 \begin{table}[htbp]
 \center
 \caption{The  parameter settings of HOZOG in the experiments. (``\# HP'' is  the abbreviation of the number of  hyperparameters.) }
\begin{tabular}{c|c|c|c|c|c|c}
\hline
 \multicolumn{2}{c|}{\textbf{Experiment}} & \textbf{\# HP}  & \textbf{Dataset}&  $q$ & $\mu$ &  $\gamma$       \\ \hline
\multicolumn{2}{c|}{\multirow{3}{*}{$l_2$-regularized logistic regression}} &\multirow{3}{*}{1} & News20 & 1 & 0.01 & 0.05  \\
\multicolumn{2}{l|}{} & & Covtype & 1 & 0.01 &  0.03  \\
 \multicolumn{2}{l|}{} & & Real-sim & 1 & 0.01 & 0.005  \\  \hline
\multirow{3}{*}{Deep Neural Networks} & 2-layer CNN & 100 & \multirow{3}{*}{CIFAR-10} & 3 & 0.01 & 0.001   \\
& VGG-16  & 20 & & 3 & 1 & 1 \\
& ResNet-152  & 10 & & 3 & 1 & 5 \\ \hline
\multicolumn{2}{c|}{Data hyper-cleaning}  & 500$/$1250 & Mnist & 5 &  1 & 1 \\ \hline
\end{tabular}
\label{table:parameter}
\end{table}

\subsection{$l_2$-Regularized Logistic Regression}\label{secLR}
\noindent \textbf{Experimental setup:}
We consider to estimate the regularization
parameter in the  $l_2$-regularized logistic regression model.
We split one data set into three subsets (\emph{i.e.}, the train set ${\cal D}_{tr}$, validation set ${\cal D}_{val}$ and test set ${\cal D}_{t}$) with a ratio of 2:1:1.  
We use the logistic loss $l(t)=\log(1+e^{-t})$ as the loss function.  The hyperparameter optimization problem for $l_2$-regularized logistic regression is formulated as follows.
\begin{align}\label{l2_lr}
\mathop{\arg\min}_{\lambda\in[-10,10]}  \sum_{i\in{\cal D}_{val}}  l(y_i \langle x_i, w(\lambda) \rangle ),
\quad\quad s.t.\quad  w(\lambda)\in\mathop{\arg\min}_{w\in\mathbb{R}^d} \sum_{i\in{\cal D}_{tr}}l(y_i \langle x_i,w(\lambda) \rangle )+e^{\lambda}\Vert w\Vert^2
\end{align}
\\
 The solver used for solving the inner objective is L-BFGS\footnote{The implementation is available at \url{https://github.com/fabianp/hoag}.} \citep{liu1989limited} for HOAG  and Adam \citep{kingma2014adam} for the others. 

\noindent \textbf{Results and discussions:} Figure \ref{figureLR} presents the convergence results of  suboptimality,  $\| \nabla f(\lambda)\|_2$ and   test error  \emph{vs.} the running time for different methods. Note that we take same initial values of $\lambda$ and $w$ for all gradient-based methods, while the black-box methods naturally start from different points.   Because HOAG works with tolerances and warm start strategy, HOAG has a fast convergence at the early stage but a slow convergence at the late stage as shown in Figures \ref{fig1.4}-\ref{fig1.6}. We observe that HOZOG  runs faster than other gradient-based methods. This is because that FOR and REV need much time to compute  hyper-gradients.  Figures \ref{fig1.4}-\ref{fig1.6} provide $\| \nabla f(\lambda)\|_2$ of different methods as functions of running time. We can see that the black-box methods (\textit{i.e.}, BOHB and GPBO) spend much time on exploring because $\| \nabla f(\lambda)\|_2$ of these methods didn't strictly go down in the early stage. Overall, all the results show that HOZOG has a faster convergence than other methods.
\begin{figure*}[t]
\centering
\begin{subfigure}[b]{0.32\textwidth}
\centering
\includegraphics[width=2in]{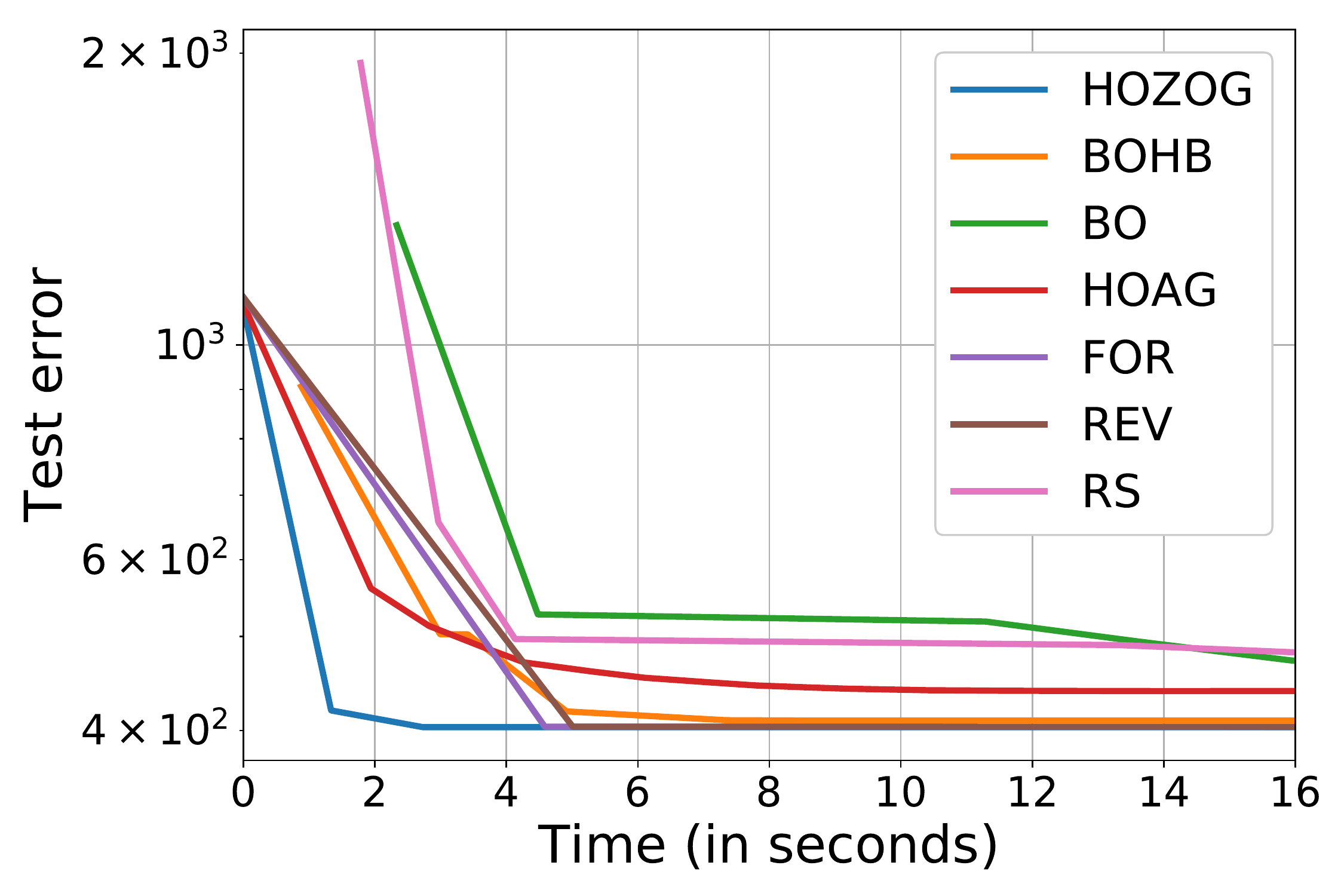}
\caption{News20}
\label{fig1.7}
\end{subfigure}
\begin{subfigure}[b]{0.32\textwidth}
\centering
  \includegraphics[width=2in]  {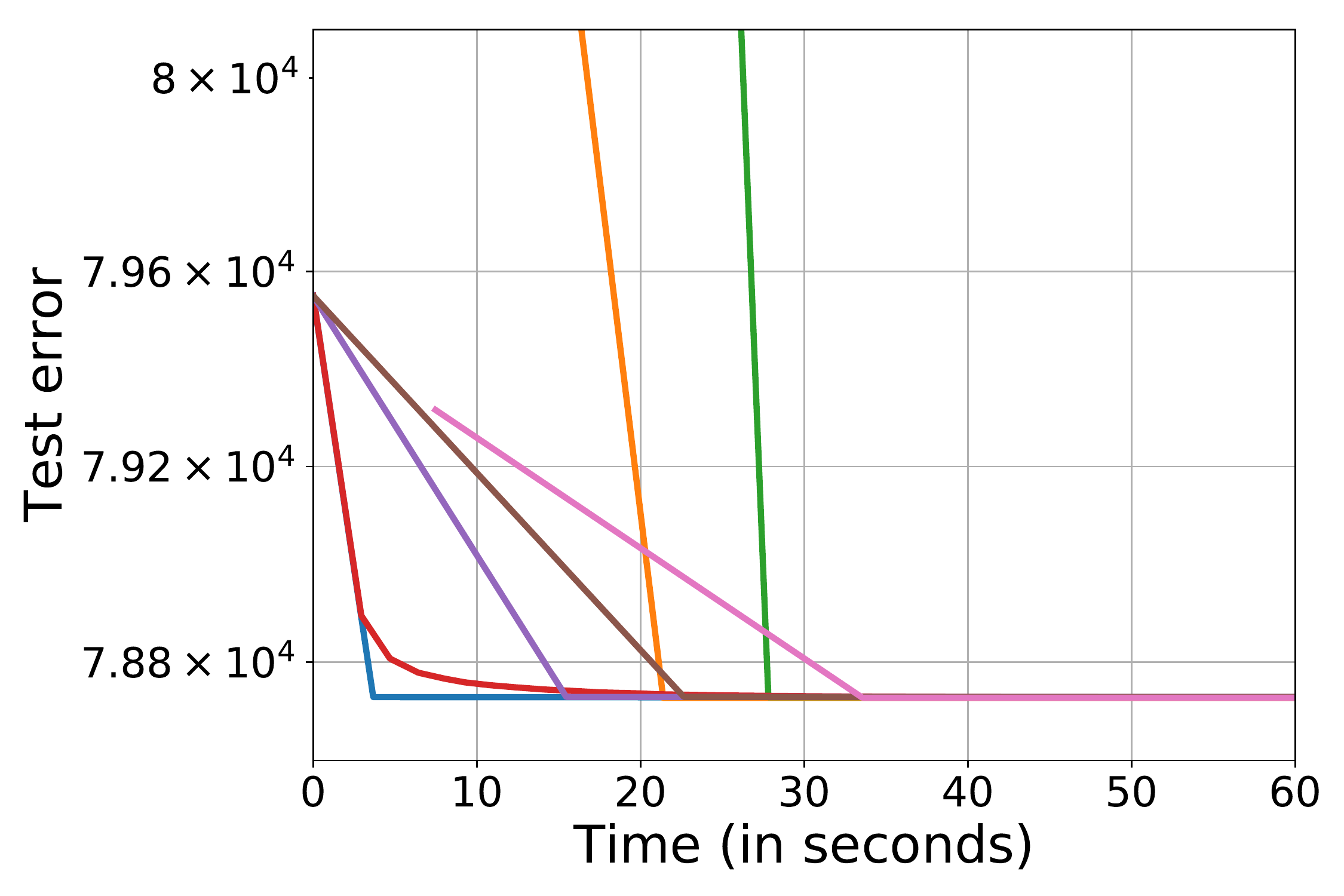}
\caption{Covtype}
\end{subfigure}
\begin{subfigure}[b]{0.32\textwidth}
\centering
\includegraphics[width=2in] {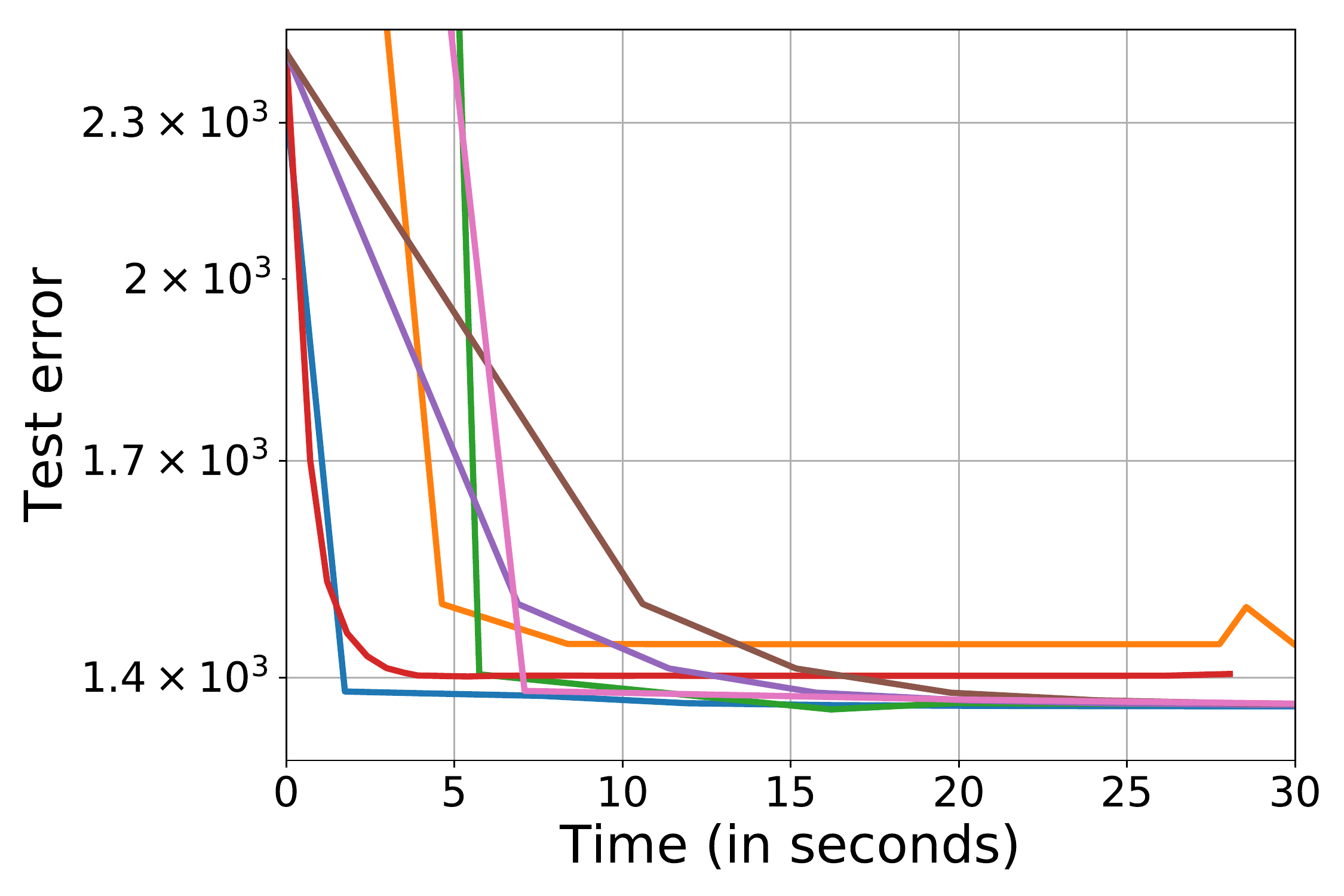}
\caption{Real-sim}
\label{fig1.9}
\end{subfigure}
\begin{subfigure}[b]{0.32\textwidth}
\centering
\includegraphics[width=2in]{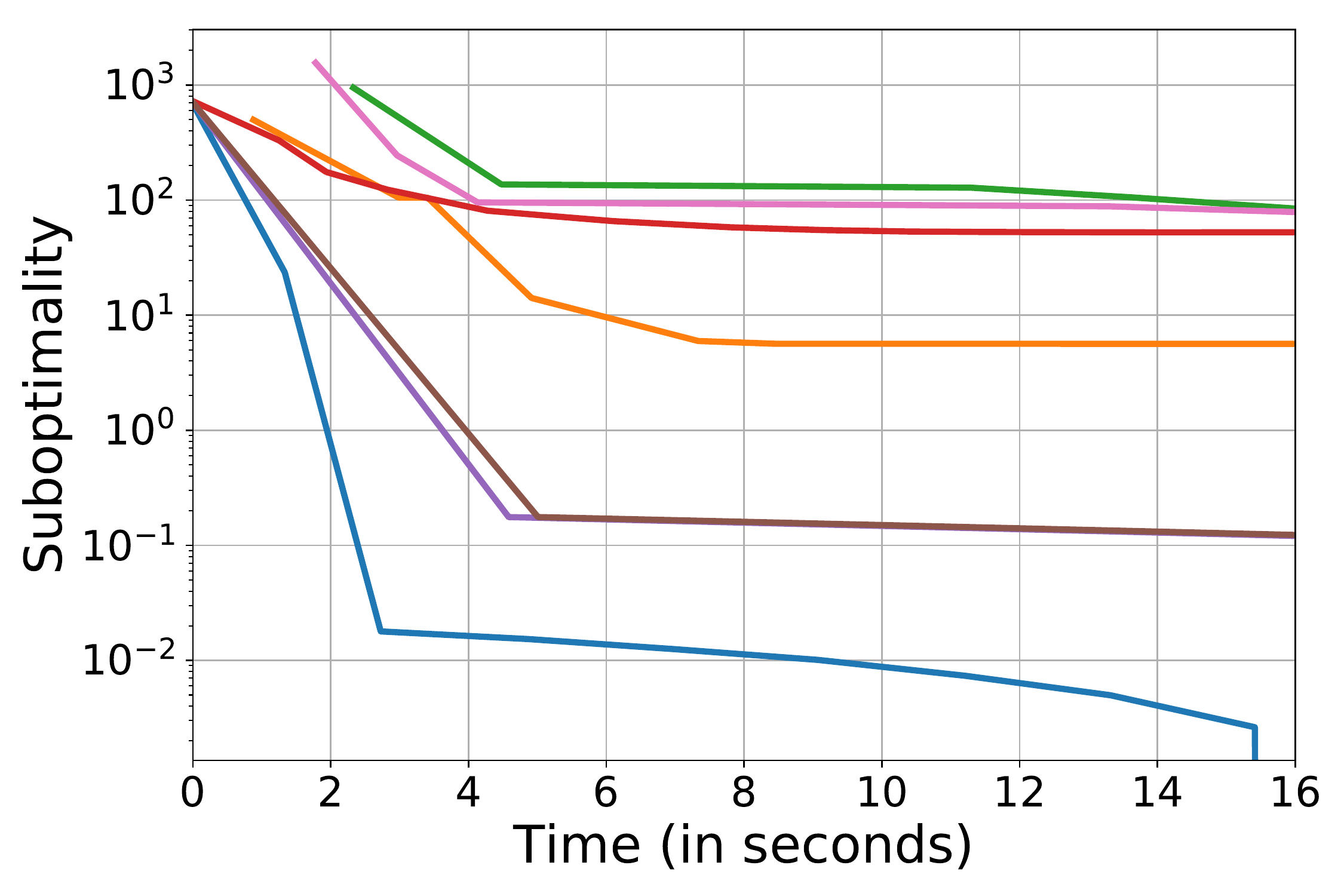}
\caption{News20}
\label{fig1.1}
\end{subfigure}
\begin{subfigure}[b]{0.32\textwidth}
\centering
\includegraphics[width=2in]{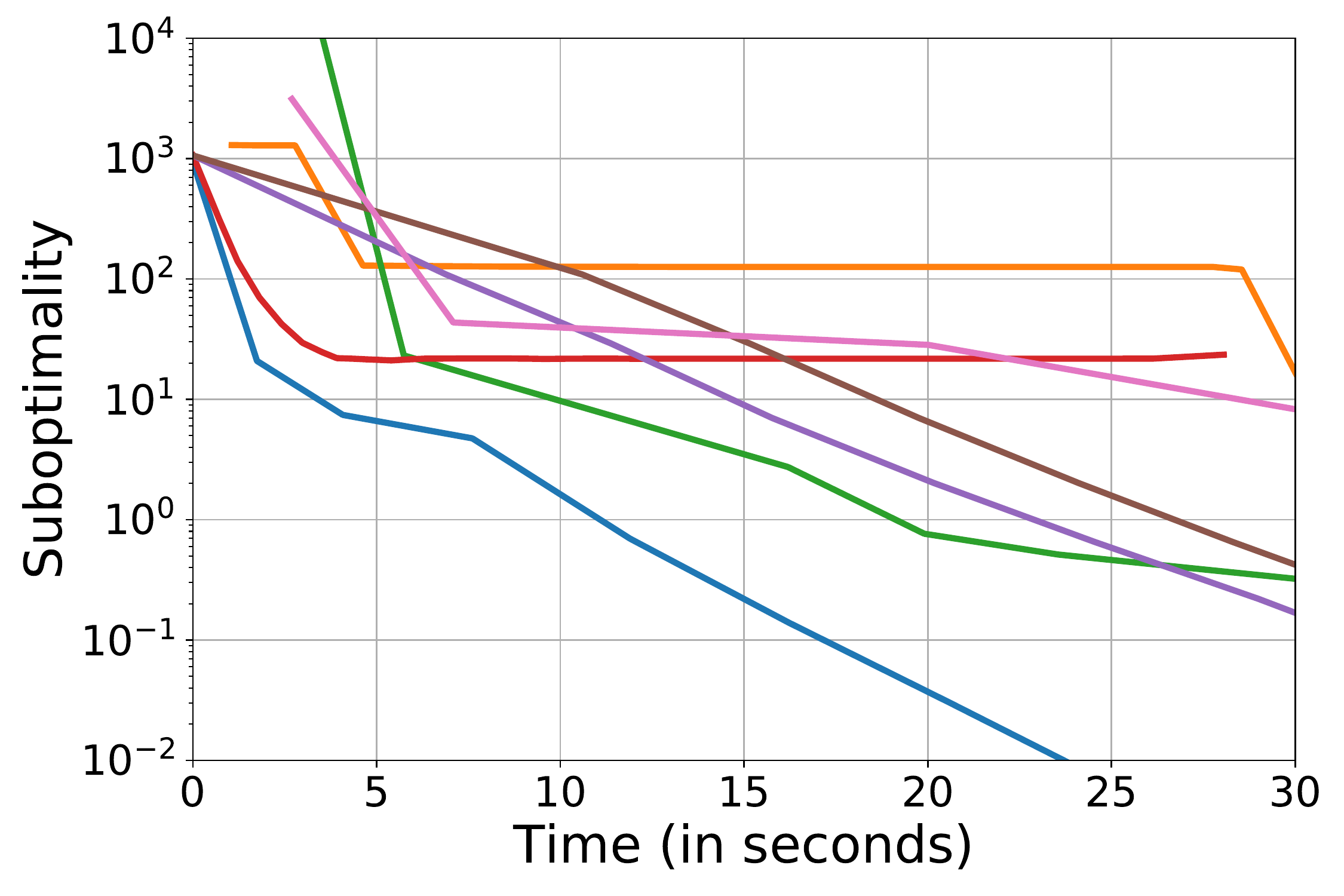}
\caption{Real-sim}
\end{subfigure}
\begin{subfigure}[b]{0.32\textwidth}
\centering
\includegraphics[width=2in]{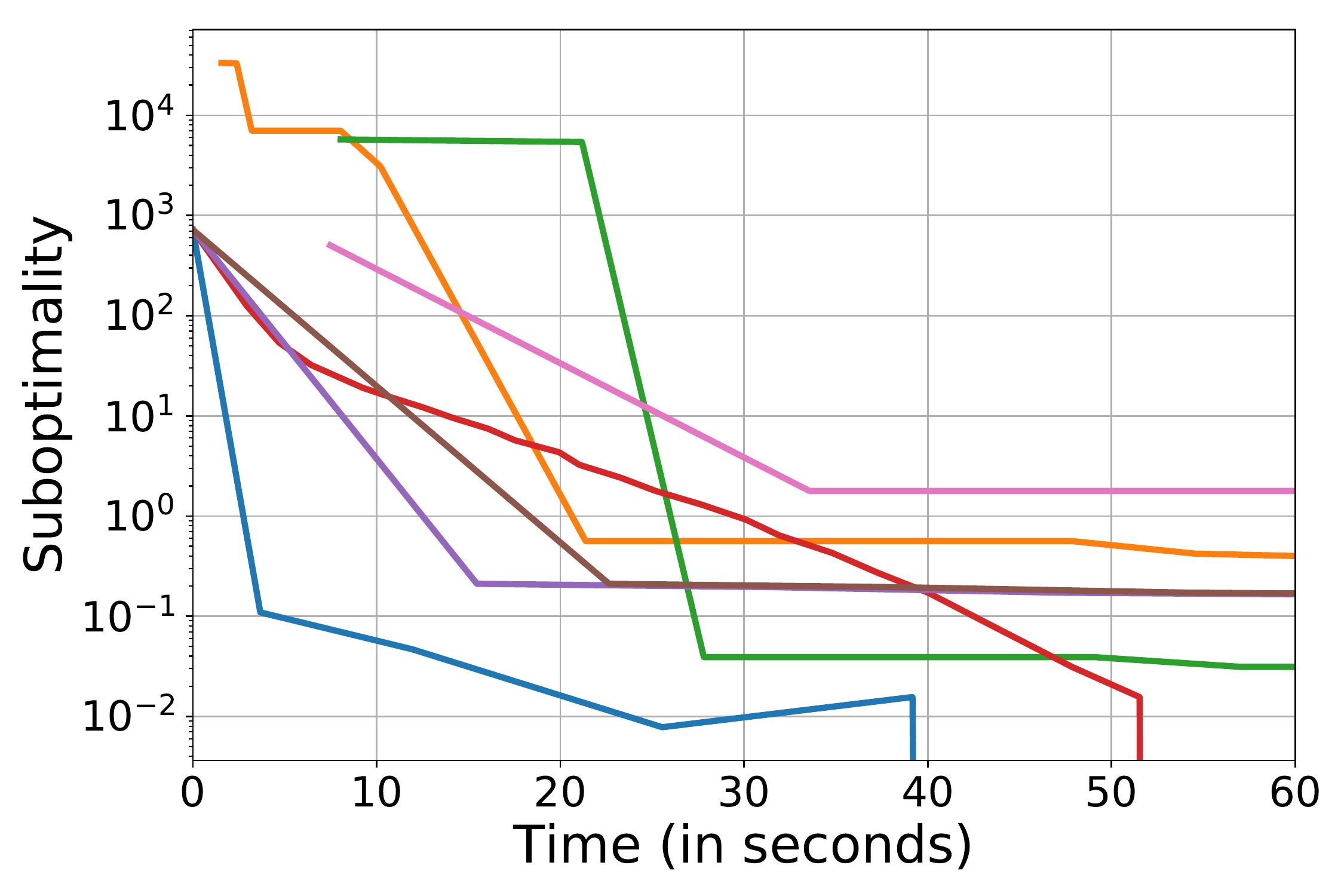}
\caption{Covtype}
\label{fig1.3}
\end{subfigure}
\begin{subfigure}[b]{0.32\textwidth}
\centering
\includegraphics[width=2in]{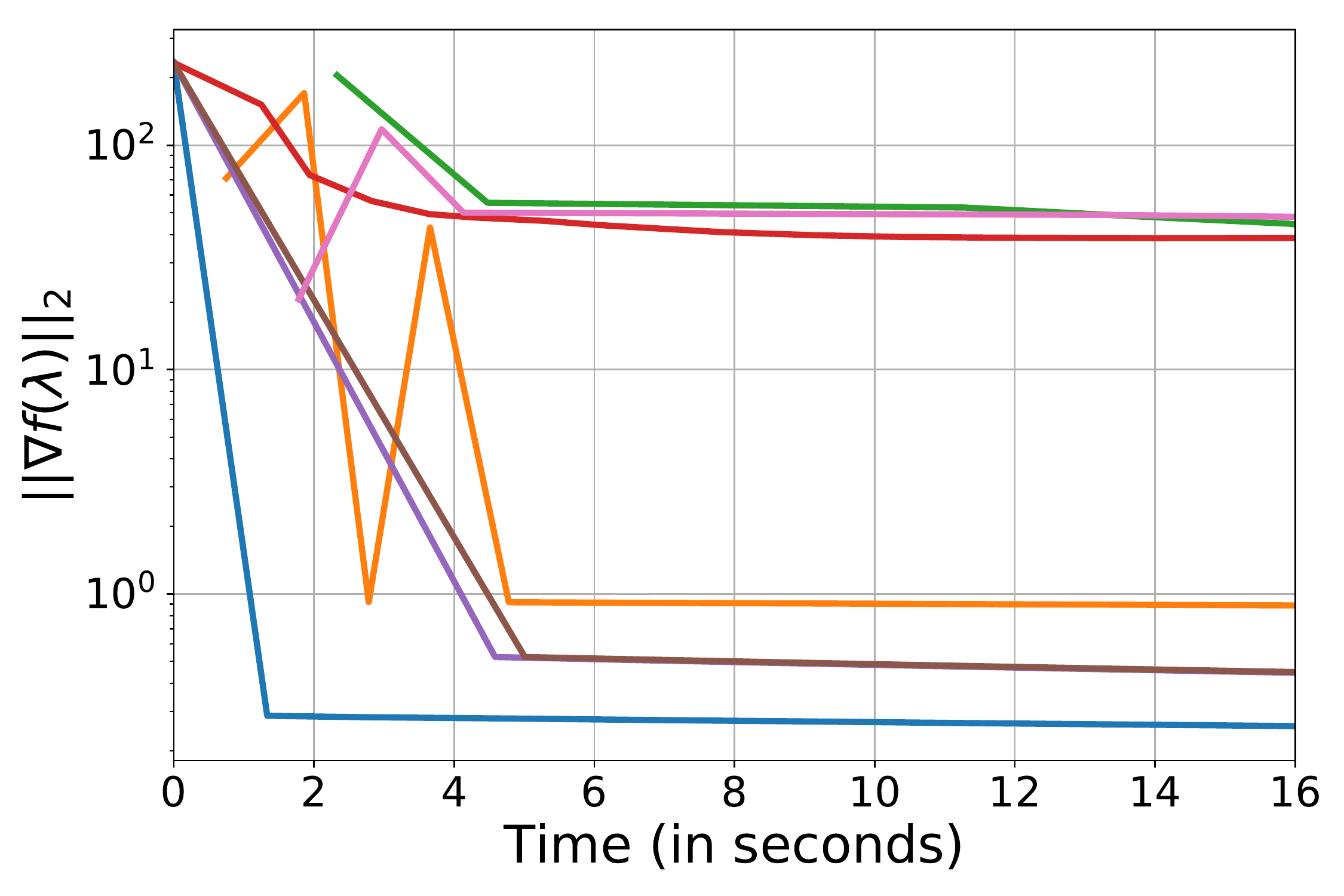}
\caption{News20}
\label{fig1.4}
\end{subfigure}
\begin{subfigure}[b]{0.32\textwidth}
\centering
\includegraphics[width=2in]{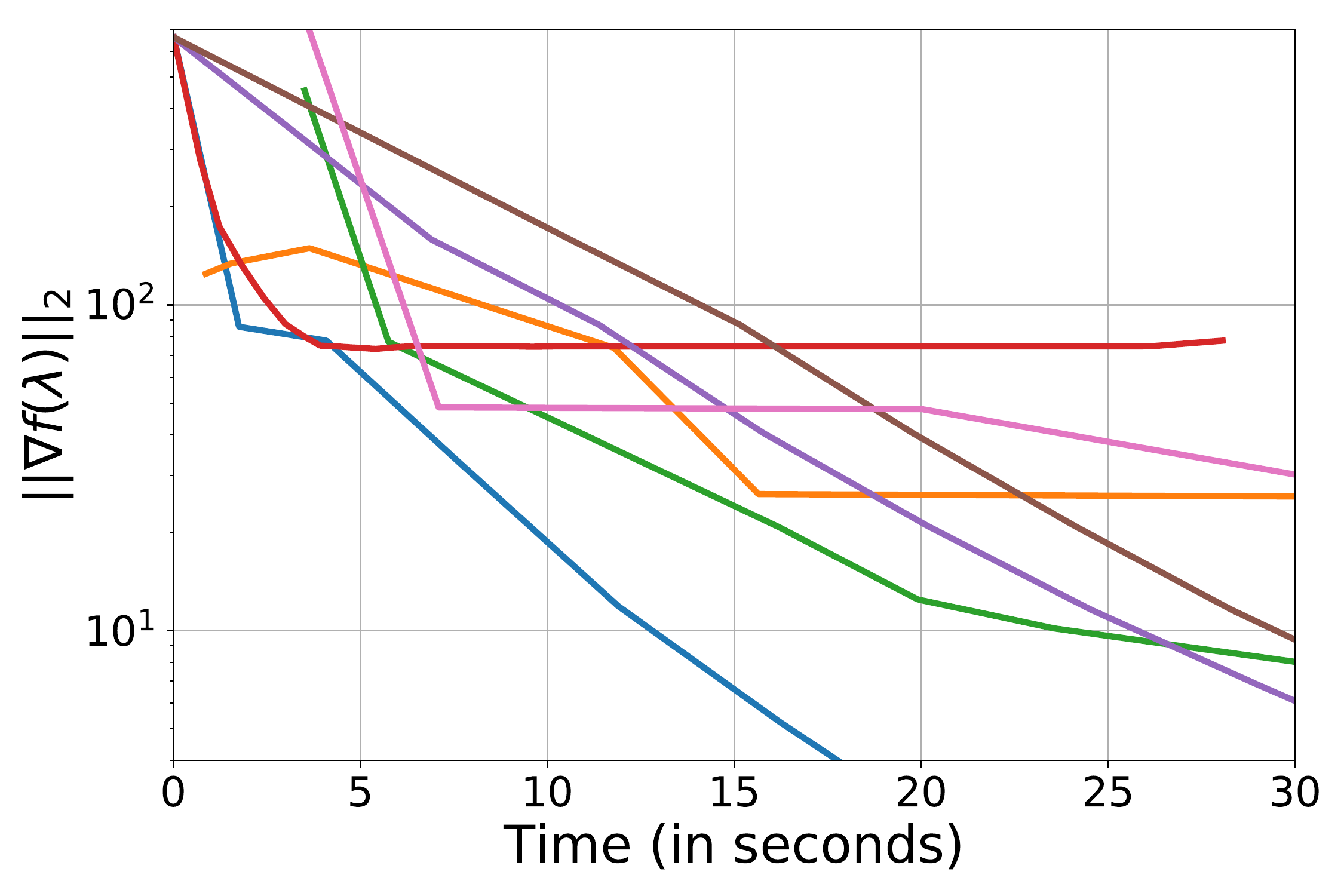}
\caption{Real-sim}
\end{subfigure}	
\begin{subfigure}[b]{0.32\textwidth}
\centering
\includegraphics[width=2in]{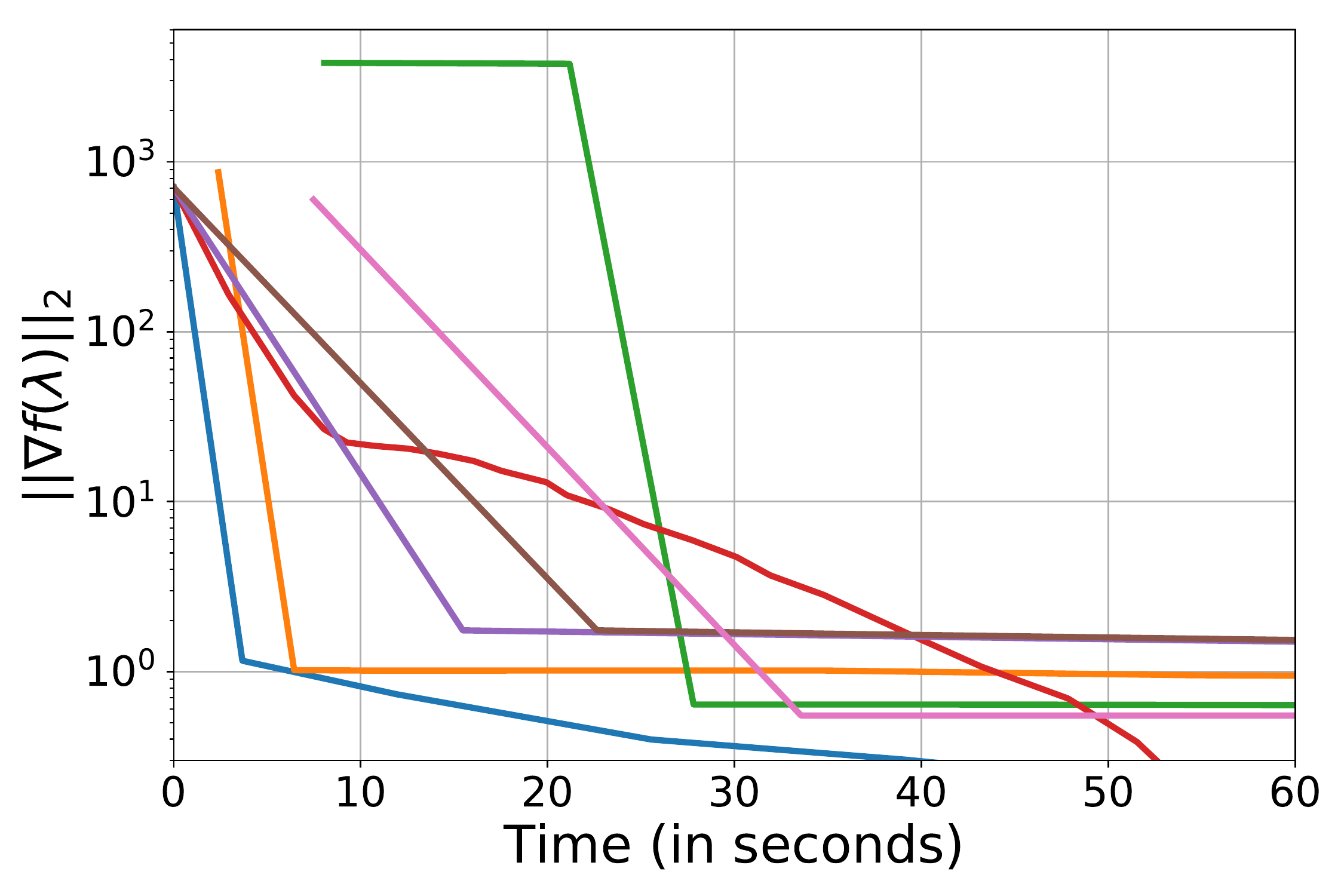}
\caption{Covtype}
\label{fig1.6}
\end{subfigure}
\caption{Comparison of different hyperparameter optimization algorithms for $l_2$-regularized logistic regression sharing the same legend.
 (a)-(c): Test error. (d)-(f): Suboptimality.
		(g)-(i): $\| \nabla f(\lambda)\|_2$. (Larger figures can be found in the supplement material.)
	}
	\label{figureLR}
\end{figure*}
\subsection{Deep Neural Networks}\label{secDNN}
 \vspace*{-8pt}

\noindent \textbf{Experimental setup:} We validate the advantages of HOZOG on optimizing learning rates of DNN which is much more complicated in both structure and training compared to $l_2$-regularized logistic regression.

Specifically, the training of modern DNN is usually an intriguing process, involving multiple heuristic hyperparameter schedules, \textit{e.g.} learning rate with exponential weight decay. Instead of intuitive settings, we propose to apply epoch-wise learning rates and jointly optimize these hyperparameters. The experiments are conducted on CIFAR-10 dataset with \num[group-separator={,}]{50000} samples. To demonstrate the scalability of HOZOG, three deep neural networks with various structure are used, including (1) two layers DNN (2-layer CNN) with convolutional, max pooling, and normalizing layers; (2) VGG-16~\citep{simonyan2014very}, (3) ResNet-152~\citep{he2016deep}.
The initialization of inner problem is randomized for different meta-iterations to avoid the potential dependence on the quirks of particular settings.
In detail, for all experiments we apply 50 meta-iterations and optimize inner problems using stochastic gradient descent, with batch size of 256. On CNN, 100 epochs for inner problem are used, which indicates 100 hyperparameters are involved.
On VGG-16, the original model takes $224\times 224$ images as inputs, and we adjust the size of the first fully-connected layer from $7\times 7$ convolution to $1\times 1$ to fit CIFAR-10 inputs. Here 20 epochs for inner are used.
On ResNet-152, similar processing is exploited and the inner epoch is 10. 

\noindent \textbf{Results and discussions:} The results are summarized in Figure \ref{figureCNN}.  The experimental results show that the learning rates computed by HOZOG achieve the lowest test error and the fastest descending speed compared to baselines on all tasks. Moreover, the proposed method requires much less time to attain the best hyperparameters, and tends to have smaller variances in gradients.  It is  noteworthy that,  some state-of-the-art hyperparameter optimization approaches (including HOAG, REV and FOR) are missing in this setting,  due to the algorithms of  REV and FOR are limited to smooth functions and the implementation of HOAG is limited to the  hyperparameter optimization problems with a small number of hyperparamters. However,  these difficulties are avoided by our HOZOG, which also demonstrates the flexibility of HOZOG. Moreover, as a brutal search method, the performance of RS is very unstable, which can be identified from the hyper-gradients. For BO and BOHB, the instability also exists, potentially due to the highly complexity of the network structure. Another noteworthy problem with respect to BO and BOHB is the computational overhead in sampling, which make the meta-iteration extremely time consuming, compared to other methods.

We observe that the difficulty of this problem mainly comes from model complexity, instead of hyper-parameter numbers. For CNN with 100 hyper-parameters, HOZOG shows advantages in both time and suboptimality, although baselines can also efficiently find a reasonable solution. For VGG-16 and ResNet-152, we notice that though the size of hyperparameters is reduced, it takes baselines longer time to find acceptable results. Instead, HOZOG still shows fast convergence empirically. This observation indicates that HOZOG is potentially more suitable for hyperparameter optimization in large DNN.




\begin{figure*}[t]
	\centering
	\begin{subfigure}[b]{0.32\textwidth}
		\centering
		\includegraphics[width=2in]{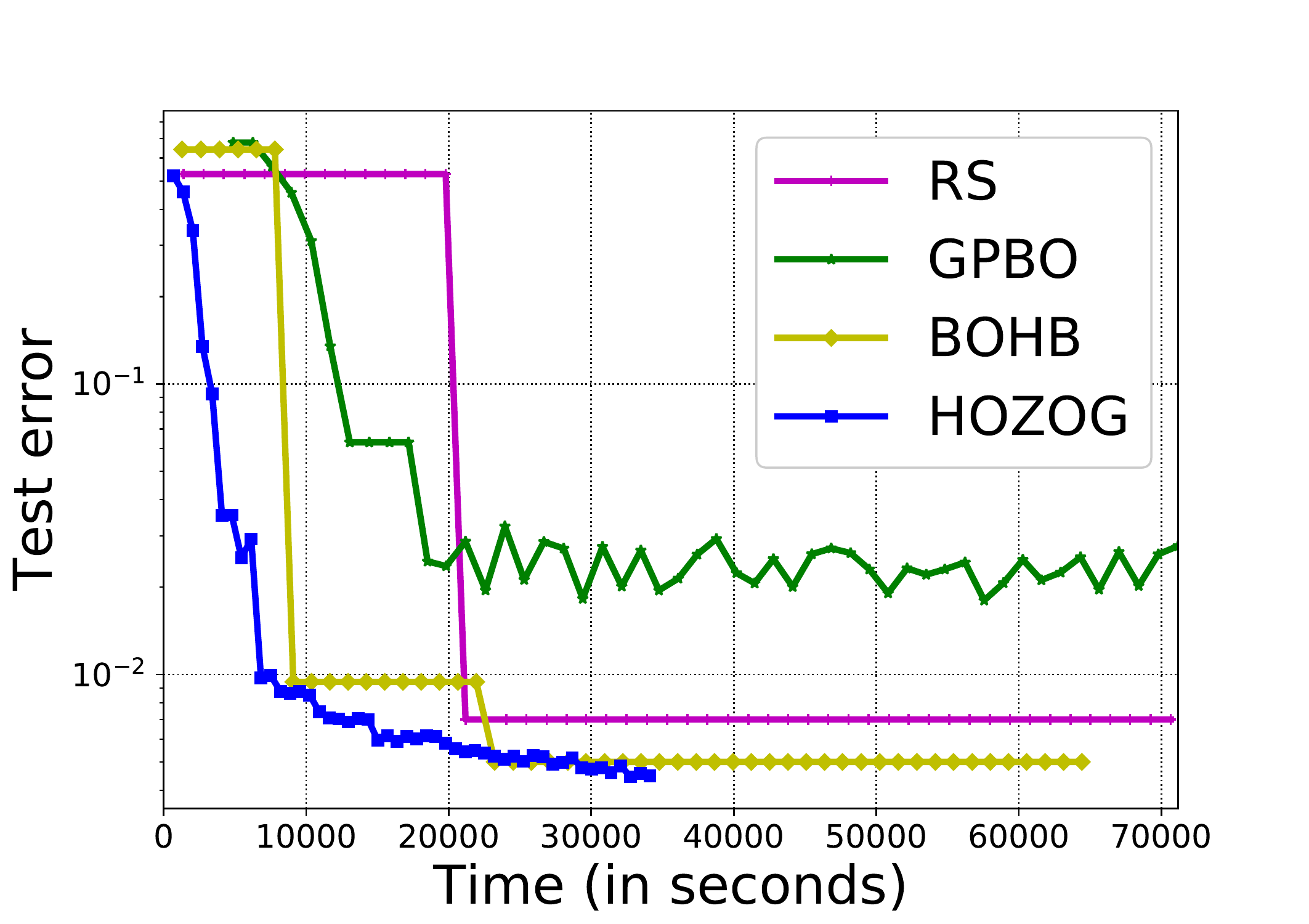}
		\caption{2-layer CNN}
	\end{subfigure}
		\begin{subfigure}[b]{0.32\textwidth}
		\centering
		\includegraphics[width=2in]{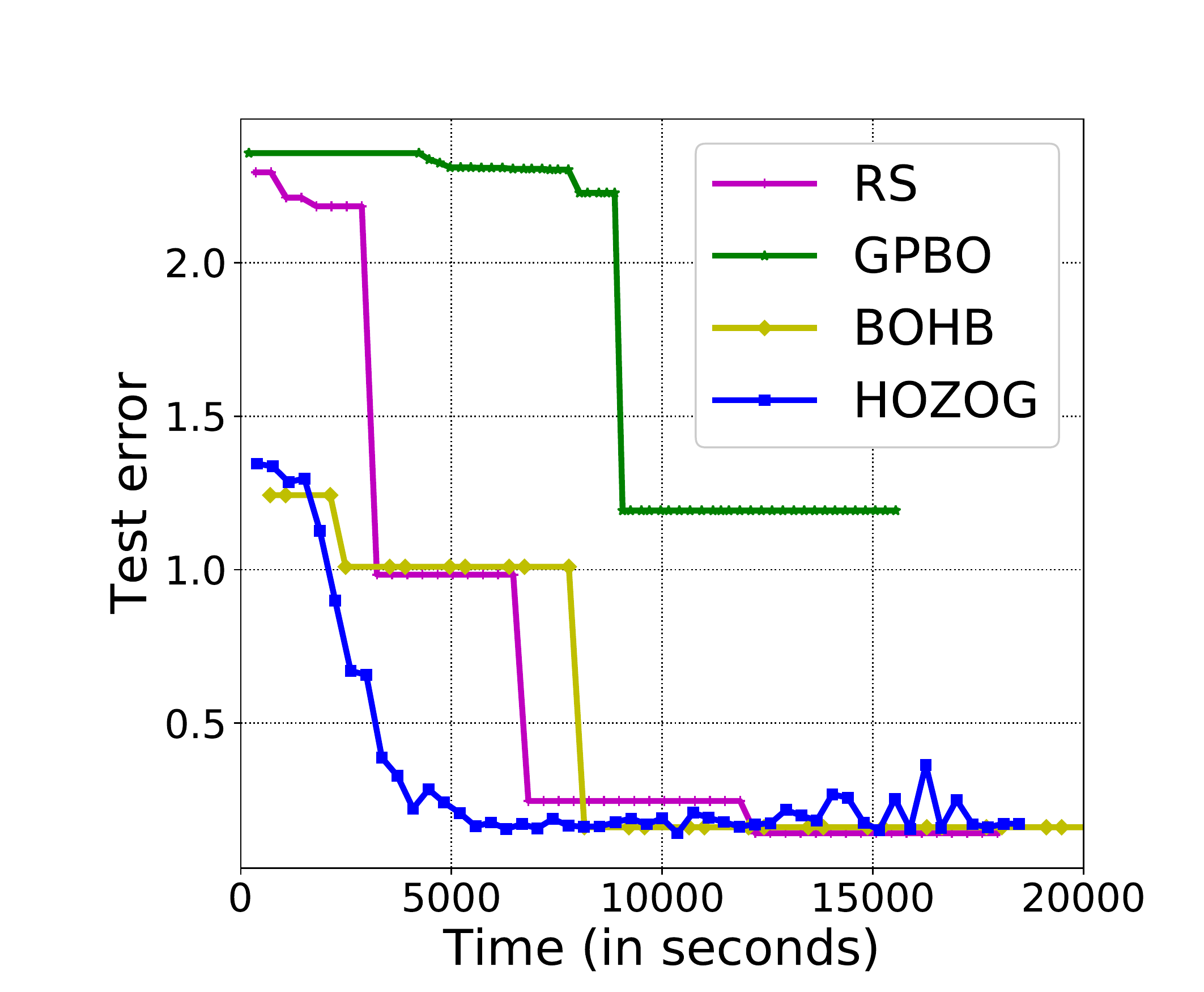}
		\caption{VGG-16}
	\end{subfigure}
		\begin{subfigure}[b]{0.32\textwidth}
		\centering
		\includegraphics[width=2in]{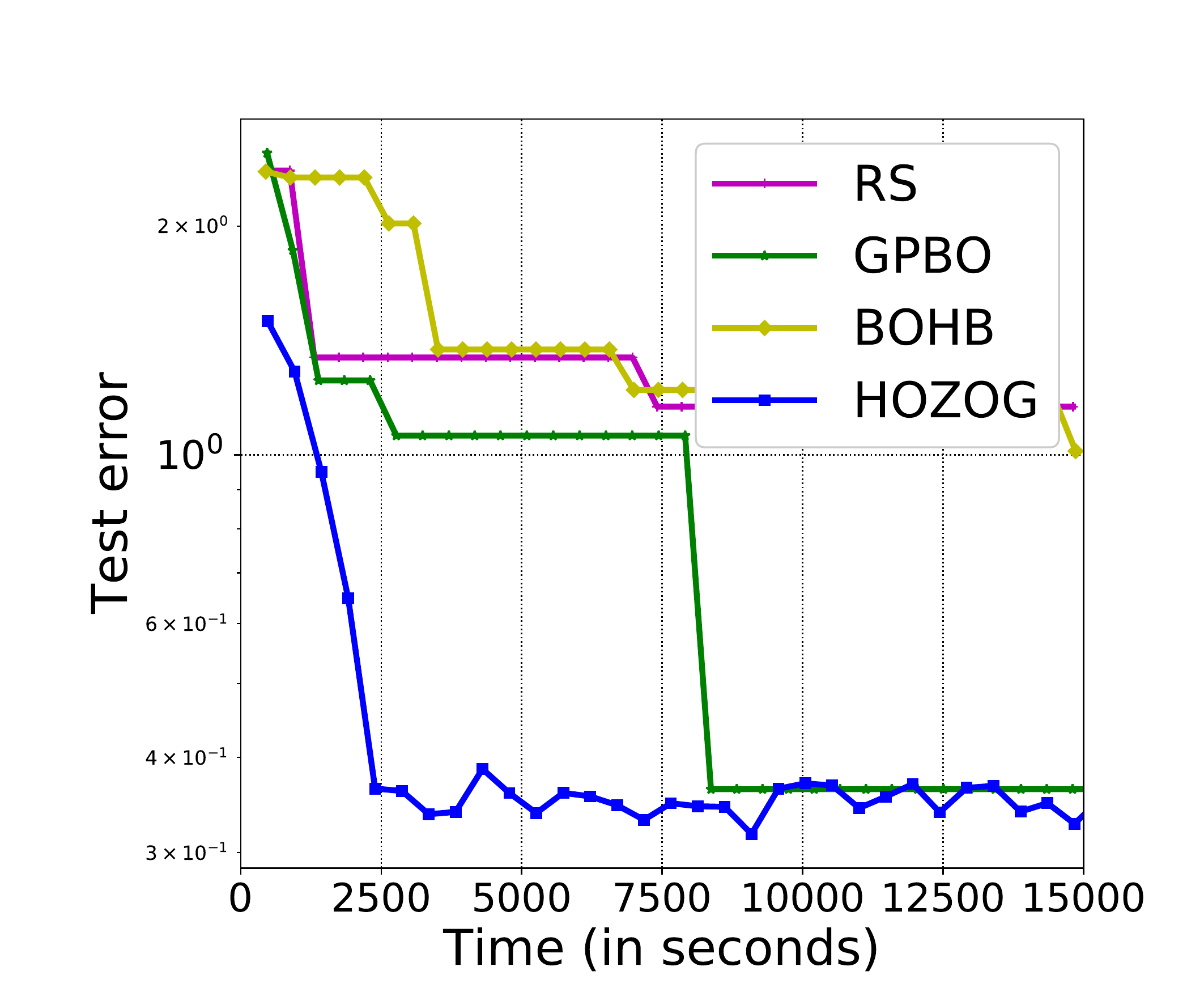}
		\caption{ResNet-152}
	\end{subfigure}
		\begin{subfigure}[b]{0.32\textwidth}
		\centering
		\includegraphics[width=2in]{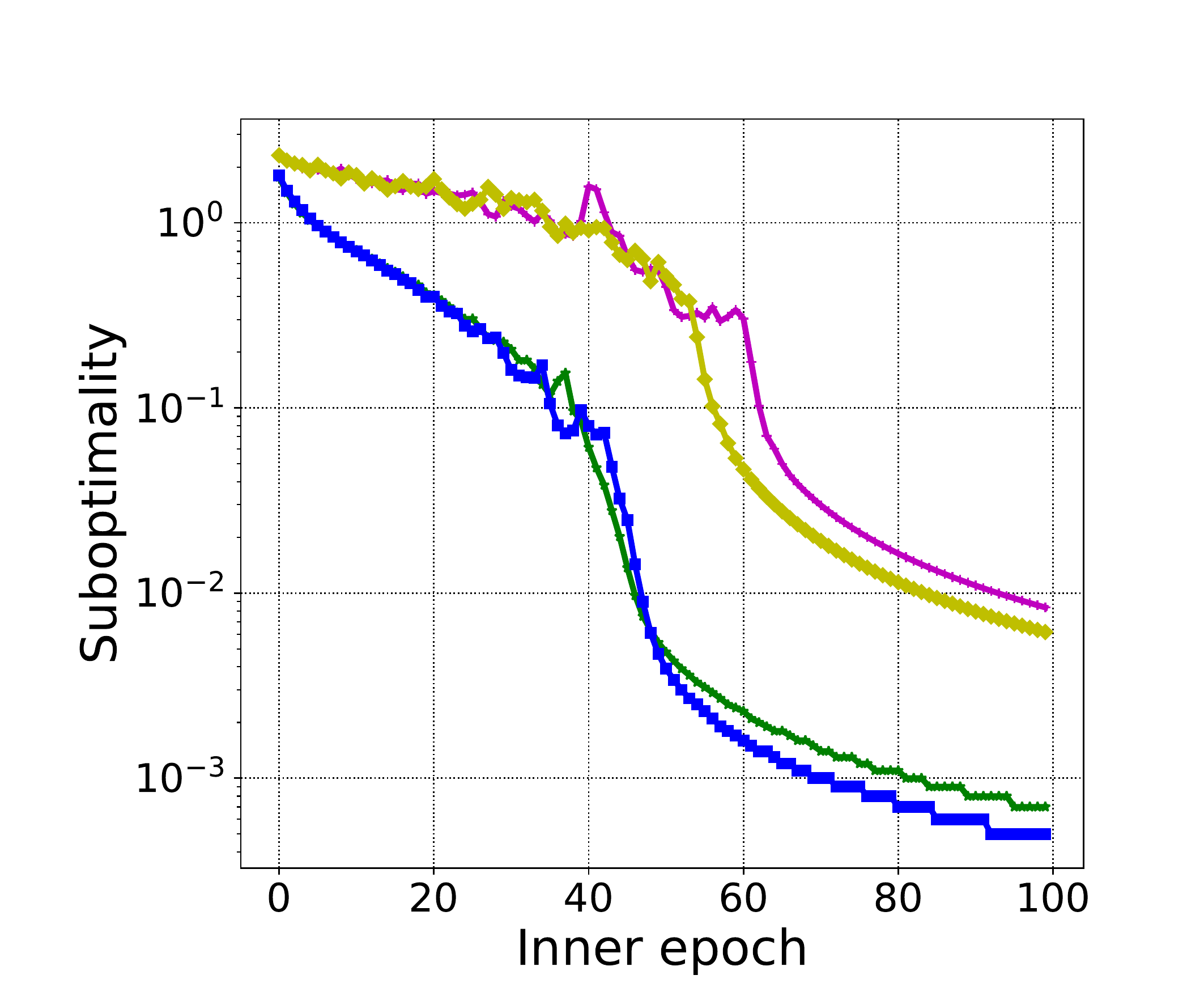}
		\caption{2-layer CNN}
	\end{subfigure}
	\begin{subfigure}[b]{0.32\textwidth}
		\centering
		\includegraphics[width=2in]{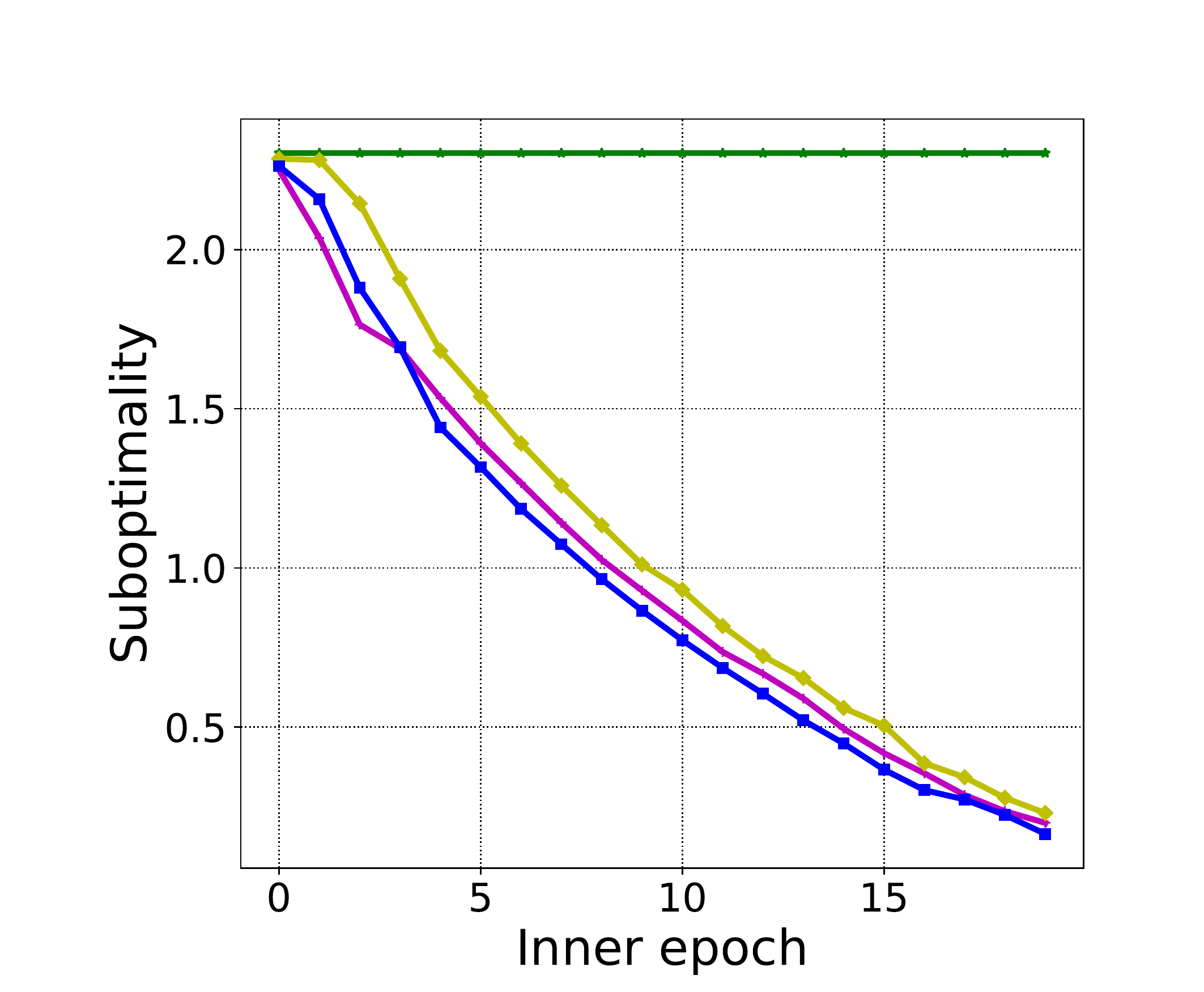}
		\caption{VGG-16}
	\end{subfigure}
	\begin{subfigure}[b]{0.32\textwidth}
		\centering
		\includegraphics[width=2in]{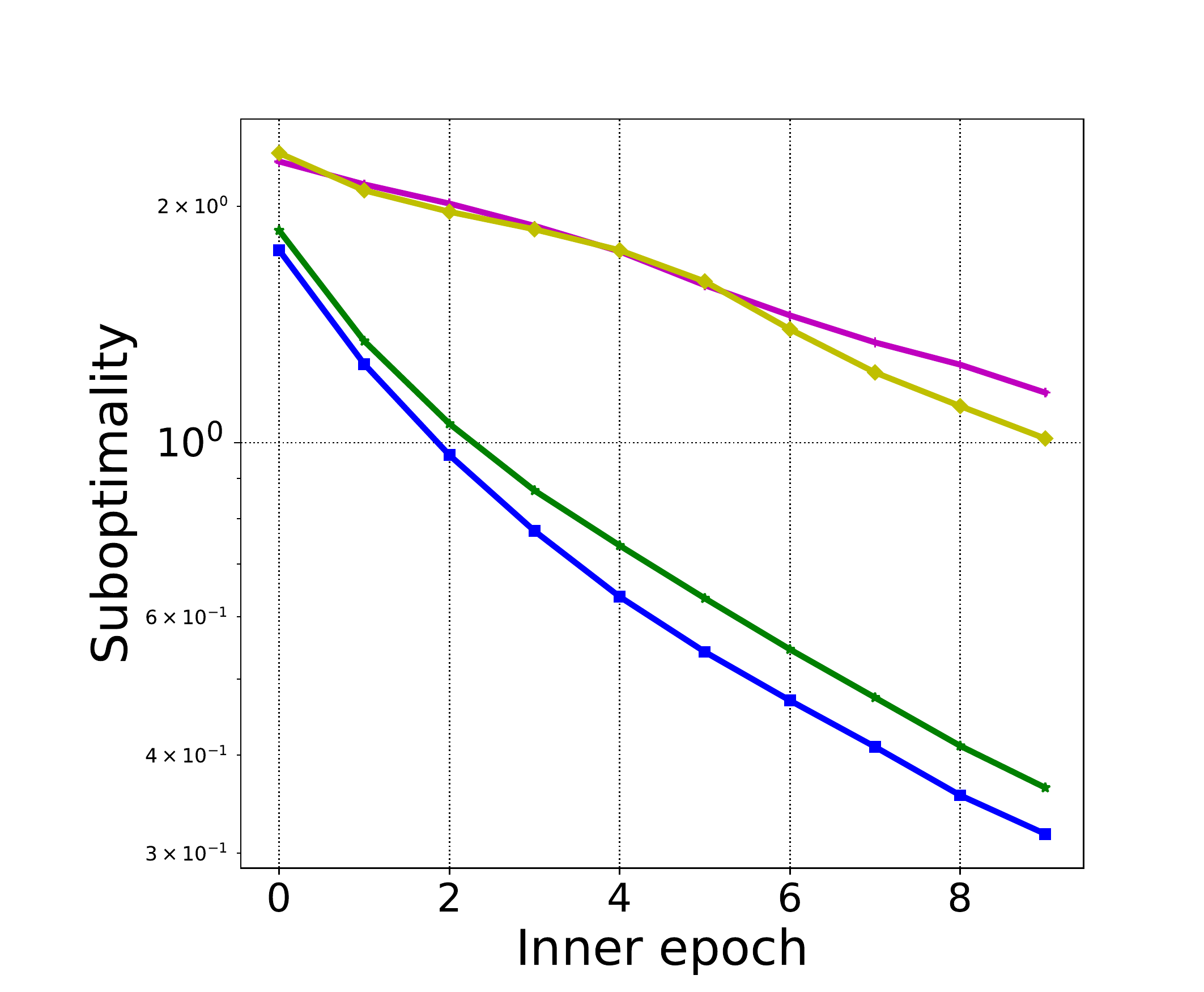}
		\caption{ResNet-152}
	\end{subfigure}
	\begin{subfigure}[b]{0.32\textwidth}
		\centering
		\includegraphics[width=2in]{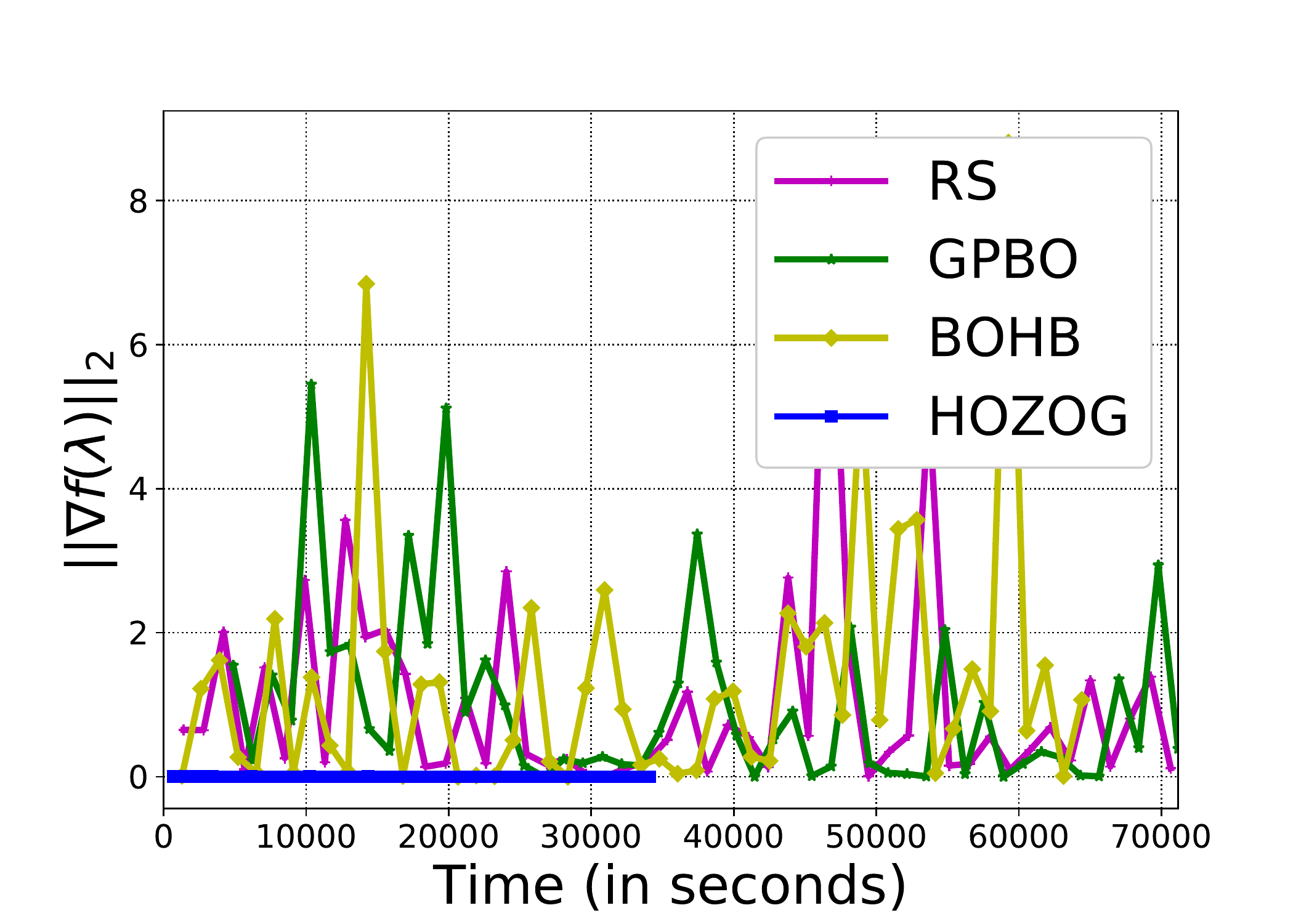}
		\caption{2-layer CNN}
	\end{subfigure}
		\begin{subfigure}[b]{0.32\textwidth}
		\centering
		\includegraphics[width=2in]{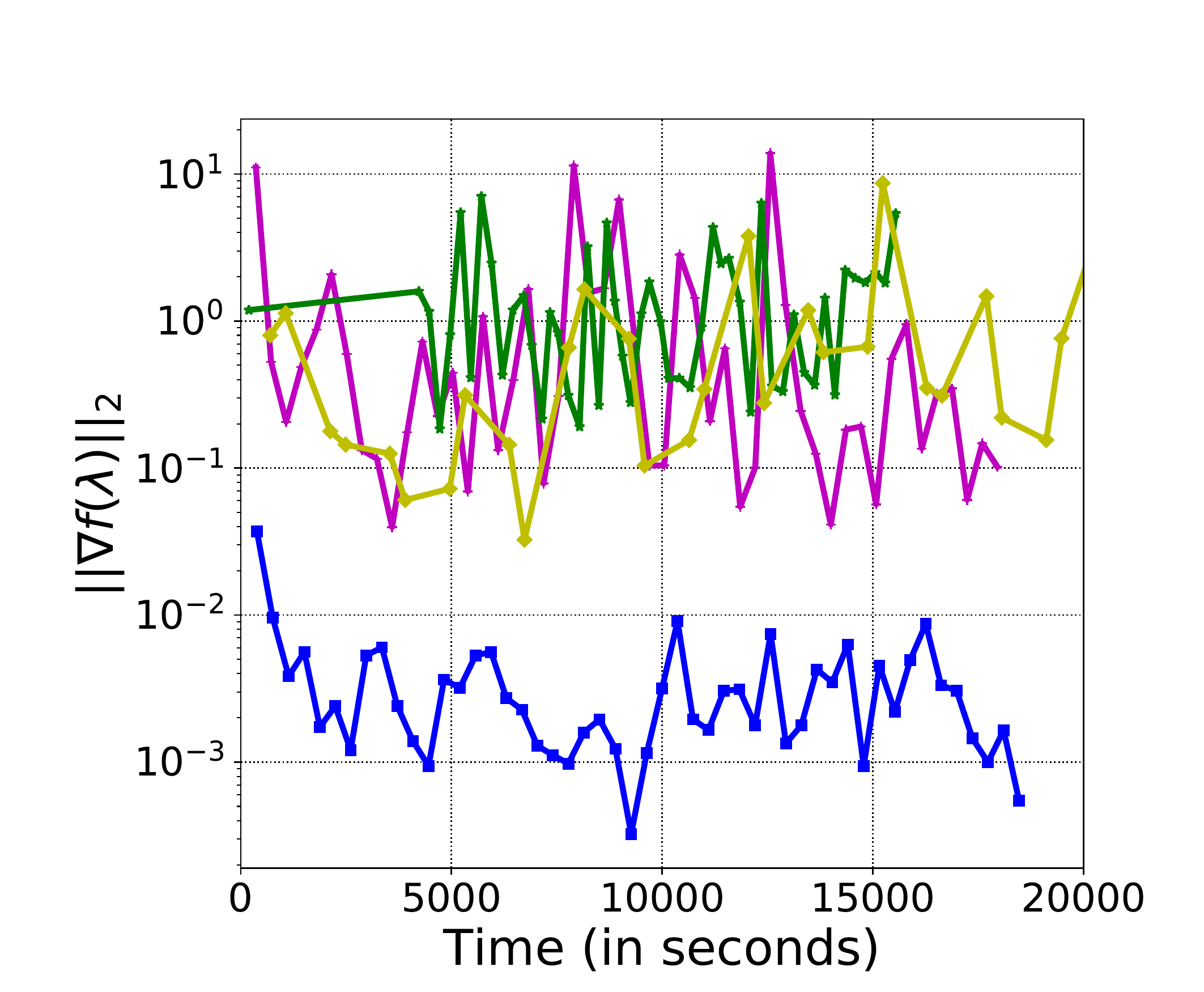}
		\caption{VGG-16}
	\end{subfigure}
		\begin{subfigure}[b]{0.32\textwidth}
		\centering
		\includegraphics[width=2in]{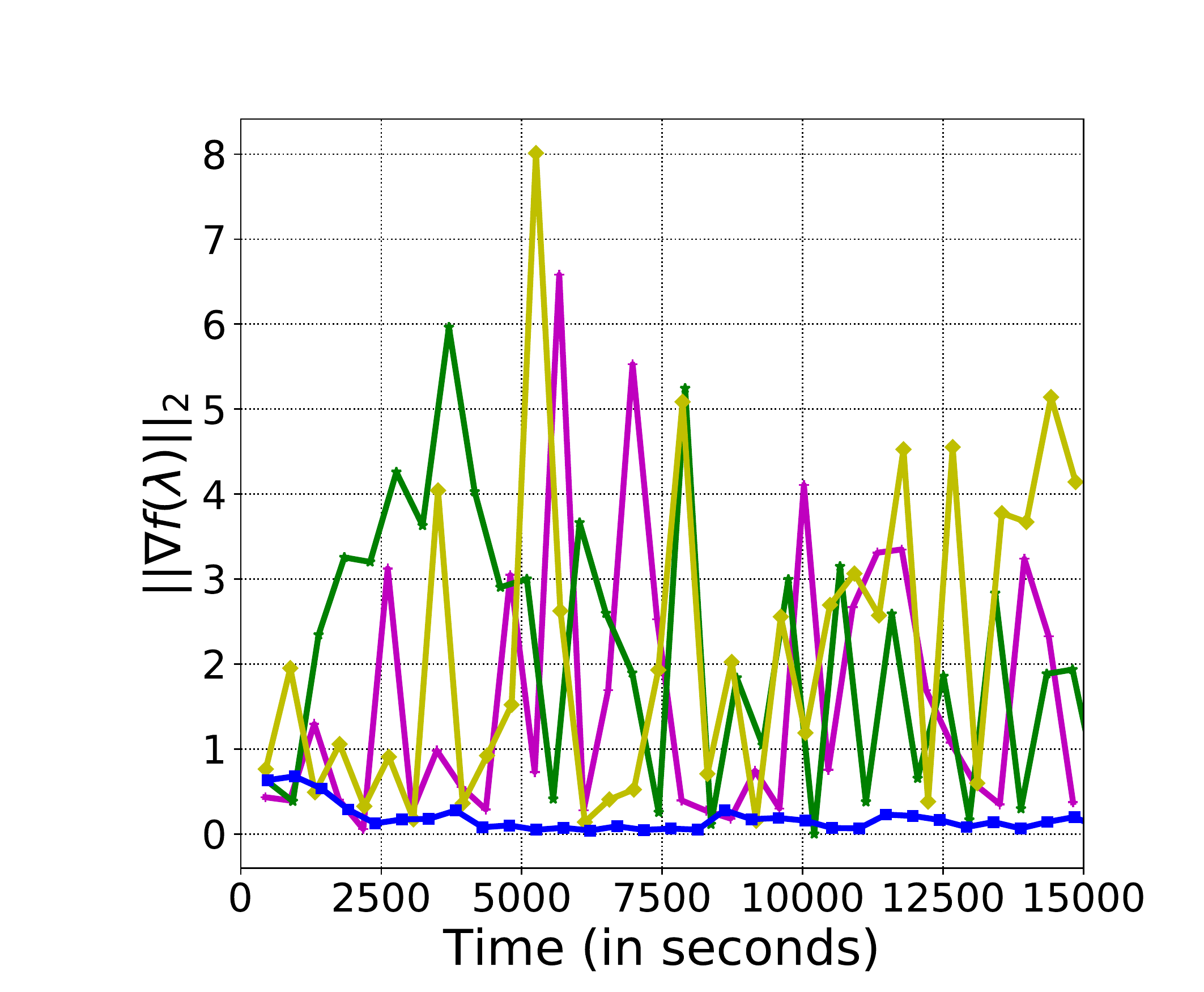}
		\caption{ResNet-152}
	\end{subfigure}
	\caption{Comparison of different hyperparameter optimization algorithms for 2-layer CNN, VGG-16 and ResNet-152 sharing the same legend. 	 (a)-(c): Test error. (d)-(f): Suboptimality.
		(g)-(i): $\| \nabla f(\lambda)\|_2$.
(Larger figures can be found in the supplement material.)	}
 \vspace*{-16pt}
	\label{figureCNN}
\end{figure*}

\subsection{Data Hyper-Cleaning}
\noindent \textbf{Experimental setup:} We  evaluate HOZOG on tuning the hyperparameters of data hyper-cleaning task.
Compared with the preceding problems, the data cleaning task is more challenging, since it has  more hyperparameters (hundreds or even  thousands). 

Assuming that we have a label noise dataset, with only limited  clean data provided. The data hyper-cleaning task is to allocate a hyperparameter weight $\lambda_i$ to a certain data point or a group of data points to counteract the influence of noisy samples. We split a certain data set into three subsets: ${\cal D}_{tr}$ of $N_{tr}$ training samples, ${\cal D}_{val}$ of $N_{val}$ validation samples and a test set ${\cal D}_{t}$ containing the $N_{t}$ samples.
We set random labels to $\lceil0.5*N_{tr}\rceil$ training examples, and select a random subset ${\cal D}_{f}$ from $ {\cal D}_{tr}$.

Similar to \cite{franceschi2017forward}, we considered a plain softmax regression model with parameters $W$ (weights) and $b$ (bias). The error of a model $(W, b)$ on an example $(x,y)$ was evaluated by using the cross-entropy $l(W, b, (x,y))$ both in the training objective function, $L$, and in the validation one, $E$. We added in $L$ an hyperparameter vector $\lambda\in\mathbb{R}^{N_{h}}$ that weights each group of examples in the training phase through sigmoid function, \emph{i.e.} $L(W,b)=\frac{1}{N_{tr}}\sum_{g\in \cal G}\sum_{i\in g}\textrm{sigmoid}(\lambda_g) l(W,b,(x_i,y_i))$, where $\cal G$ contain $N_h$ groups random select from ${\cal D}_{tr}$. Thus, we have the hyperparameter optimization problem as follows.
\begin{eqnarray}\label{data hyper}
\mathop{\arg\min}_{\lambda\in\mathbb{R}^{N_{h}}}  E(W(\lambda),b(\lambda)),
\quad	s.t. \ \ [W(\lambda),b(\lambda)] \approx \arg \min_{W,b} L(W,b)
\end{eqnarray}
\\
We instance two subset dataset for the MNIST dataset, with $N_{tr} = 5000$, $N_{val}=5000$, $N_{t}=10000$, $N_h=1250$ and $N_{tr} = 1000$, $N_{val}=1000$, $N_{t}=4000$, $N_h=500$.
We use a standard gradient descent method for the inner problem with fixed learning rate 0.05 and 4000 iteration.
 RS is used as baseline method, and BOHB and REV are used as comparison.

\begin{figure*}[t]
	\centering
	\begin{subfigure}[b]{0.32\textwidth}
		\centering
		\includegraphics[width=2in]{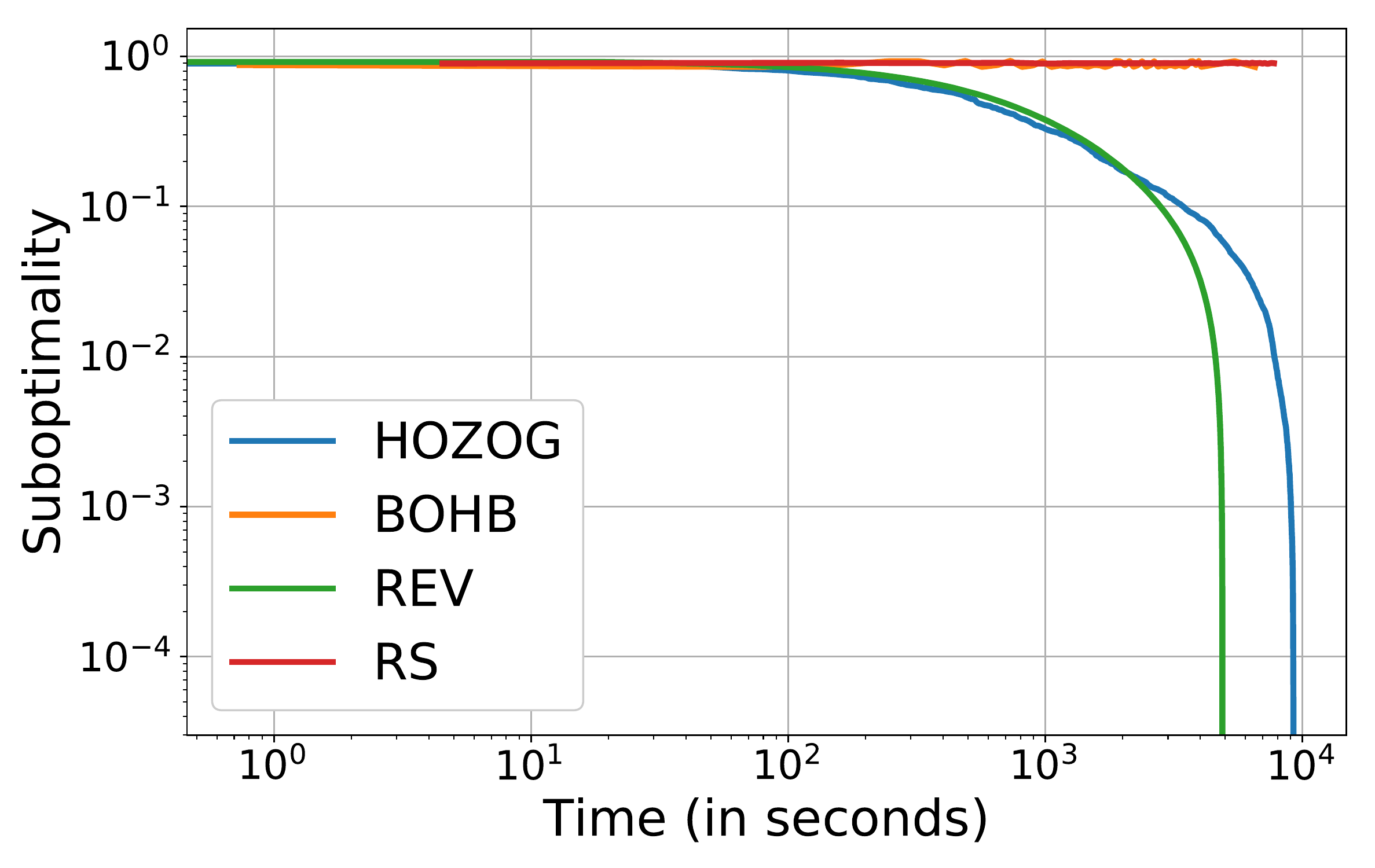}
		\caption{500 HP}
	\end{subfigure}
	\begin{subfigure}[b]{0.32\textwidth}
		\centering
		\includegraphics[width=2in]{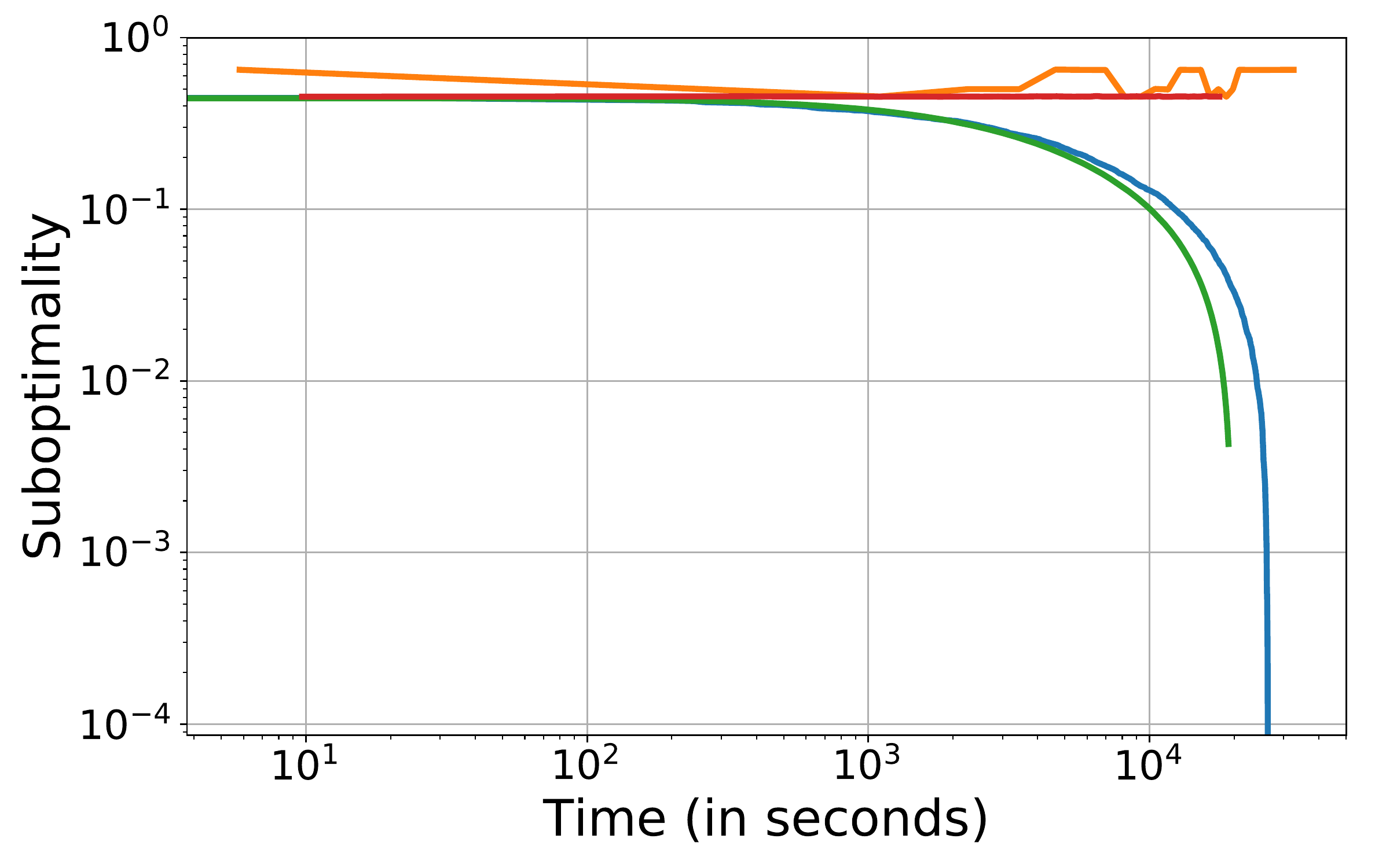}
		\caption{1250 HP}
	\end{subfigure}
	\begin{subfigure}[b]{0.32\textwidth}
		\centering
		\includegraphics[width=2in]{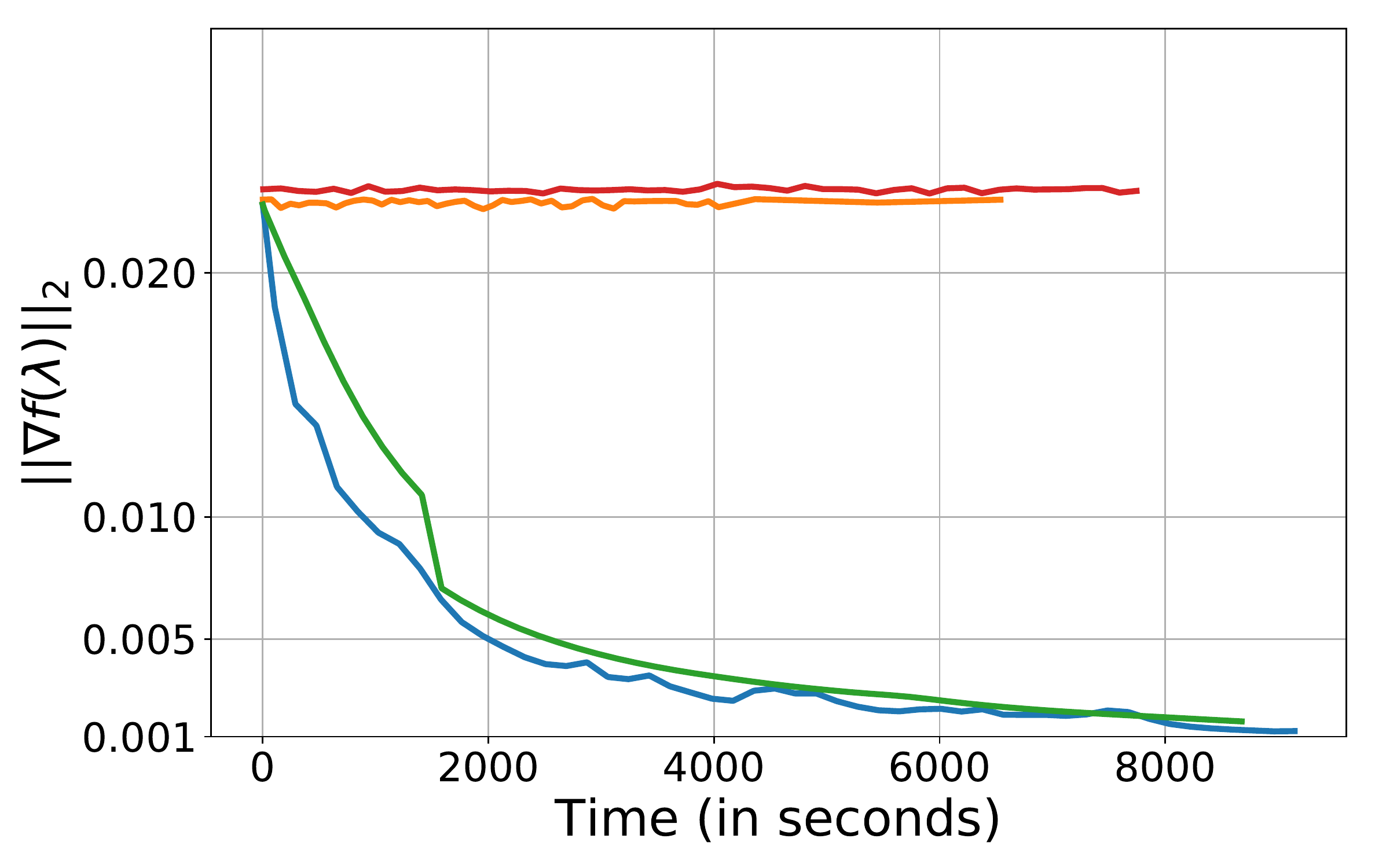}
		\caption{500 HP}
	\end{subfigure}
	\begin{subfigure}[b]{0.32\textwidth}
		\centering
		\includegraphics[width=2in]{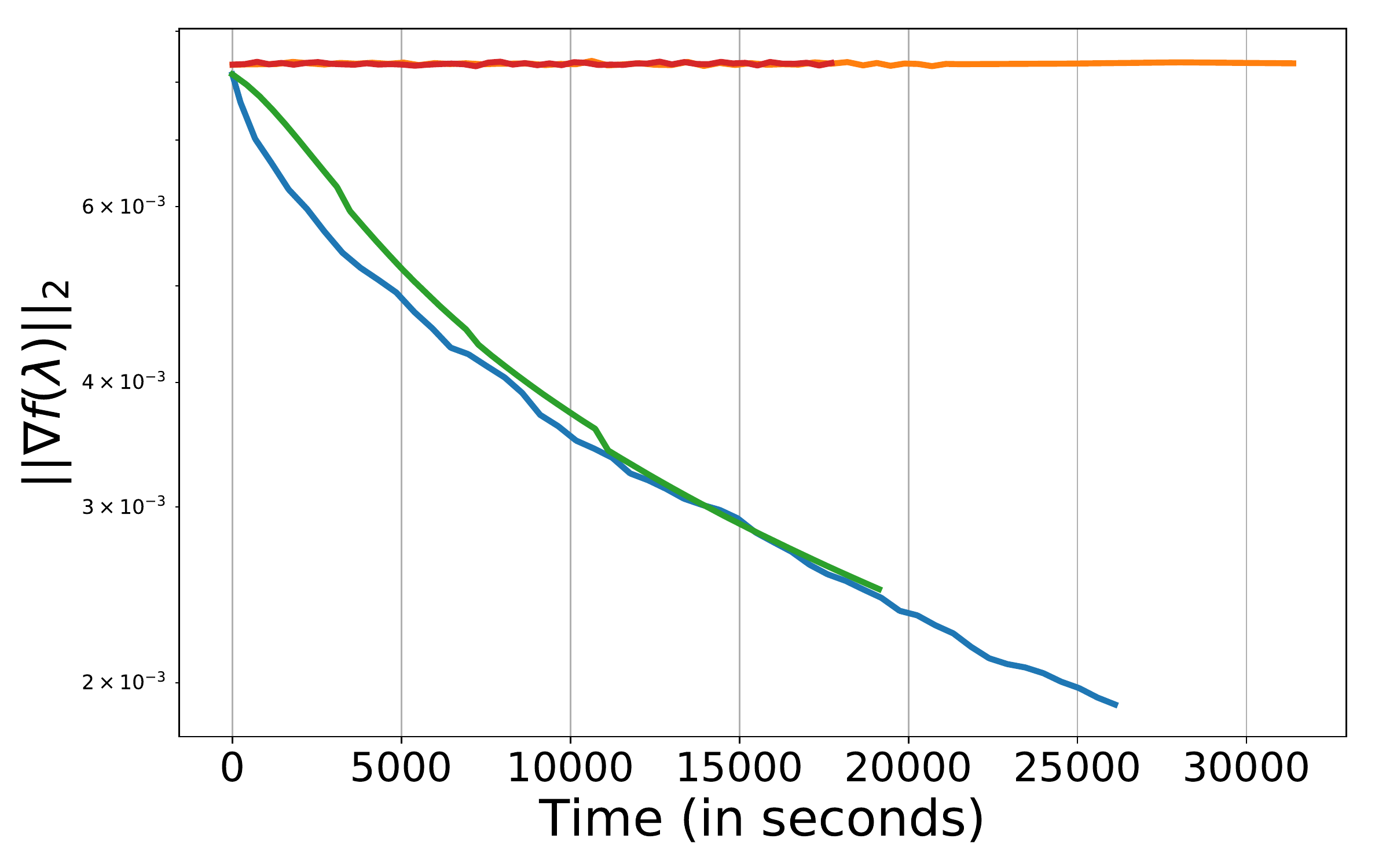}
		\caption{1250 HP}
	\end{subfigure}
	\begin{subfigure}[b]{0.32\textwidth}
		\centering
		\includegraphics[width=2in]{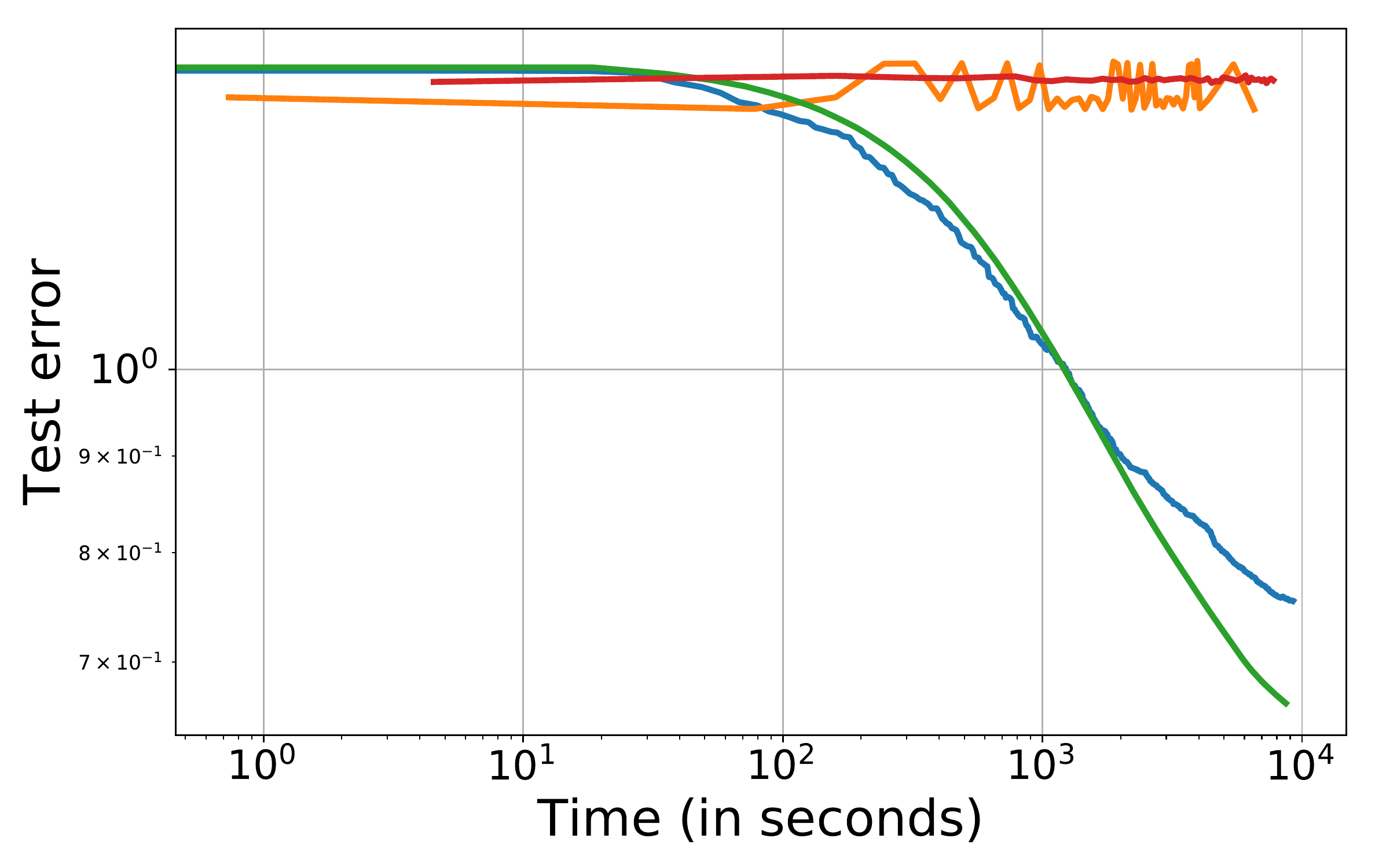}
		\caption{500 HP}
	\end{subfigure}
	\begin{subfigure}[b]{0.32\textwidth}
		\centering
		\includegraphics[width=2in]{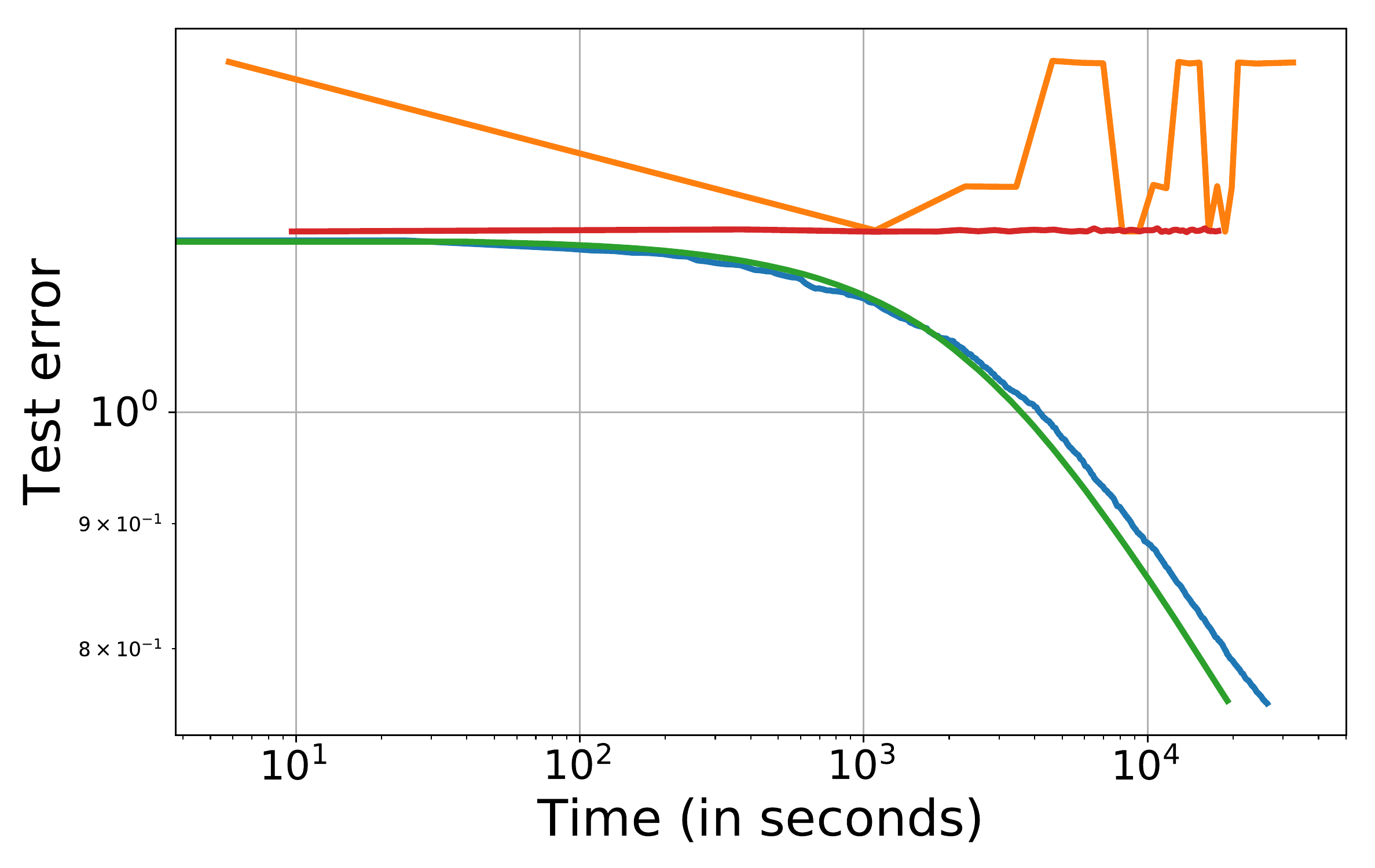}
		\caption{1250 HP}
	\end{subfigure}
	\caption{Comparison of different hyperparameter optimization algorithms for data hyper-cleaning sharing the same legend, where ``HP'' is  the abbreviation of  hyperparameters
		(a)-(b): Suboptimality. (c)-(d): $\| \nabla f(\lambda)\|_2$.
		(e)-(f):  Test error. (Larger figures can be found in the supplement material.)
	}
	\label{figureCleaner}
\end{figure*}

\noindent \textbf{Results and discussions:} Figure \ref{figureCleaner} presents the results of HOZOG, BOHB, REV and RS for data hyper-cleaning. Note that the  methods of GPBO, FOR and HOAG are missing here, because the hyperparameter size is beyond the capability of their implementations. The results show that HOZOG can beat RS and BOHB easily, while not perform completely as good as REV in the long run. This is because REV is an exact gradient method whose convergence rate is faster than the one of zeroth-order gradient method (\textit{i.e.}, HOZOG) by a constant whose value is depending  on $p$ \citep{nesterov2017random}. However, computing the exact gradients in REV is costly. Specifically,  REV takes about 40 seconds to finish the computation of one hyper-gradient under the  setting of 1250 hyperparameters, which is only about 24 seconds for HOZOG. This is the reason why our method converges faster than REV in the early stage of training.  Importantly, the application scenarios of REV are limited to smooth functions, \emph{e.g.},  not suitable for the experimental settings of convolutional neural networks and deeper neural networks. However, our HOZOG can be utilized to  a broader class of functions (\textit{i.e.}, continuous functions).


\subsection{Discussion: Importance of HOZOG}
The experimental results show that the black-box optimization methods  have a weak performance for the high-dimensional  hyperparameter optimization problems which is also verified in a large number of existing references \citep{brochu2010tutorial,snoek2012practical}, while they have the advantages of \textit{simplicity} and \textit{flexibility}.  On the other hand, the existing gradient-based methods \citep{franceschi2017forward,franceschi2018bilevel} need experienced researchers to provide a customized program against  the optimization algorithm and sometime it would fail, while they have the advantages of \textit{scalability} and \textit{efficiency}.
HOZOG inherits all the benefits  from  both approaches in that, the gradients are computed in a black-box manner, while the hyperparameter search is accomplished via gradient descent.  Especially, for high-dimensional  hyperparameter optimization problems which have no customized RFHO algorithm,  HOZOG currently  is  the only choice for this kind of problems to the best of our knowledge.

\section{Conclusion}
\textit{Effectiveness, efficiency, scalability, simplicity and flexibility} (\textit{i.e.}, E2S2F) are  important evaluation criteria for hyperparameter optimization methods. In this paper, we  proposed a new  hyperparameter optimization paradigm with zeroth-order hyper-gradients (HOZOG) which  is the first  method having all these benefits to the best of our knowledge.  We proved the feasibility
of using HOZOG to achieve hyperparameter optimization under the condition of  Lipschitz continuity. The experimental  results on three representative  hyperparameter (the size is from 1 to \num[group-separator={,}]{1250}) optimization tasks not only verify the result in the  feasibility analysis, but also  demonstrate the benefits of HOZOG  in terms of E2S2F, compared with the  state-of-the-art hyperparameter optimization methods.


\section*{Acknowledgments}
We thank the anonymous reviewers and Dr. Frank Nussbaum for their careful reading of our manuscript and their many insightful comments and suggestions.



%
%
%

\bibliography{sample}

\begin{thebibliography}{36}
\providecommand{\natexlab}[1]{#1}
\providecommand{\url}[1]{\texttt{#1}}
\expandafter\ifx\csname urlstyle\endcsname\relax
  \providecommand{\doi}[1]{doi: #1}\else
  \providecommand{\doi}{doi: \begingroup \urlstyle{rm}\Url}\fi

\bibitem[Allen-Zhu and Hazan(2016)]{allen2016variance}
Zeyuan Allen-Zhu and Elad Hazan.
\newblock Variance reduction for faster non-convex optimization.
\newblock In \emph{International conference on machine learning}, pages
  699--707, 2016.

\bibitem[Bergstra and Bengio(2012)]{bergstra2012random}
James Bergstra and Yoshua Bengio.
\newblock Random search for hyper-parameter optimization.
\newblock \emph{Journal of Machine Learning Research}, 13\penalty0
  (Feb):\penalty0 281--305, 2012.

\bibitem[Bredies and Lorenz(2007)]{bredies2007iterative}
Kristian Bredies and Dirk~A Lorenz.
\newblock Iterative soft-thresholding converges linearly.
\newblock \emph{arXiv preprint arXiv:0709.1598}, 2007.

\bibitem[Brochu et~al.(2010)Brochu, Cora, and De~Freitas]{brochu2010tutorial}
Eric Brochu, Vlad~M Cora, and Nando De~Freitas.
\newblock A tutorial on bayesian optimization of expensive cost functions, with
  application to active user modeling and hierarchical reinforcement learning.
\newblock \emph{arXiv preprint arXiv:1012.2599}, 2010.

\bibitem[Cui et~al.(2017)Cui, Hong, and Liu]{Cui2017Strong}
Jianbo Cui, Jialin Hong, and Zhihui Liu.
\newblock Strong convergence rate of finite difference approximations for
  stochastic cubic schr\"{o}dinger equations ??
\newblock \emph{Journal of Differential Equations}, 263:\penalty0
  S002203961730253X, 2017.

\bibitem[Defazio et~al.(2014)Defazio, Bach, and
  Lacoste-Julien]{defazio2014saga}
Aaron Defazio, Francis Bach, and Simon Lacoste-Julien.
\newblock Saga: A fast incremental gradient method with support for
  non-strongly convex composite objectives.
\newblock In \emph{Advances in neural information processing systems}, pages
  1646--1654, 2014.

\bibitem[Falkner et~al.(2018)Falkner, Klein, and Hutter]{falkner2018bohb}
Stefan Falkner, Aaron Klein, and Frank Hutter.
\newblock Bohb: Robust and efficient hyperparameter optimization at scale.
\newblock In \emph{International Conference on Machine Learning}, pages
  1436--1445, 2018.

\bibitem[Fang et~al.(2018)Fang, Li, Lin, and Zhang]{fang2018spider}
Cong Fang, Chris~Junchi Li, Zhouchen Lin, and Tong Zhang.
\newblock Spider: Near-optimal non-convex optimization via stochastic
  path-integrated differential estimator.
\newblock In \emph{Advances in Neural Information Processing Systems}, pages
  689--699, 2018.

\bibitem[Federer(2014)]{federer2014geometric}
Herbert Federer.
\newblock \emph{Geometric measure theory}.
\newblock Springer, 2014.

\bibitem[Franceschi et~al.(2017)Franceschi, Donini, Frasconi, and
  Pontil]{franceschi2017forward}
Luca Franceschi, Michele Donini, Paolo Frasconi, and Massimiliano Pontil.
\newblock Forward and reverse gradient-based hyperparameter optimization.
\newblock In \emph{Proceedings of the 34th International Conference on Machine
  Learning-Volume 70}, pages 1165--1173. JMLR. org, 2017.

\bibitem[Franceschi et~al.(2018)Franceschi, Frasconi, Salzo, Grazzi, and
  Pontil]{franceschi2018bilevel}
Luca Franceschi, Paolo Frasconi, Saverio Salzo, Riccardo Grazzi, and
  Massimiliano Pontil.
\newblock Bilevel programming for hyperparameter optimization and
  meta-learning.
\newblock In \emph{International Conference on Machine Learning}, pages
  1563--1572, 2018.

\bibitem[Ghadimi and Lan(2013)]{ghadimi2013stochastic}
Saeed Ghadimi and Guanghui Lan.
\newblock Stochastic first-and zeroth-order methods for nonconvex stochastic
  programming.
\newblock \emph{SIAM Journal on Optimization}, 23\penalty0 (4):\penalty0
  2341--2368, 2013.

\bibitem[Goodfellow et~al.(2016)Goodfellow, Bengio, and
  Courville]{goodfellow2016deep}
Ian Goodfellow, Yoshua Bengio, and Aaron Courville.
\newblock \emph{Deep learning}.
\newblock MIT press, 2016.

\bibitem[Gu and Huo(2018)]{gu2018asynchronous}
Bin Gu and Zhouyuan Huo.
\newblock Asynchronous doubly stochastic group regularized learning.
\newblock In \emph{International Conference on Artificial Intelligence and
  Statistics (AISTATS 2018)}, 2018.

\bibitem[Gu and Ling(2015)]{gu2015new}
Bin Gu and Charles Ling.
\newblock A new generalized error path algorithm for model selection.
\newblock In \emph{International Conference on Machine Learning}, pages
  2549--2558, 2015.

\bibitem[Gu et~al.(2018)Gu, Huo, Deng, and Huang]{gu2018faster}
Bin Gu, Zhouyuan Huo, Cheng Deng, and Heng Huang.
\newblock Faster derivative-free stochastic algorithm for shared memory
  machines.
\newblock In \emph{International Conference on Machine Learning}, pages
  1807--1816, 2018.

\bibitem[He et~al.(2016)He, Zhang, Ren, and Sun]{he2016deep}
Kaiming He, Xiangyu Zhang, Shaoqing Ren, and Jian Sun.
\newblock Deep residual learning for image recognition.
\newblock In \emph{Proceedings of the IEEE conference on computer vision and
  pattern recognition}, pages 770--778, 2016.

\bibitem[Kingma and Ba(2014)]{kingma2014adam}
Diederik~P Kingma and Jimmy Ba.
\newblock Adam: A method for stochastic optimization.
\newblock \emph{arXiv preprint arXiv:1412.6980}, 2014.

\bibitem[Krizhevsky et~al.(2012)Krizhevsky, Sutskever, and
  Hinton]{krizhevsky2012imagenet}
Alex Krizhevsky, Ilya Sutskever, and Geoffrey~E Hinton.
\newblock Imagenet classification with deep convolutional neural networks.
\newblock In \emph{Advances in neural information processing systems}, pages
  1097--1105, 2012.

\bibitem[Li et~al.(2017)Li, Jamieson, DeSalvo, Rostamizadeh, and
  Talwalkar]{li2017hyperband}
Lisha Li, Kevin Jamieson, Giulia DeSalvo, Afshin Rostamizadeh, and Ameet
  Talwalkar.
\newblock Hyperband: A novel bandit-based approach to hyperparameter
  optimization.
\newblock \emph{The Journal of Machine Learning Research}, 18\penalty0
  (1):\penalty0 6765--6816, 2017.

\bibitem[Lian et~al.(2016)Lian, Zhang, Hsieh, Huang, and
  Liu]{lian2016comprehensive}
Xiangru Lian, Huan Zhang, Cho-Jui Hsieh, Yijun Huang, and Ji~Liu.
\newblock A comprehensive linear speedup analysis for asynchronous stochastic
  parallel optimization from zeroth-order to first-order.
\newblock In \emph{Advances in Neural Information Processing Systems}, pages
  3054--3062, 2016.

\bibitem[Liu and Nocedal(1989)]{liu1989limited}
Dong~C Liu and Jorge Nocedal.
\newblock On the limited memory bfgs method for large scale optimization.
\newblock \emph{Mathematical programming}, 45\penalty0 (1-3):\penalty0
  503--528, 1989.

\bibitem[Maclaurin et~al.(2015)Maclaurin, Duvenaud, and
  Adams]{maclaurin2015gradient}
Dougal Maclaurin, David Duvenaud, and Ryan Adams.
\newblock Gradient-based hyperparameter optimization through reversible
  learning.
\newblock In \emph{International Conference on Machine Learning}, pages
  2113--2122, 2015.

\bibitem[Nesterov and Spokoiny(2017)]{nesterov2017random}
Yurii Nesterov and Vladimir Spokoiny.
\newblock Random gradient-free minimization of convex functions.
\newblock \emph{Foundations of Computational Mathematics}, 17\penalty0
  (2):\penalty0 527--566, 2017.

\bibitem[Pedregosa(2016)]{pedregosa2016hyperparameter}
Fabian Pedregosa.
\newblock Hyperparameter optimization with approximate gradient.
\newblock In \emph{International Conference on Machine Learning}, pages
  737--746, 2016.

\bibitem[Reddi et~al.(2016)Reddi, Hefny, Sra, Poczos, and
  Smola]{reddi2016stochastic}
Sashank~J Reddi, Ahmed Hefny, Suvrit Sra, Barnabas Poczos, and Alex Smola.
\newblock Stochastic variance reduction for nonconvex optimization.
\newblock In \emph{International conference on machine learning}, pages
  314--323, 2016.

\bibitem[Rudin(1976)]{Rudin1976Principles}
Walter Rudin.
\newblock \emph{Principles of mathematical analysis}.
\newblock 1976.

\bibitem[Rudin et~al.(1964)]{rudin1964principles}
Walter Rudin et~al.
\newblock \emph{Principles of mathematical analysis}, volume~3.
\newblock McGraw-hill New York, 1964.

\bibitem[Simonyan and Zisserman(2014)]{simonyan2014very}
Karen Simonyan and Andrew Zisserman.
\newblock Very deep convolutional networks for large-scale image recognition.
\newblock \emph{arXiv preprint arXiv:1409.1556}, 2014.

\bibitem[Snoek et~al.(2012)Snoek, Larochelle, and Adams]{snoek2012practical}
Jasper Snoek, Hugo Larochelle, and Ryan~P Adams.
\newblock Practical bayesian optimization of machine learning algorithms.
\newblock In \emph{Advances in neural information processing systems}, pages
  2951--2959, 2012.

\bibitem[Tibshirani(1996)]{tibshirani1996regression}
Robert Tibshirani.
\newblock Regression shrinkage and selection via the lasso.
\newblock \emph{Journal of the Royal Statistical Society: Series B
  (Methodological)}, 58\penalty0 (1):\penalty0 267--288, 1996.

\bibitem[Vapnik(2013)]{vapnik2013nature}
Vladimir Vapnik.
\newblock \emph{The nature of statistical learning theory}.
\newblock Springer science \& business media, 2013.

\bibitem[Virmaux and Scaman(2018)]{virmaux2018lipschitz}
Aladin Virmaux and Kevin Scaman.
\newblock Lipschitz regularity of deep neural networks: analysis and efficient
  estimation.
\newblock In \emph{Advances in Neural Information Processing Systems}, pages
  3835--3844, 2018.

\bibitem[Xiao and Zhang(2014)]{xiao2014proximal}
Lin Xiao and Tong Zhang.
\newblock A proximal stochastic gradient method with progressive variance
  reduction.
\newblock \emph{SIAM Journal on Optimization}, 24\penalty0 (4):\penalty0
  2057--2075, 2014.

\bibitem[Zhao et~al.(2014)Zhao, Yu, Wang, Arora, and Liu]{zhao2014accelerated}
Tuo Zhao, Mo~Yu, Yiming Wang, Raman Arora, and Han Liu.
\newblock Accelerated mini-batch randomized block coordinate descent method.
\newblock In \emph{Advances in neural information processing systems}, pages
  3329--3337, 2014.

\bibitem[Zou(2006)]{zou2006adaptive}
Hui Zou.
\newblock The adaptive lasso and its oracle properties.
\newblock \emph{Journal of the American statistical association}, 101\penalty0
  (476):\penalty0 1418--1429, 2006.

\end{thebibliography}

\section*{Appendix A: Proof of Theorem  \ref{theorem1}}  \label{firstthm0.5}
Before proving Theorem \ref{theorem1}, we first give Lemma \ref{lem1}.
\begin{lemma}\label{lem1}
Let $f$  and $g$ be a continuous function of $ \mathbb{R}^d \times \mathbb{R}^p \rightarrow \mathbb{R}$, and let $w_0 \in \mathbb{R}^d$ and $\lambda_0 \in \mathbb{R}^p$. Assume that $f$ and $g$ are continuous at the points $w_0$ and $\lambda_0$,  and let a be a real number.  If $h=f( g(w, \lambda),\lambda)$, then $h$ is continuous at $w_0$ and $\lambda_0$.
\end{lemma}
\begin{proof}
Given $\delta' \in \mathbb{R}^d$ and $\delta \in \mathbb{R}^p$, according to the definition of continuous function in Definition \ref{continuous}, we have that
\begin{eqnarray}
&& \lim_{\delta' \rightarrow \textbf{0}, \delta \rightarrow \textbf{0}}h(w_0+\delta', \lambda_0+\delta)
\\ \nonumber &=& \lim_{\delta' \rightarrow \textbf{0}, \delta \rightarrow \textbf{0}} f( g \left (w_0+\delta', \lambda_0+\delta \right ), \lambda_0+\delta)
\\ \nonumber &=&  f\left ( \lim_{\delta' \rightarrow \textbf{0}, \delta \rightarrow \textbf{0}} g \left (w_0+\delta', \lambda_0+\delta \right ), \lambda_0 \right )
\\ \nonumber &=&  f( g(w_0, \lambda_0),\lambda_0)
\end{eqnarray}
where  the second equality uses the definition of continuous function in
Definition \ref{continuous}.
This
completes the proof.
\end{proof}

\begin{customthm}{1}\label{theorem3}
If the hyperparameters $\lambda$ are continuous and the mapping functions $\Phi_t (w,\lambda)$ (for every $t\in \{1,\ldots,T \}$) are continuous, the mapping function $\mathcal{A}(\lambda)$  is continuous, and the outer objective $E$ is continuous, we have that the  $\mathcal{A}$-based constrained optimization problem  $f(\lambda)$ is continuous \emph{w.r.t.} $\lambda$.
\end{customthm}

\begin{proof}
As defined in Definition \ref{Iterative_alg},  the mapping function is actually the function
\begin{eqnarray}
\mathcal{A}(\lambda) = w_T=\Phi_T(\Phi_{T-1}(\ldots(\Phi_1(w_0,\lambda), \lambda),\ldots,\lambda)
\end{eqnarray}
Because each mapping function $\Phi_t (w,\lambda)$ is  continuous \emph{w.r.t.} $w$ and $\lambda$, we can recursively use Lemma \ref{lem1} to have that the mapping function $\mathcal{A}$ is continuous \emph{w.r.t.} $\lambda$.

Because $f(\lambda) = E(w_T,\lambda)$ and the function $E(w,\lambda)$ is continuous \emph{w.r.t.} $w$ and $\lambda$, we have that the  function $f(\lambda)$ is continuous \emph{w.r.t.} $\lambda$ according to Lemma \ref{lem1}. This
completes the proof.
\end{proof}

\section*{Appendix B: Proof of Theorem  \ref{theorem2}}
Before proving Theorem  \ref{theorem2}, we first give Lemma \ref{lemma_lip} which is provided in \citep{federer2014geometric}.
\begin{lemma}[\cite{federer2014geometric}] \label{lemma_lip}
  If $f(\lambda) : \mathbb{R}^p \rightarrow \mathbb{R}$ is a  Lipschitz continuous
function. Then, its Lipschitz constant $L(f)$ is
\begin{eqnarray}
L(f) = \sup_{\lambda \in \mathbb{R}^p} \| \partial_{\lambda} f(f(\lambda)) \|_2
\end{eqnarray}
\end{lemma}

\begin{customthm}{2}
Given the continuous mapping functions $\Phi_t (w_{t-1},\lambda)$ where $t\in \{1,\ldots,T \}$), $A_t=\frac{\partial \Phi_t (w_{t-1},\lambda)}{ \partial w_{t-1}}$, $B_t=\frac{\partial \Phi_t (w_{t-1},\lambda)}{ \partial \lambda}$. Given the continuous objective function $E(w_T,\lambda)$,  $A_{T+1}=\frac{\partial E (w_{T},\lambda)}{ \partial w_T}$ and $B_{T+1}=\frac{\partial E (w_{T},\lambda)}{ \partial \lambda }$. Let $L_{A_{t}}=\sup_{\lambda \in \mathbb{R}^p, w\in \mathbb{R}^d} \left  \| A_{t+1} \right  \|_2$, $L_{B_{t}}=\sup_{\lambda \in \mathbb{R}^p, w\in \mathbb{R}^d} \left  \|  B_{t}\right  \|_2$. Let $L(f) $ denote the  Lipschitz constant of the continuous function $f(\lambda)$,  we  can  upper bound  $L(f)$ by $\sum_{t=1}^{T+1}   L_{B_{t}}  L_{A_{t+1}} \ldots L_{A_{T+1}} $.
\end{customthm}
\begin{proof}Firstly, according to the chain rule \citep{Rudin1976Principles}, we give the computation of $\partial_{\lambda} f(f(\lambda))
$ as follows.
\begin{eqnarray}
&& \partial_{\lambda} f(\lambda) =  \frac{\partial E (w_{T},\lambda)}{ \partial w_T} \frac{\partial  w_T}{ \partial \lambda}  + \frac{\partial E (w_{T},\lambda)}{ \partial \lambda }
\\ \nonumber &=& A_{T+1} \frac{\partial  w_T}{ \partial \lambda} + B_{T+1}
\\ \nonumber &=& A_{T+1} \left ( \frac{\partial \Phi_T (w_{T-1},\lambda)}{ \partial w_{T-1}} \frac{\partial  w_{T-1}}{ \partial \lambda}  + \frac{\partial \Phi_T (w_{T-1},\lambda)}{ \partial \lambda } \right ) + B_{T+1}
\\ \nonumber &=& A_{T+1} \left (A_{T} \frac{\partial  w_{T-1}}{ \partial \lambda}  + B_{T} \right ) + B_{T+1}
\\ \nonumber &=& A_{T+1} A_{T} \frac{\partial  w_{T-1}}{ \partial \lambda}  + A_{T+1} B_{T}  + B_{T+1}
\\ \nonumber &=& \sum_{t=1}^{T+1}B_{t}   A_{t+1}\ldots  A_{T+1}
\end{eqnarray}

Secondly, according to Lemma \ref{lemma_lip}, we have that
\begin{eqnarray}
&& L(f) = \sup_{\lambda \in \mathbb{R}^p} \| \partial_{\lambda} f(\lambda) \|_2
\\ \nonumber &=& \sup_{\lambda \in \mathbb{R}^p} \| \partial_{\lambda} f(\lambda) \|_2
\\ \nonumber &=& \sup_{\lambda \in \mathbb{R}^p} \left  \| \sum_{t=1}^{T+1}B_{t}   A_{t+1}\ldots  A_{T+1}  \right  \|_2
\\ \nonumber &\leq& \sum_{t=1}^{T+1} \sup_{\lambda \in \mathbb{R}^p} \left  \|  B_{t}   A_{t+1}\ldots  A_{T+1}  \right  \|_2
\\ \nonumber &\leq& \sum_{t=1}^{T+1} \sup_{\lambda \in \mathbb{R}^p, w\in \mathbb{R}^d} \left  \|  B_{t}\right  \|_2 \sup_{\lambda \in \mathbb{R}^p, w\in \mathbb{R}^d} \left  \| A_{t+1} \right  \|_2 \ldots  \sup_{\lambda \in \mathbb{R}^p, w\in \mathbb{R}^d} \left  \| A_{T+1}  \right  \|_2
\\ \nonumber &\leq& \sum_{t=1}^{T+1}   L_{B_{t}}  L_{A_{t+1}} \ldots L_{A_{T+1}}
\end{eqnarray}
This
completes the proof.
\end{proof}

\section*{Appendix C: Parameter Sensitivity of HOZOG}
In this part, we provide more experimental results of HOZOG on the  $l_2$-regularized logistic regression (on News20 dataset), the  data hyper-cleaning task (with 500 hyperparameters) and deep neural networks (including 2-layer CNN, VGG-16 and ResNet-152) under different  settings of  parameters $q$, $\mu$ and $\gamma$  to show  the parameter sensitivity  of HOZOG. The results are shown in Figures \ref{figureSensity5}, \ref{figureSensity1}, \ref{figureSensity2}, \ref{figureSensity3} and \ref{figureSensity4}  which demonstrate the convergence curves of HOZOG under the settings of $q$, $\gamma$ and $\mu$. From the results, we find  that HOZOG is robust to the  different settings of $q$, $\mu$ and $\gamma$.

\begin{figure*}[htbp]
	\centering
	\begin{subfigure}[b]{0.32\textwidth}
		\centering
		\includegraphics[width=2in]{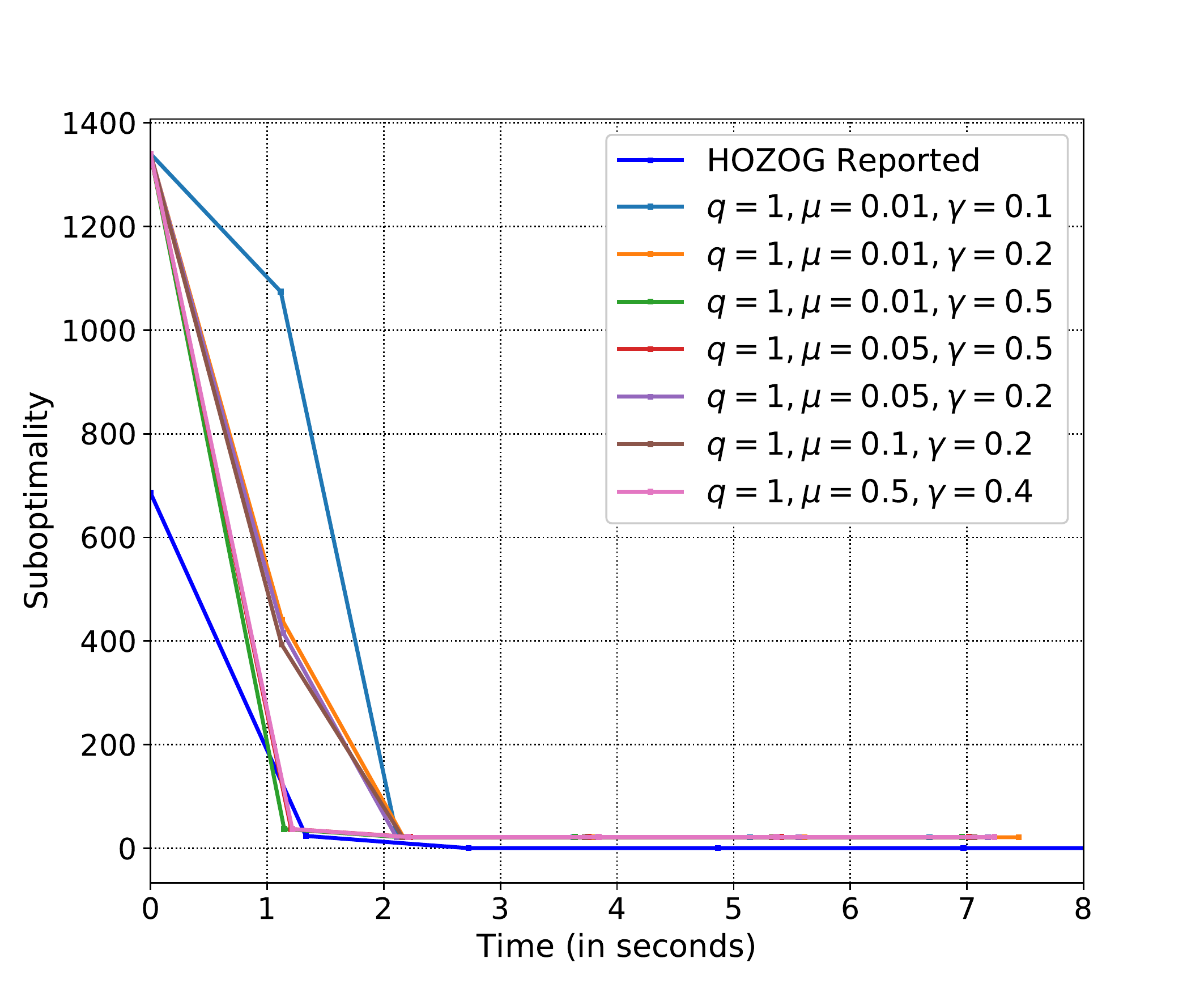}
		\caption{Suboptimality}
	\end{subfigure}
	\begin{subfigure}[b]{0.32\textwidth}
		\centering
		\includegraphics[width=2in]{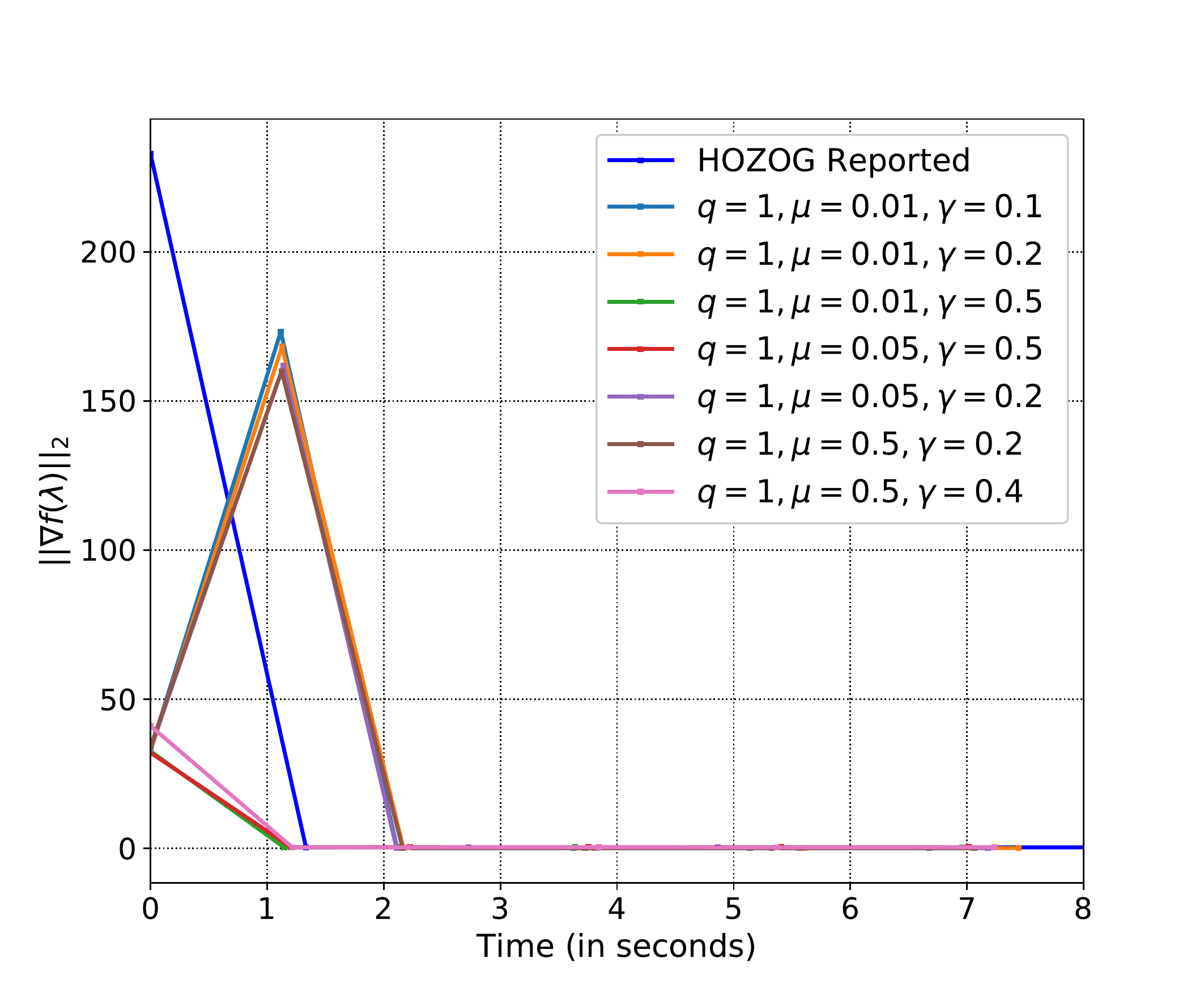}
		\caption{$\| \nabla f(\lambda)\|_2$}
	\end{subfigure}
	\begin{subfigure}[b]{0.32\textwidth}
		\centering
		\includegraphics[width=2in]{figparameter_comparisonLRnews20_loss-eps-converted-to.pdf}
		\caption{Test error}
	\end{subfigure}
	\caption{Comparison of different hyperparameter settings for $l_2$-regularized logistic regression.
	}
	\label{figureSensity5}
\end{figure*}

\begin{figure*}[htbp]
	\centering
	\begin{subfigure}[b]{0.32\textwidth}
		\centering
		\includegraphics[width=2in]{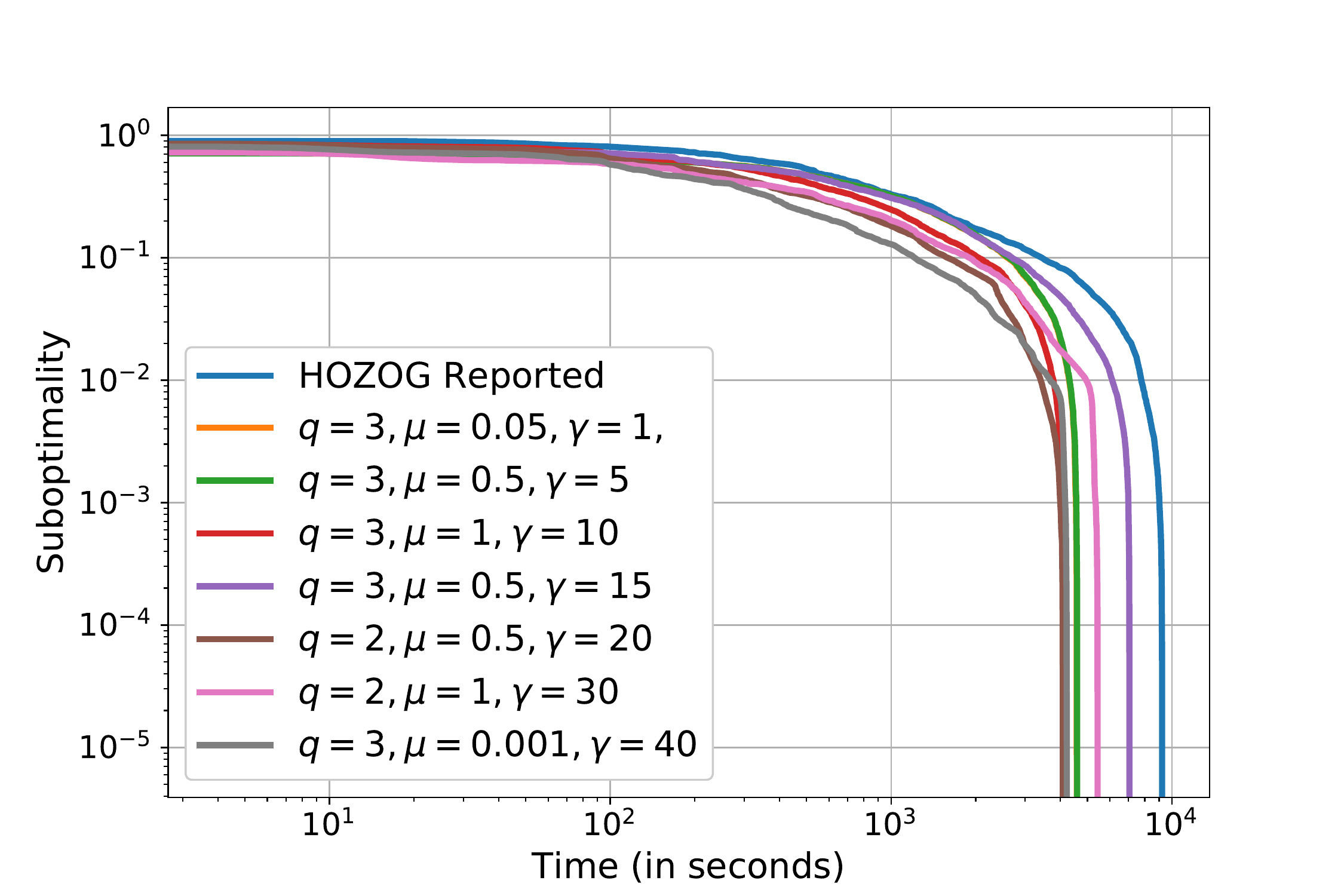}
		\caption{Suboptimality}
	\end{subfigure}
	\begin{subfigure}[b]{0.32\textwidth}
		\centering
		\includegraphics[width=2in]{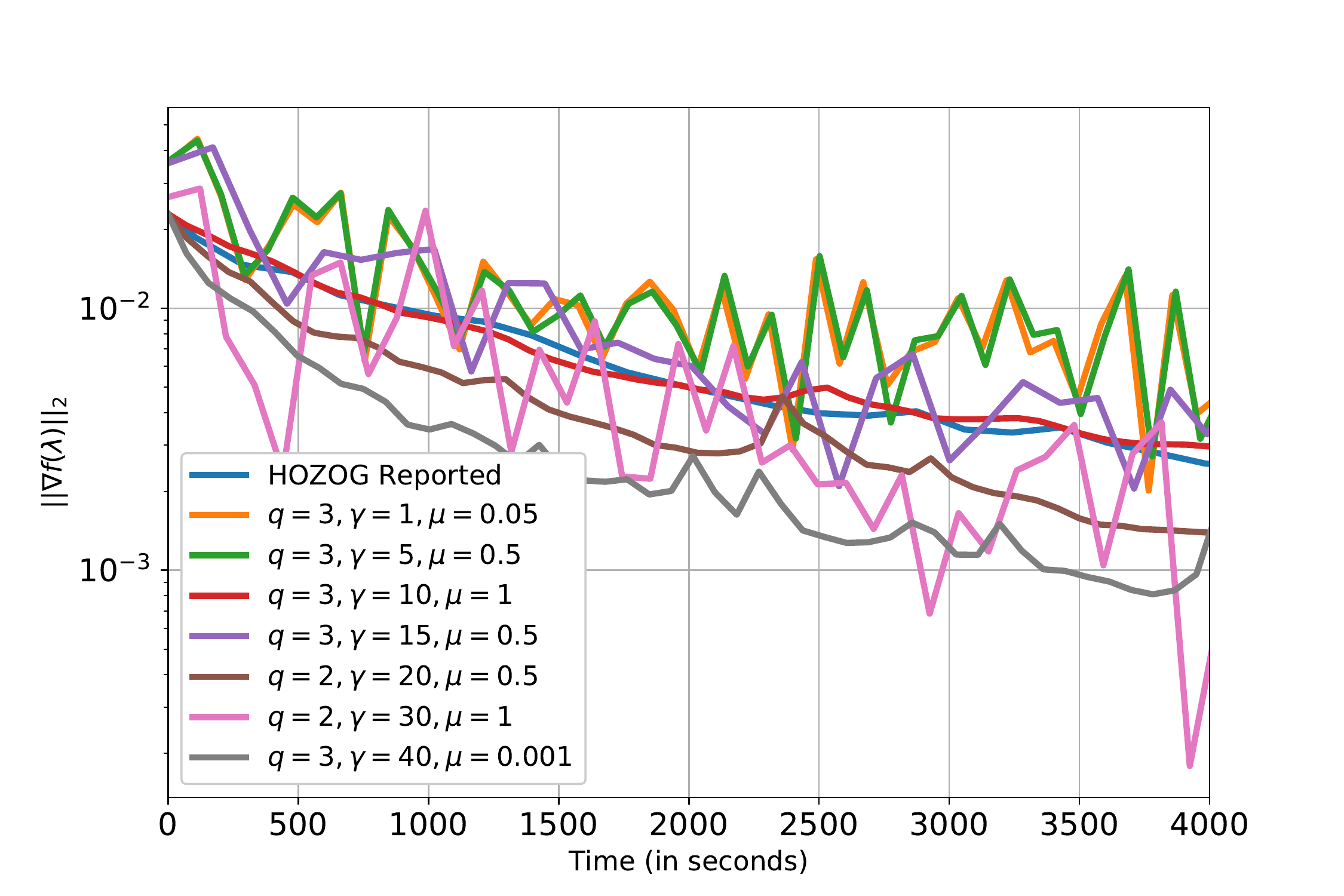}
		\caption{$\| \nabla f(\lambda)\|_2$}
	\end{subfigure}
	\begin{subfigure}[b]{0.32\textwidth}
		\centering
		\includegraphics[width=2in]{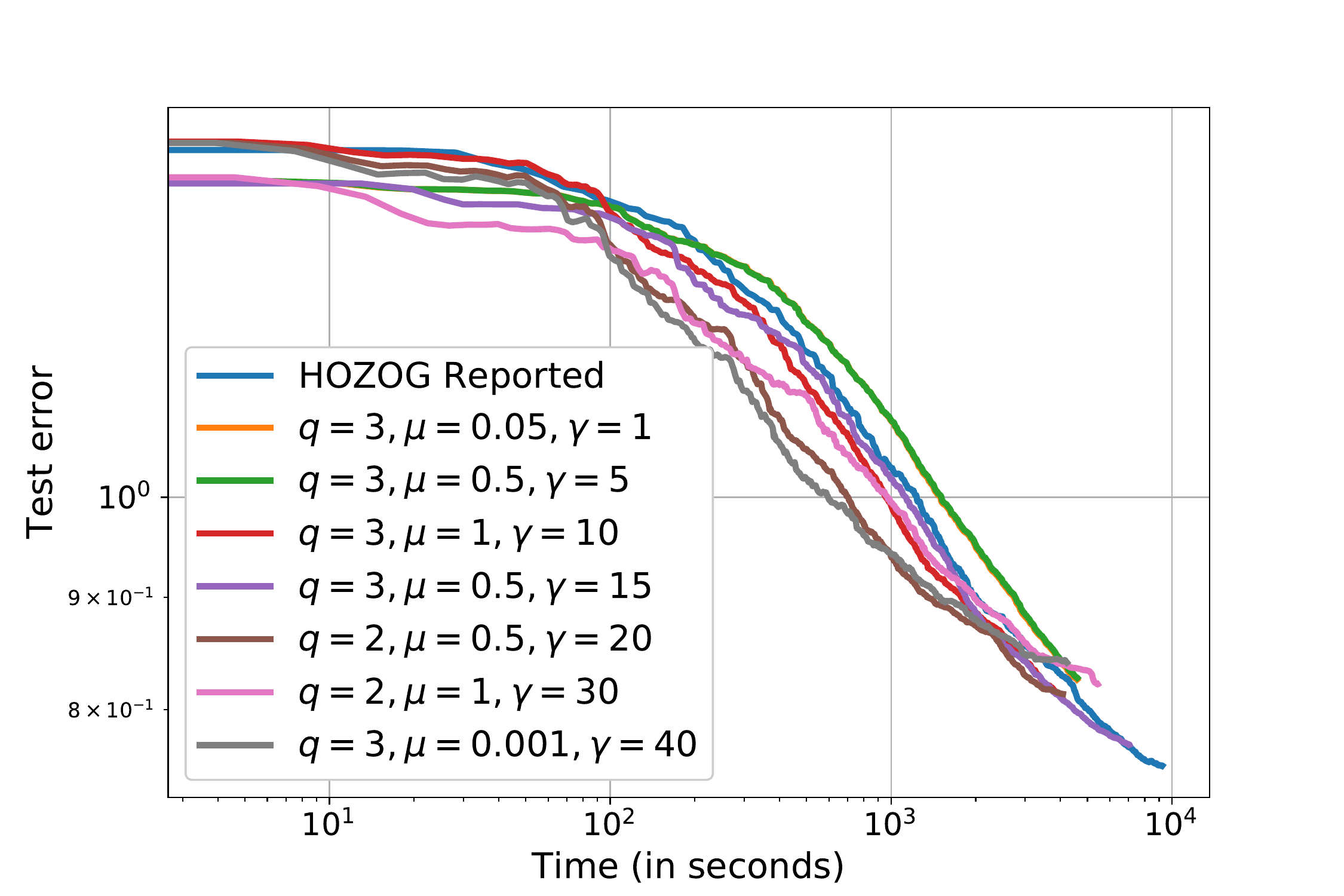}
		\caption{Test error}
	\end{subfigure}
	\caption{Comparison of different hyperparameter settings for data hyper-cleaning.
	}
	\label{figureSensity1}
\end{figure*}

\begin{figure*}[htbp]
	\centering
	\begin{subfigure}[b]{0.32\textwidth}
		\centering
		\includegraphics[width=2in]{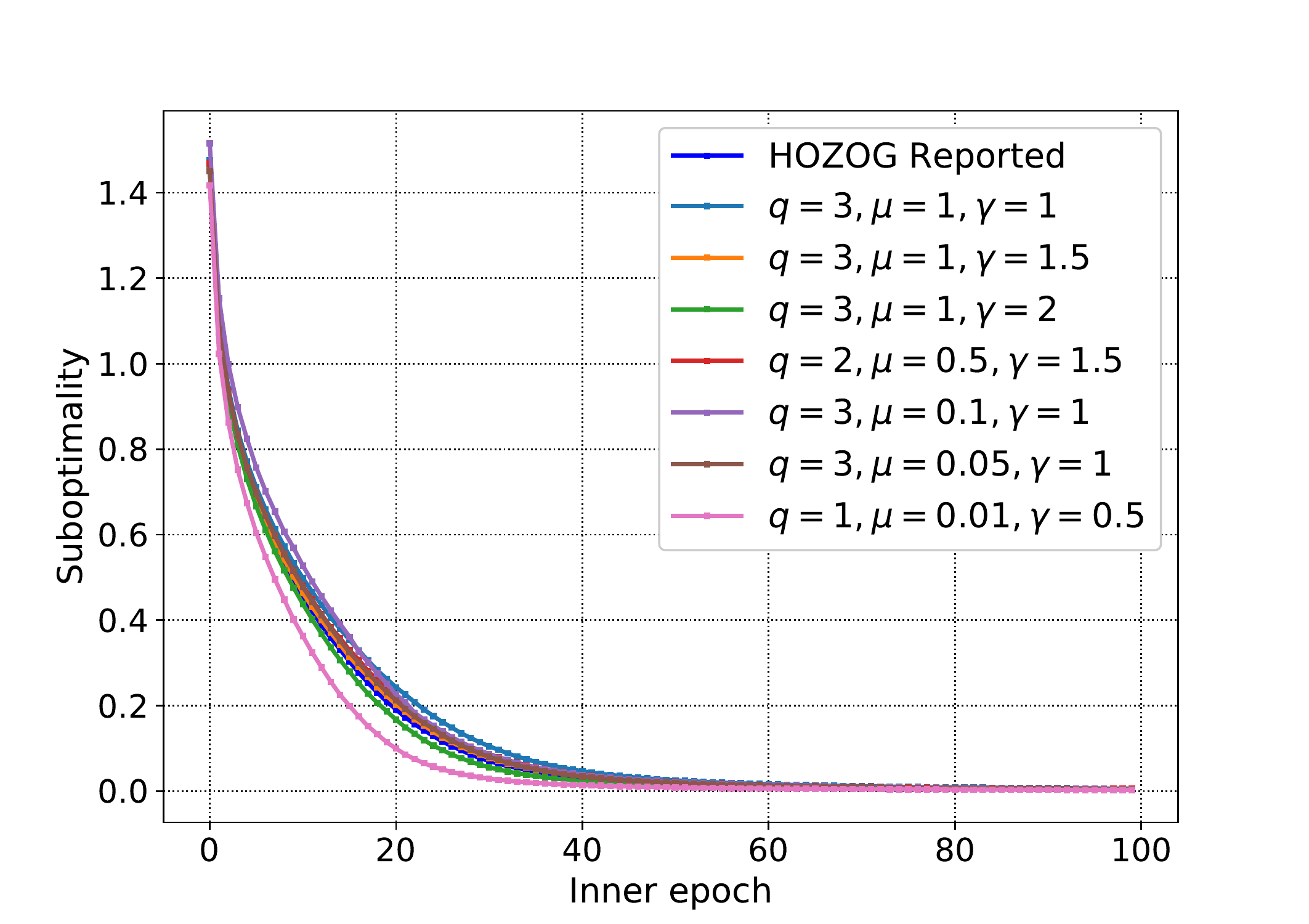}
		\caption{Suboptimality}
	\end{subfigure}
	\begin{subfigure}[b]{0.32\textwidth}
		\centering
		\includegraphics[width=2in]{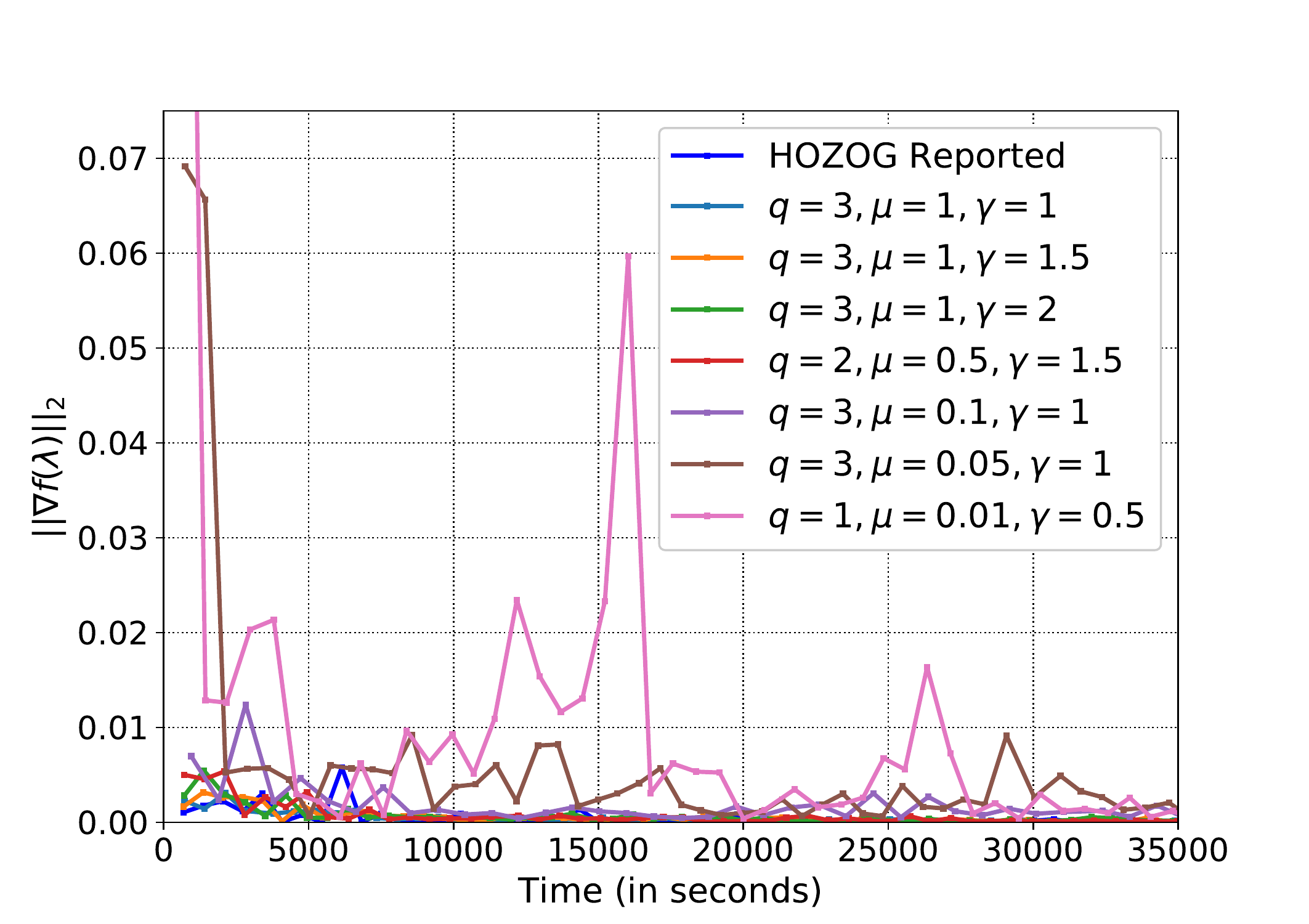}
		\caption{$\| \nabla f(\lambda)\|_2$}
	\end{subfigure}
	\begin{subfigure}[b]{0.32\textwidth}
		\centering
		\includegraphics[width=2in]{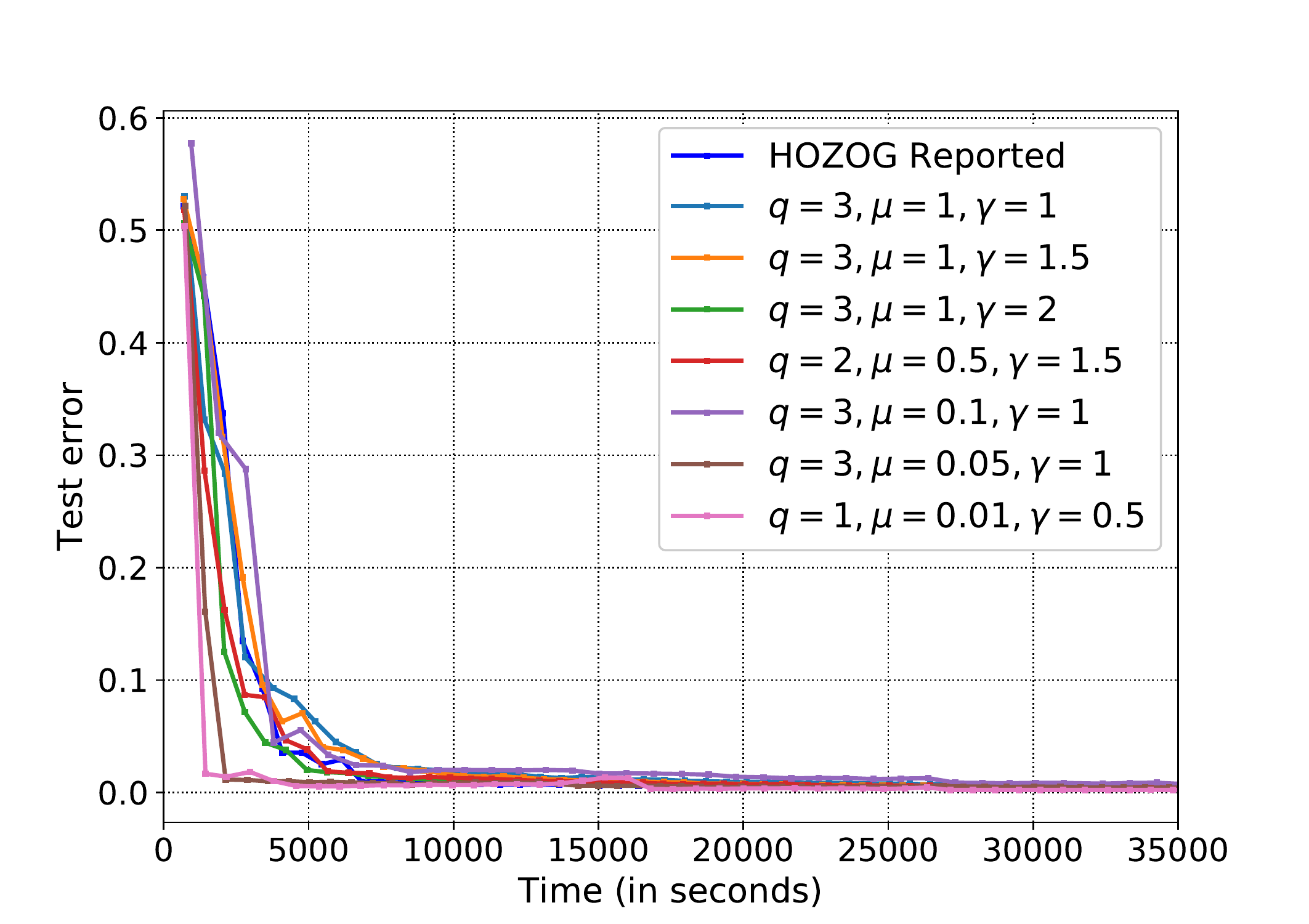}
		\caption{Test error}
	\end{subfigure}
	\caption{Comparison of different hyperparameter settings for 2-layer CNN.
	}
	\label{figureSensity2}
\end{figure*}

\begin{figure*}[htbp]
	\centering
	\begin{subfigure}[b]{0.32\textwidth}
		\centering
		\includegraphics[width=2in]{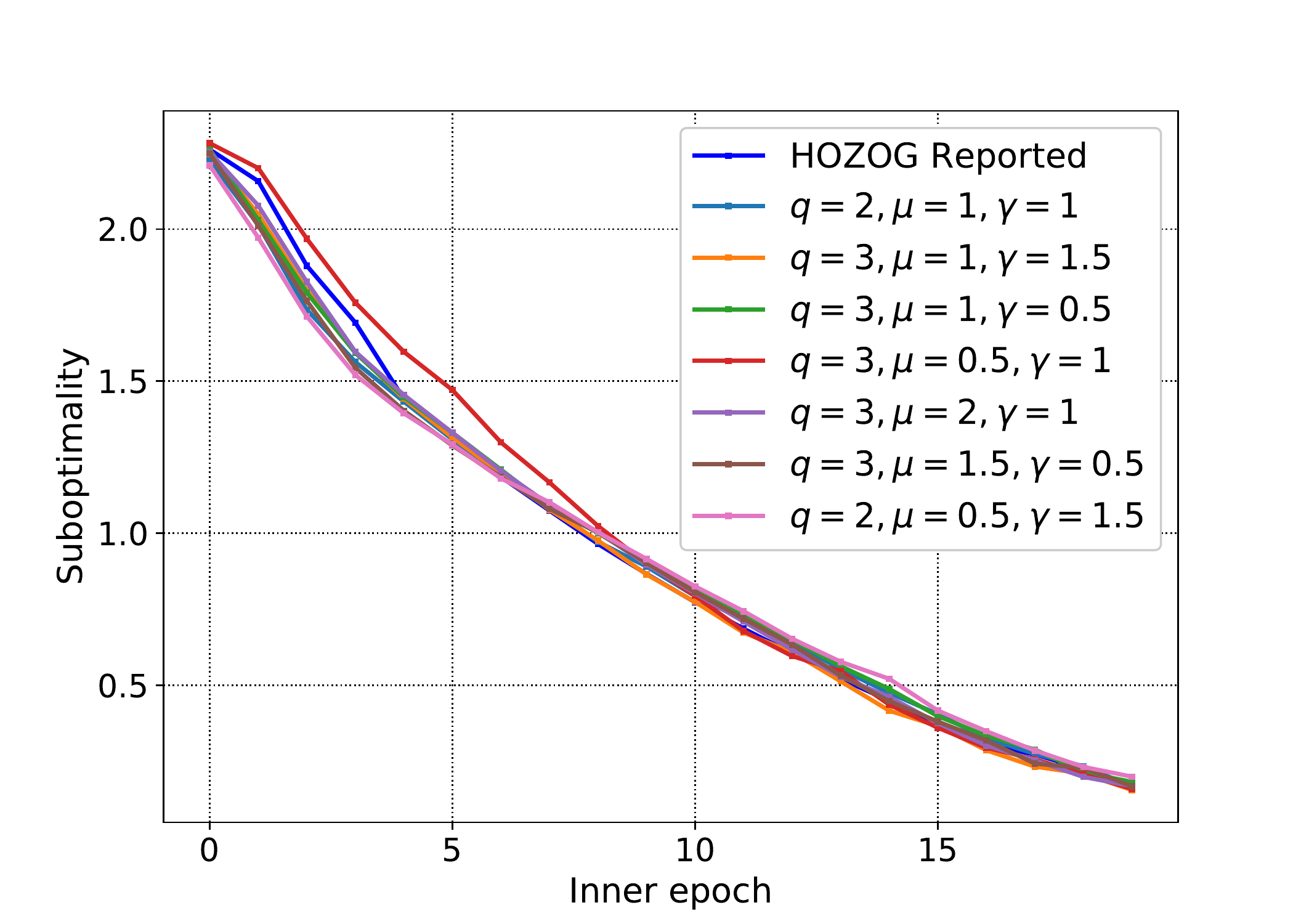}
		\caption{Suboptimality}
	\end{subfigure}
	\begin{subfigure}[b]{0.32\textwidth}
		\centering
		\includegraphics[width=2in]{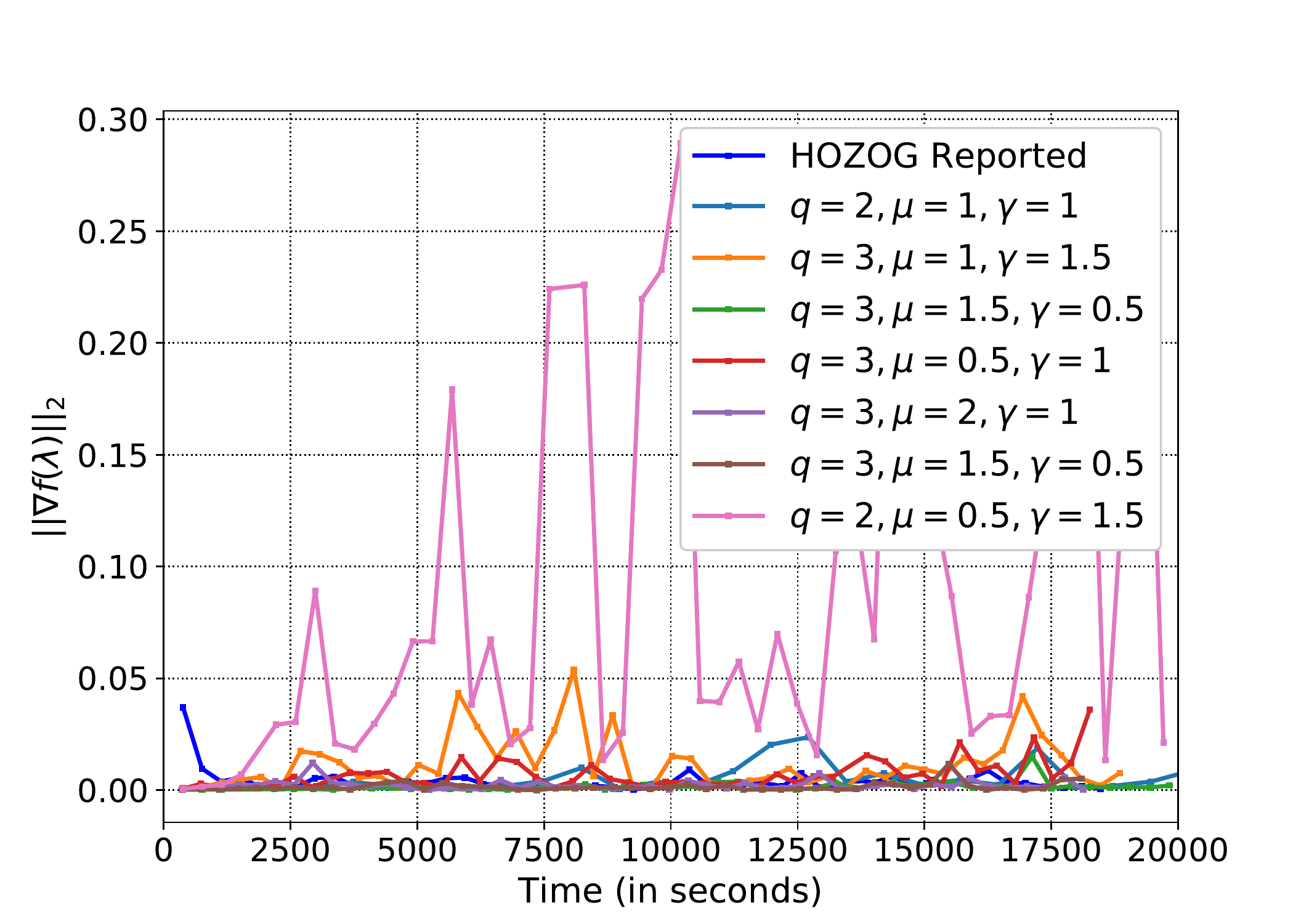}
		\caption{$\| \nabla f(\lambda)\|_2$}
	\end{subfigure}
	\begin{subfigure}[b]{0.32\textwidth}
		\centering
		\includegraphics[width=2in]{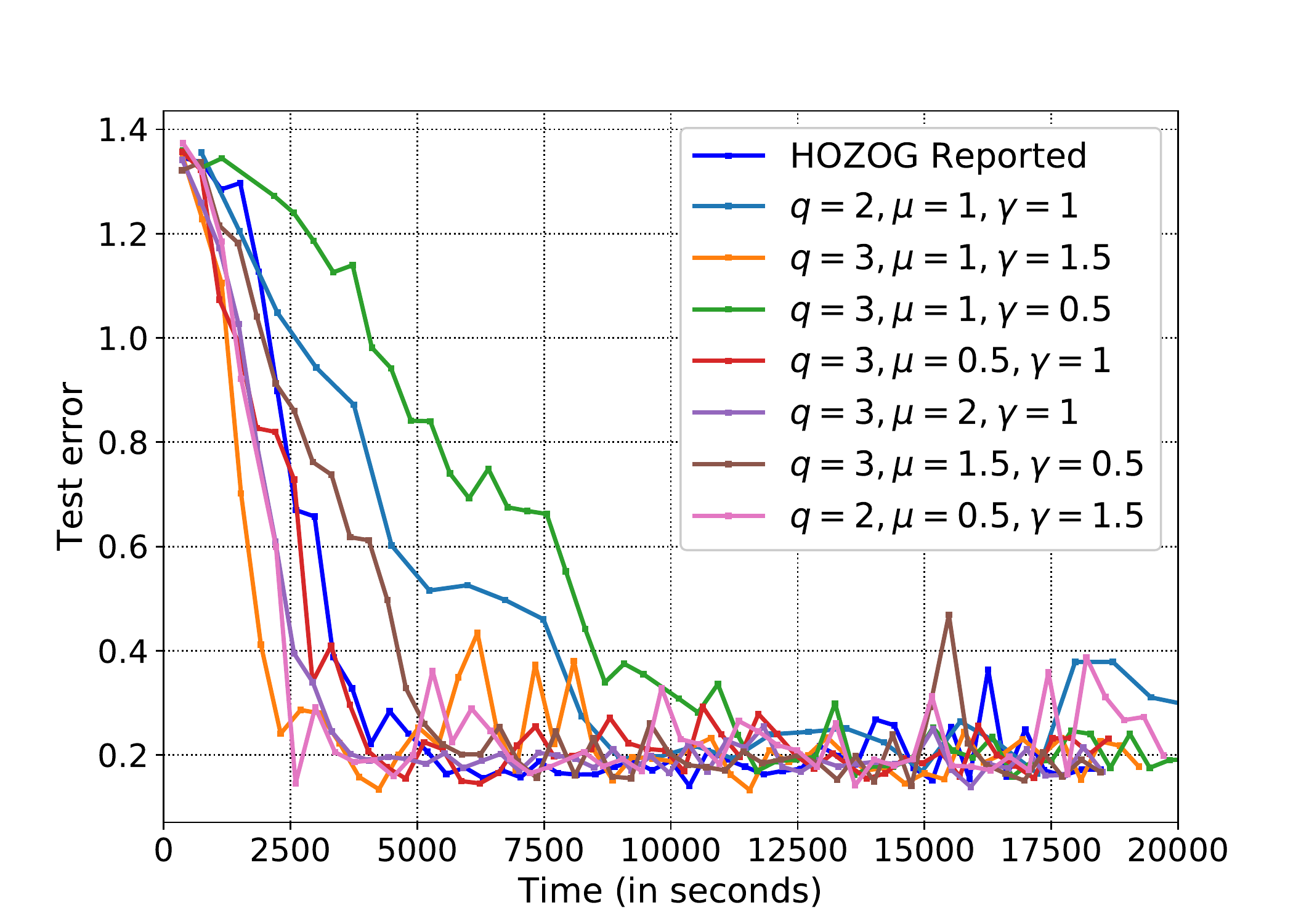}
		\caption{Test error}
	\end{subfigure}
	\caption{Comparison of different hyperparameter settings for VGG-16.
	}
	\label{figureSensity3}
\end{figure*}

\begin{figure*}[htbp]
	\centering
	\begin{subfigure}[b]{0.32\textwidth}
		\centering
		\includegraphics[width=2in]{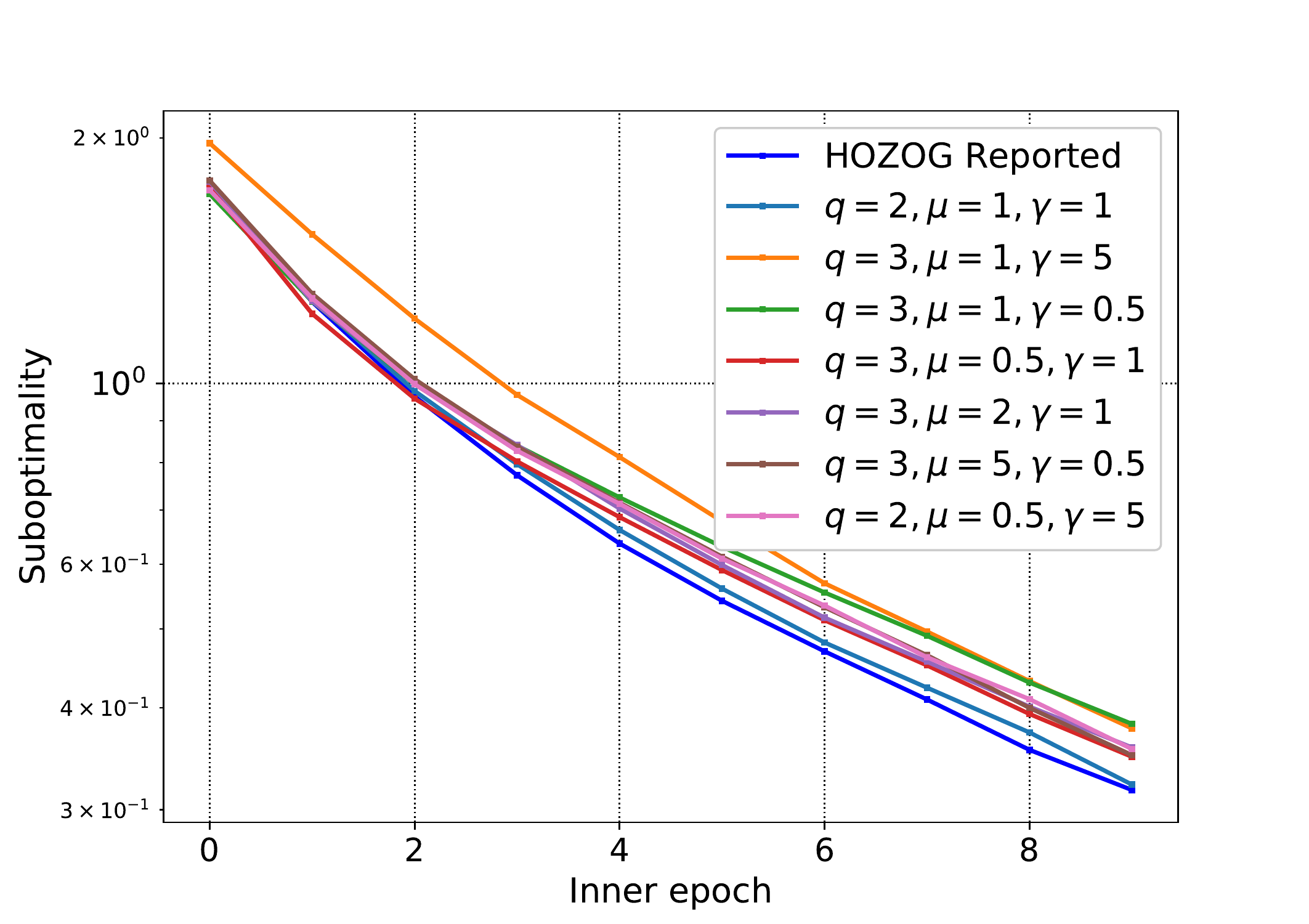}
		\caption{Suboptimality}
	\end{subfigure}
	\begin{subfigure}[b]{0.32\textwidth}
		\centering
		\includegraphics[width=2in]{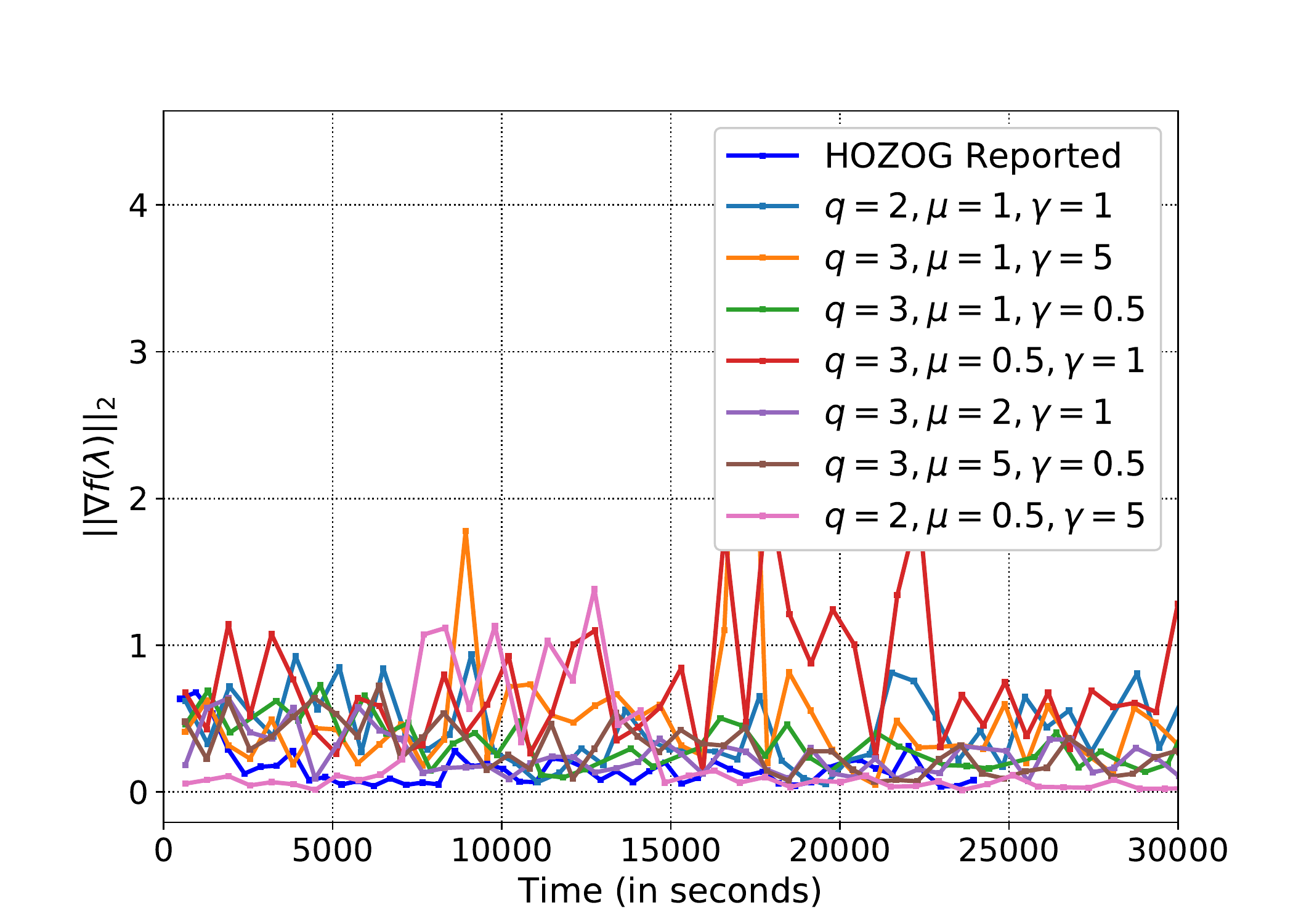}
		\caption{$\| \nabla f(\lambda)\|_2$}
	\end{subfigure}
	\begin{subfigure}[b]{0.32\textwidth}
		\centering
		\includegraphics[width=2in]{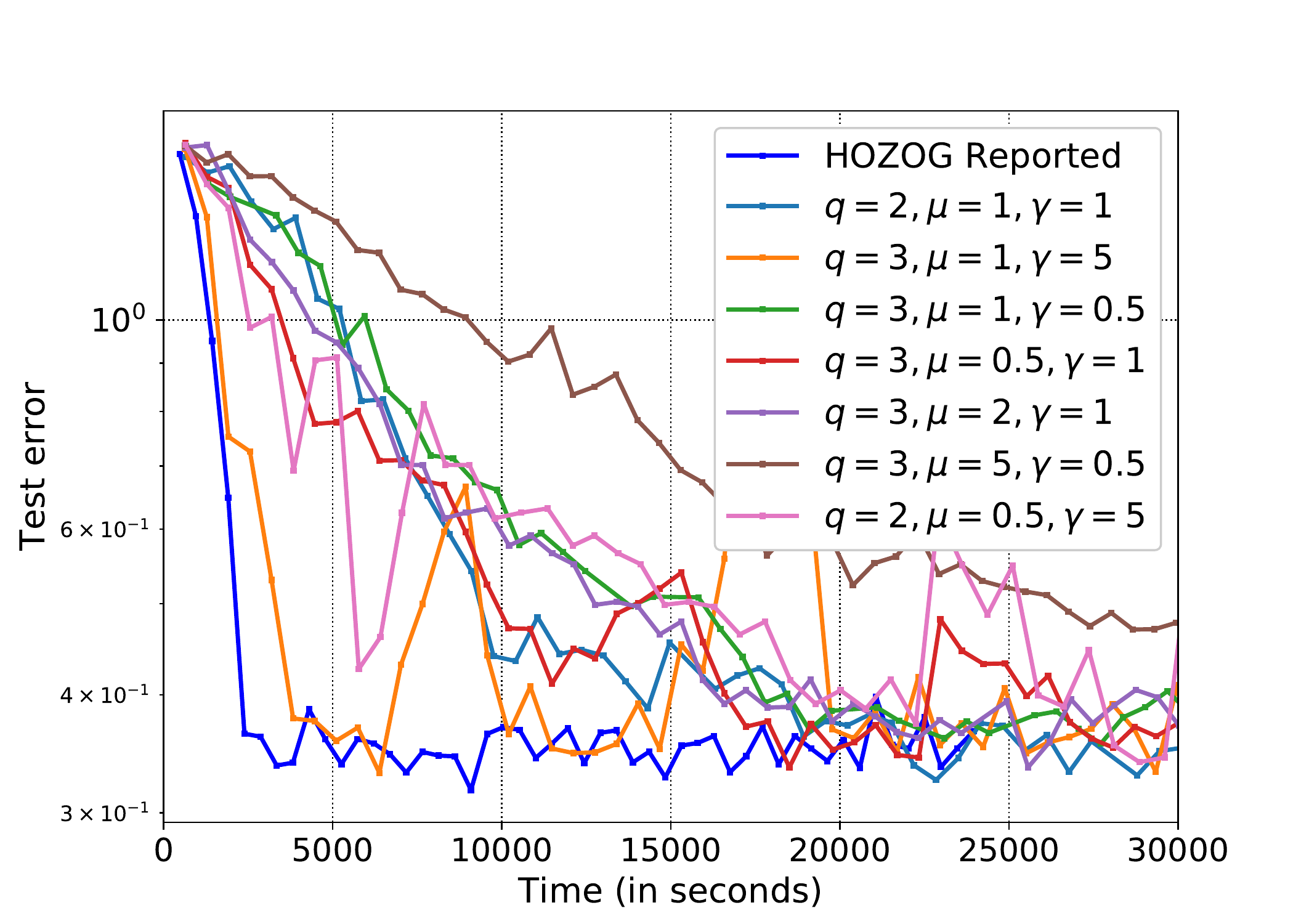}
		\caption{Test error}
	\end{subfigure}
	\caption{Comparison of different hyperparameter settings for ResNet-152.
	}
	\label{figureSensity4}
\end{figure*}

\end{document}